\documentclass{article}

\usepackage{natbib}
\usepackage{fullpage}
\usepackage[utf8]{inputenc}
\usepackage[T1]{fontenc}
\usepackage{hyperref}
\usepackage{url}
\usepackage{booktabs}
\usepackage{amsfonts}
\usepackage{amsmath,amssymb,amsthm,mathtools}
\usepackage{nicefrac}
\usepackage{microtype}
\usepackage{xcolor}

\usepackage{float}
\usepackage{enumitem}
\usepackage{tikz}
\usepackage{tcolorbox}
\usetikzlibrary{arrows.meta,calc,positioning,fit,backgrounds}
\hypersetup{hidelinks,hypertexnames=false}
\usepackage{subcaption}
\usepackage{tabularx}

\usepackage{algorithm}
\usepackage{algpseudocode}

\usepackage{todonotes}
\usepackage{aliascnt}

\newtheorem{theorem}{Theorem}[section]
\newaliascnt{lemma}{theorem}
\newtheorem{lemma}[lemma]{Lemma}
\aliascntresetthe{lemma}
\newaliascnt{proposition}{theorem}
\newtheorem{proposition}[proposition]{Proposition}
\aliascntresetthe{proposition}
\newaliascnt{corollary}{theorem}
\newtheorem{corollary}[corollary]{Corollary}
\aliascntresetthe{corollary}
\newaliascnt{definition}{theorem}
\newtheorem{definition}[definition]{Definition}
\aliascntresetthe{definition}
\newaliascnt{assumption}{theorem}

\aliascntresetthe{assumption}
\newaliascnt{remark}{theorem}
\newtheorem{remark}[remark]{Remark}
\aliascntresetthe{remark}
\newaliascnt{example}{theorem}

\aliascntresetthe{example}

\usepackage[capitalize,noabbrev]{cleveref}
\makeatletter
\AddToHook{cmd/appendix/before}{\def\cref@section@alias{appendix}\def\cref@subsection@alias{appendix}}
\makeatother

\crefname{theorem}{Theorem}{theorems}
\Crefname{theorem}{Theorem}{Theorems}
\crefname{lemma}{Lemma}{lemmas}
\Crefname{lemma}{Lemma}{Lemmas}
\crefname{proposition}{Proposition}{propositions}
\Crefname{proposition}{Proposition}{Propositions}
\crefname{corollary}{Corollary}{corollaries}
\Crefname{corollary}{Corollary}{Corollaries}
\crefname{definition}{Definition}{definitions}
\Crefname{definition}{Definition}{Definitions}
\crefname{assumption}{Assumption}{assumptions}
\Crefname{assumption}{Assumption}{Assumptions}
\crefname{remark}{Remark}{remarks}
\Crefname{remark}{Remark}{Remarks}
\crefname{example}{Example}{examples}
\Crefname{example}{Example}{Examples}

\usepackage{pgfplots}
\pgfplotsset{compat=1.18}

\definecolor{myOrange}{RGB}{235, 105, 92}
\definecolor{myMaroon}{RGB}{128, 0, 0}
\definecolor{maroon}{rgb}{0.5,0.0,0.0}
\definecolor{heavymaroon}{rgb}{0.3,0.0,0.0}

\definecolor{mygreen}{RGB}{160, 200, 140}

\hypersetup{ 
colorlinks=true, 
linkcolor=myOrange, 
citecolor=myMaroon, filecolor=cyan, 
urlcolor=mygreen
}

\usepackage{lipsum}                     %
\usepackage{xargs}
\usepackage{xfrac}
\usepackage{amssymb}
\usepackage{amsmath}
\usepackage{bm} 
\usepackage{mathtools} 
\providecommand{\colloneq}{\coloneqq}
\usepackage{amsthm}
\usepackage{mleftright}
\usepackage{refcount}

\def\[#1\]{\begin{align*}#1\end{align*}}

\NewDocumentCommand{\numberthis}{om}{%
  \IfNoValueTF{#1}{%
    \refstepcounter{equation}\tag{\theequation}%
  }{%
    \tag{#1}%
  }%
  \label{#2}%
}

\renewcommand{\vec}[1]{\bm{#1}}

\newcommand{\R}{\mathbb{R}}

\newcommand{\eps}{\varepsilon}

\newcommand{\ThetaSet}{\Theta}
\newcommand{\PhiExact}{\Phi_{\mathrm{exact}}}
\newcommand{\PhiProx}{\Phi_{\mathrm{prox}}}
\newcommand{\PhiExactSmooth}{\Phi_{\mathrm{exact}}^{\mathrm{sm}}}
\newcommand{\PhiMirrorExact}[1]{\Phi_{\mathrm{exact}}^{#1}}
\newcommand{\PhiMirrorProx}[1]{\Phi_{\mathrm{prox}}^{#1}}
\newcommand{\PhiMirrorRad}[1]{\Phi_{\mathrm{rad}}^{#1}}
\newcommand{\PhiFTRL}[1]{\Phi_{\mathrm{FTRL}}^{#1}}

\DeclareMathOperator{\dom}{dom}

\DeclareMathOperator{\relint}{relint}
\DeclareMathOperator{\aff}{aff}
\DeclareMathOperator{\cl}{cl}
\DeclareMathOperator{\spanop}{span}
\DeclareMathOperator{\osc}{osc}
\DeclareMathOperator{\Reg}{Reg}
\DeclareMathOperator{\prox}{prox}
\DeclareMathOperator*{\argmin}{arg\,min}

\newcommand{\dd}{\mathrm{d}}
\newcommand{\transpose}{\mathsf{T}}
\newcommand{\inner}[2]{\left\langle #1,#2\right\rangle}
\newcommand{\norm}[1]{\left\lVert #1\right\rVert}
\newcommand{\dualnorm}[1]{\left\lVert #1\right\rVert_{\!*}}
\newcommand{\abs}[1]{\left\lvert #1\right\rvert}
\newcommand{\setof}[2]{\left\{#1\,:\,#2\right\}}
\newcommand{\indicator}{\iota}
\newcommand{\one}{\mathbf{1}}

\newcommand{\Proj}{\Pi}
\newcommand{\normal}{N}
\newcommand{\Breg}{D}
\newcommand{\Moreau}{M}
\newcommand{\field}{\vec{D}}
\newcommand{\gvec}{\vec{g}}
\newcommand{\xvec}{\vec{x}}
\newcommand{\yvec}{\vec{y}}
\newcommand{\zvec}{\vec{z}}
\newcommand{\avec}{\vec{a}}
\newcommand{\hvec}{\vec{h}}
\newcommand{\rvec}{\vec{r}}
\newcommand{\uvec}{\vec{u}}
\newcommand{\vvec}{\vec{v}}
\newcommand{\bvec}{\vec{b}}
\newcommand{\dvec}{\vec{d}}
\newcommand{\pvec}{\vec{p}}
\newcommand{\qvec}{\vec{q}}
\newcommand{\nvec}{\vec{n}}
\newcommand{\svec}{\vec{s}}
\newcommand{\thetavec}{\vec{\theta}}
\newcommand{\betavec}{\vec{\beta}}
\newcommand{\evec}{\vec{e}}
\newcommand{\Dphidual}{\widetilde{\field}_{\phi}}
\newcommand{\Dphi}{\field_{\phi}}
\newcommand{\DR}{\Breg_R}

\newcommand{\MF}{\Moreau_F}
\newcommand{\MRF}{\Moreau^R_F}
\newcommand{\simplex}{\Delta}
\DeclareMathOperator{\softmax}{softmax}

\newcommand{\gradtheta}{\nabla_{\vec{\theta}}}

\newtheorem*{repeatthminner}{\repeatthmname}
\newcommand{\repeatthmname}{}

\newenvironment{repeatthm}[1]{%
  \renewcommand{\repeatthmname}{Theorem~\ref{#1}}%
  \begin{repeatthminner}%
}{%
  \end{repeatthminner}%
}

\newtheorem*{repeatlemmainner}{\repeatlemmaname}
\newcommand{\repeatlemmaname}{}

\newenvironment{repeatlemma}[1]{%
  \renewcommand{\repeatlemmaname}{Lemma~\ref{#1}}%
  \begin{repeatlemmainner}%
}{%
  \end{repeatlemmainner}%
}

\newtheorem*{repeatpropinner}{\repeatpropname}
\newcommand{\repeatpropname}{}

\newenvironment{repeatprop}[1]{%
  \renewcommand{\repeatpropname}{Proposition~\ref{#1}}%
  \begin{repeatpropinner}%
}{%
  \end{repeatpropinner}%
}

\newtheorem*{repeatcorollaryinner}{\repeatcorollaryname}
\newcommand{\repeatcorollaryname}{}

\newenvironment{repeatcorollary}[1]{%
  \renewcommand{\repeatcorollaryname}{Corollary~\ref{#1}}%
  \begin{repeatcorollaryinner}%
}{%
  \end{repeatcorollaryinner}%
}

\providecommand{\E}{\mathbb{E}}
\providecommand{\CE}{\mathrm{CE}}
\providecommand{\CCE}{\mathrm{CCE}}
\providecommand{\ConCE}{\mathrm{ConCE}}
\providecommand{\PCE}{\mathrm{PCE}}

\title{Exact-Form Regret for Gradient Descent, Mirror Descent and \\ Follow-the-Regularized-Leader}

\author{
Ashkan Soleymani\\
MIT\\
\texttt{ashkanso@mit.edu}
\and
Gabriele Farina\\
MIT\\
\texttt{gfarina@mit.edu}
\and
Patrick Jaillet\\
MIT\\
\texttt{jaillet@mit.edu}
}

\begin{document}

\date{}

\maketitle
\begin{abstract}
Online gradient descent is usually studied through external regret, where the
learner competes with fixed alternatives.  Recent work shows that first-order
methods control much richer action-dependent deviations.  We ask for a
geometric characterization of the deviations with respect to which
online gradient descent, mirror descent, and follow-the-regularized-leader (FTRL)
achieve no regret.  We identify exactness as the common principle.  Here,
exactness means that the relevant displacement field is generated by a scalar
potential, or equivalently that the associated one-form is exact in the
geometry used by the algorithm.  This geometry depends on the algorithm.  For
gradient descent it is ordinary Euclidean geometry, for mirror descent it is
the geometry induced by the regularizer, and for FTRL it is the cumulative
dual state.  Under mild regularity conditions, exactness yields sublinear
regret, while nonzero circulation provides the complementary obstruction and
leads to linear regret.  This gives a unified geometric framework for understanding the deviation classes controlled by these algorithms and reveals that
different first-order methods can control genuinely different classes of
deviations. These deviation classes have direct consequences for learning, particularly
in games.  We study the equilibrium notions induced by exact-form deviations
and introduce conservative correlated equilibrium, reflecting both the
conservative geometry of the underlying displacement fields and the restricted
family of deviations available to the players.  We characterize its relation
to correlated equilibrium, determine when the resulting equilibrium notions
coincide and when they separate, and show how these relationships depend on
the geometry and the learning algorithm.  Overall, this work gives a unified
geometric account of what first-order online learning algorithms are no-regret
with respect to, beyond fixed comparators.
\end{abstract}

\section{Introduction}
\label{intro}

Online gradient descent is classically understood as an external-regret
minimizer~\citep{zinkevich2003online,hazan2016introduction,orabona2019modern}.
Against a sequence of linear losses, its iterates $\xvec_t\in K$, where $K\subseteq\R^d$ is a nonempty closed convex action set, compete with every fixed
comparator $\xvec' \in K$.  This guarantee is the basic engine behind the
standard convergence of no-regret play to coarse correlated
equilibrium~\citep{hart2000simple,cesa2006prediction}.  From the viewpoint of
swap regret and the notion of hindsight rationality, however, external regret is only the simplest notion out of a wide spectrum.
External regret compares $\xvec_t$ with fixed alternatives, while a general deviation may
replace $\xvec_t$ by a feasible point $\phi(\xvec_t) \in K$ that depends on the
current action~\citep{greenwald2003general,blum2007external,gordon2008no}.

Recent work makes clear that gradient dynamics control more than fixed
comparators, although the guarantees appear in different forms.  In
first-order equilibrium theory, \citet{ahunbay2024first} studies infinitesimal
strategy modifications generated by continuous vector fields.  In the coarse
adversarial setting, the fields certified by projected gradient dynamics are
conservative gradient fields with the appropriate tangency to the action set.
In normal-form games, \citet{ahunbay2025semicoarse} specialize this picture to
a tractable linear--quadratic family and obtain semicoarse correlated
equilibria through LP-based guarantees.  In parallel,
\citet*{cai2025proximal} showed that online gradient descent and mirror descent
control deviations generated by proximal operators of weakly convex
functions (a proper lower-semicontinuous function $F$ is $\rho$-weakly convex, for $\rho\ge0$, if $F+\frac{\rho}{2}\norm{\cdot}^2$ is convex).  As discussed by
\citet{cai2025proximal} and \citet{ahunbay2024first}, these lines of work have substantial
overlap but are not comparable in general, reflecting the different deviation
classes and guarantees they study.

In this paper, we study $\Phi$-regret (for a family $\Phi$ of feasible deviations) against actual feasible
deviation maps $\phi\colon K\to K$.  Writing
$\Dphi(\xvec)\colloneq\xvec-\phi(\xvec)$, we show that online gradient descent
has $O(\sqrt T)$ regret against every sufficiently regular feasible deviation
whose displacement is exact, $\Dphi=\nabla\Psi$, and obtain a uniform bound
over normalized families with common smoothness and potential-oscillation
parameters.  Thus the comparator is the actual endpoint $\phi(\xvec_t)$,
rather than an infinitesimal modification of the current action.  The
underlying Euclidean exactness principle can also be obtained by specializing
the more general first-order framework of \citet{ahunbay2024first} to
displacement fields generated by feasible endomorphisms.  Working directly
with endomorphisms, however, removes the tangent-projection machinery and
allows the boundary geometry to be handled directly through feasibility,
convexity, and normal cones.  This yields explicit finite-time bounds for feasible endomorphisms.

A simple example reveals the geometry behind this formulation.  Consider an
affine deviation $\phi(\xvec)=A\xvec+\bvec$, whose displacement is
$\Dphi(\xvec)=(I-A)\xvec-\bvec$.  When $I-A$ is symmetric, this displacement
is the gradient of a quadratic potential.  Thus the symmetric affine
deviations appearing in
\citet{ahunbay2025semicoarse,cai2025proximal} are finite maps with exact
displacement fields.  The example suggests that the relevant property is not
linearity itself, but exactness of the displacement.

The same structure appears in proximal deviations.  Let $F=f+\indicator_K$, with $\indicator_K$ denoting
the indicator function of $K$. If $F$ is weakly convex,
the Moreau-envelope identity gives
$\xvec-\prox_F(\xvec)=\nabla M_F(\xvec)$,
where $\prox_F$ and $M_F$ denote the proximal operator
and the Moreau envelope of $F$, respectively.  Proximal deviations therefore also have exact displacement fields.
From this viewpoint, proximality is one mechanism for producing exactness
rather than the underlying geometric property itself.  We show that the
smooth exact-form class strictly contains the weakly convex proximal class,
while every smooth Euclidean exact-form deviation admits a sufficiently small
interpolation with the identity that is proximal.

As it turns out, this characterization is tight, as witnessed by circulations. Exact one-forms have zero
circulation around closed loops, whereas nonzero circulation can be converted
into a bounded cyclic adversary with linear regret.  For feasible finite maps,
this gives an exactness-versus-circulation criterion under the regularity and
step-size conditions developed below.  Together with the positive result, this
identifies the geometric boundary of the finite deviation class controlled by
online gradient descent.

We next extend this finite-map viewpoint to mirror descent. The relevant
notion of exactness is determined by the mirror geometry, through the one-form
$\Dphi(\xvec)^{\transpose}\dd\nabla R(\xvec)$, where $R$ is the mirror regularizer.  We first develop the theory in
the Legendre interior, where the geometry is most transparent, and recover
Bregman proximal deviations as a special case.  We then remove the Legendre
assumption and obtain corresponding upper and lower bounds for standard
constrained mirror descent with boundary iterates.  This gives a finite-map
exactness-versus-circulation characterization in mirror geometry as well.

Bregman proximality leads to a separate representation problem.  Identity
interpolation always transfers proximal regret to the radial closure of the
Bregman proximal class, and this radial class is contained in the mirror-exact
class.  In Euclidean geometry the two classes coincide after radial closure,
but this can fail for a general mirror regularizer.  In the Legendre regime,
we characterize proximal interpolation through the convexity of
$R^*-\alpha\Psi$ ($R^*$ is the Fenchel conjugate of $R$) and the curvature of the reconstructed proximal generator.
A squared $\ell_p$ example shows that no positive interpolation may exist even
when the dual potential is smooth with globally Lipschitz gradient.

We also treat FTRL separately.  Mirror descent and FTRL coincide for
constant-step linearized losses in the Legendre interior, but their finite
deviation classes can separate once boundary effects are present.  We develop
the corresponding exactness-versus-circulation theory in the cumulative dual
state of FTRL and show that FTRL exactness implies mirror exactness, while the
converse can fail.  A powered $\ell_p$ construction gives a deviation with
$O(\sqrt T)$ regret under mirror descent and linear regret under FTRL.

Finally, we study the equilibrium notions induced by these finite deviation
classes.  We call the equilibrium notion associated with Euclidean exact-form
deviations \emph{conservative correlated equilibrium} (ConCE).  The name
reflects the fact that these deviations are generated by conservative, or
exact, displacement fields, and that players are correspondingly restricted
to this conservative family of deviations rather than arbitrary measurable
remappings.  Although smooth exact-form maps strictly contain weakly convex
proximal maps, their unrestricted equilibrium constraints, in convex games with compact action sets and jointly continuous losses, coincide with those of proximal correlated equilibrium (PCE) when the
approximation error is zero, so $\ConCE=\PCE$.  We further show that this
common notion coincides with full correlated equilibrium (CE) on finite support and,
under continuous losses, when every player's action space is an interval.  In
contrast, with two-dimensional individual action spaces we construct an
uncountably supported distribution with absolutely continuous marginals for
which $\CE\subsetneq\ConCE=\PCE$.

\paragraph{Contributions relative to prior work.}
Relative to \citet{ahunbay2024first}, our Euclidean analysis specializes the
first-order geometric framework to finite feasible endomorphisms and studies
the resulting $\Phi$-regret and equilibrium classes.  This
specialization yields explicit finite-time bounds by exploiting the additional
convex geometry available for endomorphisms and avoiding the tangent
projections and auxiliary geometric quantities needed in the general
infinitesimal formulation.  Relative to
\citet{ahunbay2025semicoarse}, our exact-form class extends the tractable
linear-quadratic family to nonlinear finite deviations while retaining the
symmetric affine maps as a special case.  Relative to
\citet{cai2025proximal}, we identify exactness as the broader finite-map
mechanism containing weakly convex proximal deviations, prove strict
map-level containment, and characterize the role of identity interpolation
in relating the two classes.

Beyond the Euclidean setting, we develop the finite-map theory for standard
constrained mirror descent without Legendre duality, separate mirror exactness
from Bregman proximal representability, and show that the Euclidean
interpolation phenomenon can fail in nonquadratic mirror geometry.  We also
develop a separate cumulative-dual theory for FTRL and exhibit a boundary
separation between the finite deviation classes controlled by FTRL and mirror
descent.  These algorithmic results are complemented by the interpolation and
equilibrium results above, which determine when distinctions between deviation
classes persist at the level of correlated-equilibrium constraints.

\section{Related Work} \label{sec:related_work}

The notion of \emph{regret} is one of the central concepts in the theory of online algorithms. In its simplest incarnation, regret measures competitiveness against the best \emph{fixed} comparator in hindsight. In other words, letting $K$ denote the set of possible decisions, $\xvec_t$ be the decision at time $t$, and $\langle \gvec_t, \cdot\rangle$ be the linear loss incurred at time $t$ by the decision maker, regret takes the form
\[
    \sup_{\widehat{\xvec} \in K} \sum_{t=1}^T \langle \gvec_t , \xvec_t - \widehat{\xvec}\rangle.
\]
As is standard, we refer to this notion of regret as \emph{external} regret. Since the seminal work of \citet{zinkevich2003online}, it is known that online gradient descent keeps external regret sublinear. Generalizations of this result to online mirror descent and follow-the-regularized-leader (FTRL) are by now standard \citep{shalev2012online,hazan2016introduction,orabona2019modern}, in particular optimistic algorithms and their variants for the \emph{self-play} setting in games~\citep{rakhlin2013online,rakhlin2013optimization,syrgkanis2015fast,chen2020hedging,daskalakis2021near,farina2022near,piliouras2022beyond,soleymani2025faster,soleymani2025cautious}. 

Over the past few decades researchers have been interested in obtaining guarantees for stronger notions of online competitiveness than external regret. Within the constellation of alternative notions (which includes concepts such as interval regret, dynamic regret, policy regret, switching regret, among others), the notion of \emph{$\Phi$-regret} \citep{greenwald2003general} is the most related to our work.

$\Phi$-regret is a generalization of external regret indexed by a set $\Phi$ of \emph{deviation functions} $\phi : K \to K$, comparing the sequence of online choices $\xvec_t$ to the transformed sequence $\phi(\xvec_t)$ across all $\phi \in \Phi$:
\[
    \sup_{\phi \in \Phi} \sum_{t=1}^T \langle \gvec_t, \xvec_t - \phi(\xvec_t)\rangle .
\]
External regret is then understood as the special instantiation of $\Phi$-regret in which the set $\Phi$ contains all \emph{constant} functions.
Partially motivated by game-theoretic applications and the celebrated wealth of connections between regret and equilibrium notions \citep{foster1997calibrated,hart2000simple,cesa2006prediction}, a long line of work has been concerned with characterizing the largest set of deviations $\Phi$ that admits efficient algorithms guaranteeing sublinear $\Phi$-regret (e.g., \citep{gordon2008no,farina2022simple,morrill2021efficient,fujii2025bayes,peng2024fast,dagan2024external} and references therein). Among the most general results in that space, it is now understood that $\Phi$-regret minimization with respect to all linear or low-degree polynomial endomorphic transformations of $K$ can be achieved efficiently for any convex and compact set $K$ given via efficient oracle access \citep{daskalakis2025efficient,zhang2025learning}. These algorithms are quite involved, and make use of complex separation and fixed-point machinery to deal with the geometry of endomorphisms.

While the common wisdom had been for a long time that simple algorithms such as online gradient descent only minimize external regret---and the more complex algorithms mentioned above are thus necessary to target more challenging variants of $\Phi$-regret---recent papers have challenged this view. Specifically, \citet{ahunbay2024first,ahunbay2025semicoarse,cai2024tractable,cai2025proximal} demonstrated that online gradient descent keeps under control far more than just external regret: for example, it minimizes regret against all \emph{symmetric} linear endomorphic transformations, and all proximal transformations of $K$. These results invite intriguing questions on exactly what notions of $\Phi$-regret online gradient descent minimizes, for at least two reasons: (i) notions of $\Phi$-regret induced by proximal transformations are not currently known to be achievable in other ways (for example, by low-degree polynomial transformations); and (ii) the extreme simplicity of online gradient descent compared to more advanced $\Phi$-regret techniques makes a compelling case that the former should be used whenever possible. The main goal of the present paper is to develop a finite feasible-map theory that explains a broad class of $\Phi$-regret guarantees given by these algorithms, including online gradient descent, mirror descent and FTRL.

Our Euclidean results are closely related to the first-order framework of
\citet{ahunbay2024first}.  After specializing that framework to displacement
fields generated by feasible endomorphisms, its qualitative conservativity
and circulation conclusions yield the corresponding exactness picture for
projected gradient descent.  We focus instead on $\Phi$-regret against
the actual endpoints $\phi(\xvec_t)$ and on the equilibrium notions induced by
these finite deviations.  The endomorphism structure also permits a more
direct quantitative analysis.  Feasibility, convexity, and normal-cone geometry
remove the tangent projections and auxiliary geometric quantities of the
general infinitesimal formulation. Relative to \citet{ahunbay2025semicoarse}, our exact-form class
extends the tractable linear--quadratic family to nonlinear feasible
deviations while retaining symmetric affine endomorphisms as a special case.

Our results also address an open question raised by
\citet{cai2025proximal}.  After showing that gradient descent minimizes
proximal regret, they ask for a tight characterization of the regret guarantees
of gradient descent and note that the class controlled by the algorithm may be
strictly broader than the proximal class.  We characterize this larger
finite-map class geometrically through exactness and circulation.  For online
gradient descent, smooth exact-form deviations strictly contain weakly convex
proximal deviations, while the circulation converse identifies the
corresponding obstruction to sublinear regret.  We develop the analogous
characterization for standard constrained mirror descent with differentiable
strongly convex regularizers, where the relevant object is the mirror one-form
$\Dphi(\xvec)^{\transpose}\dd\nabla R(\xvec)$.  Under our regularity
assumptions, this identifies the full geometric deviation class controlled by
mirror descent and goes beyond the Bregman proximal class studied by
\citet{cai2025proximal}.  We also study FTRL, which is not treated in
\citet{cai2025proximal}, and show that its controlled deviation class can
differ from that of mirror descent once boundary effects are present.  We
further show that proximal representability is a separate question.  In
Euclidean geometry, identity interpolation recovers the smooth exact-form
class, while in general mirror geometry the radial Bregman proximal class can
be strictly smaller than the mirror-exact class.  Finally, we derive the
equilibrium consequences of these finite deviation classes.  To our knowledge,
the resulting equilibrium results are new, including the equality of ConCE and
PCE under the unrestricted definitions, their coincidence with CE on finite
support and on interval action spaces under continuous losses, and their strict
separation from CE on a continuous support in higher dimension.

The gradient-equilibrium framework of
\citet{angelopoulos2025gradient} is closely connected to our potential-based
viewpoint.  In the unconstrained setting, their gradient-equilibrium condition
can be recovered from the linear potentials
$\Psi_{\vvec}(\xvec)=\inner{\vvec}{\xvec}$, whose exact displacements are the
constant vectors $\nabla\Psi_{\vvec}=\vvec$.  Our framework replaces these
constant directions by nonlinear finite displacements
$\Dphi(\xvec)=\nabla\Psi(\xvec)$ generated by feasible endomorphisms
$\phi\colon K\to K$.  In this sense, both approaches are driven by the same
telescoping mechanism, while our setting uses it to study finite
$\Phi$-regret and the equilibrium constraints induced by nonlinear
action-dependent deviations.  Their constrained formulation incorporates
feasibility through boundary or regularization residuals, whereas ours builds
feasibility directly into the deviation map.

\section{Setting and Notation}

Let \(K\subseteq\R^d\) be a nonempty closed convex action set.  Throughout
the Euclidean sections, gradients, projections, and normal cones are taken in
the ambient space $\R^d$.  When a statement is intrinsically
lower-dimensional, such as the curl criterion on $\relint(K)$, derivatives
and differential forms are taken in affine coordinates on $\aff(K)$.  The
mirror section introduces its own explicit affine-hull convention.  Vectors
are written in bold, e.g. \(\xvec\), \(\gvec\), and coordinates use brackets,
e.g. \(\xvec[j]\).  We write \([T]\colloneq\{1,\ldots,T\}\).  We use
$\norm{\cdot}$ as shorthand for the Euclidean norm $\norm{\cdot}_2$, and we
denote projection onto $K$ by $\Proj_K$.  Define the ambient normal cone at
any $\xvec\in K$ by
\[
\normal_K(\xvec)\colloneq\setof{\nvec\in\R^d}{\inner{\nvec}{\yvec-\xvec}\le 0\text{ for every }\yvec\in K}.
\]

A map $\phi:K\to\R^d$ is called feasible if
$\phi(K)\subseteq K$. A learner plays \(\xvec_t\in K\) and observes a vector loss \(\gvec_t\in\R^d\).  For a deviation \(\phi:K\to K\), define 
\[
  \Dphi(\xvec) \colloneq \xvec-\phi(\xvec),\qquad
  \Reg_T(\phi) \colloneq \sum_{t=1}^T \inner{\gvec_t}{\Dphi(\xvec_t)}
\]
More generally, for a displacement field \(\field:K\to\R^d\), define $
  \Reg_T(\field) \colloneq  \sum_{t=1}^T \inner{\gvec_t}{\field(\xvec_t)}.
$
When losses \(\ell_t:K\to\R\) are convex and differentiable, taking \(\gvec_t=\nabla \ell_t(\xvec_t)\) gives
\[
  \ell_t(\xvec_t)-\ell_t(\phi(\xvec_t))
  \le \inner{\nabla\ell_t(\xvec_t)}{\xvec_t-\phi(\xvec_t)}.
\]
Thus every linearized regret bound below immediately implies the corresponding convex-loss deviation-regret bound.

\begin{definition}[Exact-form deviation]\label{def:exact}
A deviation \(\phi:K\to K\) is \emph{Euclidean exact-form} if there exists a continuously differentiable potential \(\Psi\), defined on a neighborhood of \(K\), such that
\[
  \xvec-\phi(\xvec)=\nabla\Psi(\xvec)\qquad\text{for every }\xvec\in K.
\]
\end{definition}
When a result refers to a \emph{smooth exact-form deviation}, the gradient
of its potential is locally Lipschitz on a neighborhood of $K$, i.e., the potential is of
class $C^{1,1}_{\mathrm{loc}}$ there.  This is
stronger than the bare $C^1$ regularity in the definition above and is the
regularity used later to show that scaled exact deviations are proximal; see
\Cref{thm:scaled-prox}.

We denote the Euclidean exact-form class by
\[
  \PhiExact
  \colloneq
  \setof{\phi:K\to K}{\phi\text{ is a Euclidean exact-form deviation}}.
\]
Throughout the paper, the subscript identifies the deviation class.  Thus
$\PhiProx$ denotes the Euclidean proximal class,
$\PhiExactSmooth$ denotes the smooth exact-form class, and the
superscript $R$ is reserved for a regularizer-dependent mirror class.  When a
regularizer-dependent class is considered on more than one input domain, the
domain is displayed explicitly in parentheses.

Finally, we use the symbol $\osc_K$ to denote the range of a function on the set $K$, that is,
\[
  \osc_K(\Psi)\colloneq\sup_{\xvec\in K}\Psi(\xvec)-\inf_{\xvec\in K}\Psi(\xvec).
\]

Bounds for a fixed deviation may depend on the parameters of that deviation and therefore need not yield a uniform $\Phi$-regret bound. We therefore use the following normalized class whenever a supremum over deviations is taken.

\begin{definition}[Normalized exact-form class]\label{def:normalized-exact}
For constants $B,L\ge 0$, let $\PhiExact(B,L)$ consist of all
feasible maps $\phi:K\to K$ for which there is a differentiable potential
$\Psi_\phi$, defined on a neighborhood of $K$, satisfying
\[
  \xvec-\phi(\xvec)=\nabla\Psi_\phi(\xvec),\qquad
  \osc_K(\Psi_\phi)\le B,
\]
and whose gradient satisfies
\[
  \norm{\nabla\Psi_\phi(\yvec)-\nabla\Psi_\phi(\xvec)}
  \le L\norm{\yvec-\xvec}
  \quad\text{for all }\xvec,\yvec\in K.
\]
\end{definition}

Proofs not given in the main text are found in the appendices.
\Cref{app:preliminaries} records the convex-analytic preliminaries;
\Cref{app:scaled-prox-proof} proves the scaled proximal representation;
\Cref{app:finite-interpolation} proves the finite interpolation theorem; and
\Cref{app:proofs} contains the remaining proofs, organized in the same order as
the main text.

\section{Online Gradient Descent Minimizes Exact-form Regret}\label{sec:pgd}

Online gradient descent with constant step size \(\eta>0\) is
\[
\numberthis[OGD]{eq:pgd}
  \xvec_{t+1}\colloneq\Proj_K(\xvec_t-\eta\gvec_t).
\]

We prove a telescoping regret bound for every smooth feasible exact-form
deviation $\phi$. Under bounded gradients and bounded potential oscillation,
this gives $\Reg_T(\phi)=\mathcal O(\sqrt T)$.  Exactness of the displacement
field, rather than proximality, is the geometric property that makes the bound
 telescope on $\Psi$. Compared with the unconstrained case (online gradient
descent over an unbounded domain), the only additional issue is the projection
step.  The projection residual lies in the normal cone, and the feasibility
condition $\phi(K)\subseteq K$ ensures that its contribution has the favorable
sign.

\begin{theorem}\label{thm:pgd-exact}
Let $\mleft\{ \xvec_t \mright\}_{t=1}^{T+1}$ be the iterates generated by online gradient descent~\eqref{eq:pgd}. Let \(\phi:K\to K\) be an exact-form deviation $
  \Dphi(\xvec)=\xvec-\phi(\xvec)=\nabla\Psi(\xvec)$ where $\Psi$ is $L$-smooth with respect to the Euclidean norm. Then, for every sequence \(\gvec_1,\ldots,\gvec_T\),\footnote{Similar bounds can be inferred directly by \citet[Theorem 5.2]{ahunbay2024first}.}
\[
\numberthis{eq:pgd-exact-bound}
  \sum_{t=1}^T \inner{\gvec_t}{\xvec_t-\phi(\xvec_t)}
  \le
  \frac{\Psi(\xvec_1)-\Psi(\xvec_{T+1})}{\eta}
  +L\eta\sum_{t=1}^T\norm{\gvec_t}^2.
\]
\end{theorem}

The proof of \Cref{thm:pgd-exact} is given in
\Cref{app:proofs-pgd}.

\begin{corollary} \label{corr:regret}
Let $\mleft\{ \xvec_t \mright\}_{t=1}^{T+1}$ be the iterates generated by
online gradient descent~\eqref{eq:pgd} as in \Cref{thm:pgd-exact}.  If
\(K\) is compact, \(\norm{\gvec_t}\le G\), and
$\osc_K(\Psi)\le B$ for constants $B,L,G>0$, then
\[
  \Reg_T(\phi)
  \le \frac{B}{\eta}+L\eta G^2T.
\]
In particular, choosing \(\eta=\sqrt{B/(LG^2T)}\) gives
\(
  \Reg_T(\phi)\le 2G\sqrt{LBT} = \mathcal{O}(\sqrt{T}).
\)
\end{corollary}

\begin{corollary}\label{cor:uniform-regret}
Suppose $\norm{\gvec_t}\le G$ and fix $B,L>0$.  Then OGD with
$\eta=\sqrt{B/(LG^2T)}$ satisfies
\[
 \sup_{\phi\in\PhiExact(B,L)}\Reg_T(\phi)
 \le 2G\sqrt{LBT}.
\]
Thus the algorithm controls the whole normalized class $\PhiExact(B,L)$ with one learning rate.
\end{corollary}

\begin{remark}\label{remark:one_sided_smooth}
The proof of \Cref{thm:pgd-exact,corr:regret} uses the standard quadratic upper bound for $\Psi$ to control the potential-difference term and $L$-Lipschitz continuity of $\nabla\Psi$ to control the projection-residual term. Consequently, the normalized class in \Cref{def:normalized-exact} imposes the full Lipschitz-gradient condition. On an unconstrained domain, the projection residual vanishes, so the quadratic upper bound alone is sufficient.
\end{remark}

\section{Proximal Deviations are Exact-Form}\label{sec:containment}

This section proves that the exact-form deviation class includes the proximal
deviations of \citet{cai2025proximal}.  The key identity is the Moreau-envelope formula~\citep{moreau1965proximite}, which was also observed by \citet{ahunbay2024first}.
\[
  \xvec-\prox_F(\xvec)=\nabla \MF(\xvec). \numberthis{eq:envelope}
\]
These proximal-regret deviations are generated by $\rho$-weakly convex functions with $0\le\rho<1$, so we use the corresponding weakly-convex Moreau-envelope identity rather than only its convex special case.

More precisely, let \(F\colloneq f+\indicator_K\), where \(\indicator_K\) is the indicator of \(K\).  For \(0\le\rho<1\), define
\[
\prox_F(\xvec) \colloneq \argmin_{\yvec\in\R^d}\left\{F(\yvec)+\frac12\norm{\yvec-\xvec}^2\right\},
  \qquad
  \MF(\xvec) \colloneq \min_{\yvec\in\R^d}\left\{F(\yvec)+\frac12\norm{\yvec-\xvec}^2\right\}.
\]
Because \(F\) includes \(\indicator_K\), the minimization is over \(K\).

\begin{theorem}\label{thm:prox-contained}
Let \(F\colloneq f+\indicator_K\) be proper, lower semicontinuous, and \(\rho\)-weakly convex with \(0\le\rho<1\).  Define $\phi_F(\xvec)\colloneq\prox_F(\xvec)$. Then \(\phi_F:K\to K\) is feasible and Euclidean exact-form, with potential \(\MF\), i.e., $
  \xvec-\phi_F(\xvec)=\nabla\MF(\xvec)$, whose gradient is
  \(L_F=(2-\rho)/(1-\rho)\)-Lipschitz.
Consequently projected gradient descent controls every such proximal deviation by \Cref{thm:pgd-exact}.  In particular, if \(K\) is compact, \(\norm{\gvec_t}\le G\), \(\osc_K(\MF)\le B_F\), and $B_F,G>0$, then choosing \(\eta=\sqrt{B_F/(L_F G^2T)}\) gives
\[
  \sum_{t=1}^T \inner{\gvec_t}{\xvec_t-\prox_F(\xvec_t)}
  \le \frac{B_F}{\eta}+L_F\eta G^2T = \mathcal{O}(\sqrt{T}).
\]
If $B_F=0$ or $G=0$, the corresponding regret bound is trivial and the
displayed optimizing step size need not be used.
\end{theorem}

The Moreau-envelope identity immediately recovers several standard deviation classes.  The point of the following proposition is not to introduce new deviations, but to show that many of the deviations in \citet{cai2025proximal} have exact displacement fields, and therefore fall under our regret bound.

\begin{proposition}\label{prop:special-cases}
The exact-form class contains the following deviations.
\begin{enumerate}[leftmargin=*,itemsep=1mm]
  \item \textbf{External regret.}  For \(\uvec\in K\), take \(F\colloneq\indicator_{\{\uvec\}}\).  Then \(\prox_F(\xvec)=\uvec\) and \(\xvec-\uvec=\nabla\frac12\norm{\xvec-\uvec}^2\).
  \item \textbf{Projection to a convex subset.}  For a nonempty closed convex set \(S\subseteq K\), take \(F\colloneq\indicator_S\).  Then \(\prox_F(\xvec)=\Proj_S(\xvec)\).
  \item \textbf{Projection-based/no-move regret.}  For a vector \(\vvec\), take \(F(\yvec)\colloneq\inner{\vvec}{\yvec}+\indicator_K(\yvec)\).  Then \(\prox_F(\xvec)=\Proj_K(\xvec-\vvec)\).
  \item \textbf{Interpolation deviations.}  For \(\uvec\in K\) and \(\alpha\in(0,1)\), the map \(\phi(\xvec)\colloneq(1-\alpha)\xvec+\alpha\uvec\) is the prox map of
  \[
    F(\yvec)\colloneq\frac{\alpha}{2(1-\alpha)}\norm{\yvec-\uvec}^2+\indicator_K(\yvec).
  \]
  The case \(\alpha=1\) is external regret.
  \item \textbf{Local proximal deviations.}  If \(\norm{\xvec-\prox_F(\xvec)}\le\delta\) on \(K\), then the deviation belongs to the local exact-form subclass \(\norm{\nabla\MF(\xvec)}\le\delta\).
  \item \textbf{Symmetric affine swaps.}  
    Let \(\phi(\xvec)\colloneq A\xvec+\bvec\) map \(K\) into \(K\).
    Let \(T\colloneq\spanop(K-K)\) be the direction space of
    \(\aff(K)\), fix any \(\avec\in K\), and assume
    the restriction of \(I-A\) to \(T\) is self-adjoint. Note that \(AT\subseteq T\) follows from feasibility,
    since \(A(\xvec-\yvec)=\phi(\xvec)-\phi(\yvec)\in K-K\)
    for all \(\xvec,\yvec\in K\).
    Then \(\phi\) is exact-form in the ambient sense of our
    definition. On \(\aff(K)\), a potential is
    \[
    \numberthis{eq:symmetric-affine-potential}
      \Psi(\avec+\hvec)\colloneq
      \inner{\Dphi(\avec)}{\hvec}
      +\frac12\inner{\hvec}{(I-A)\hvec},
      \qquad \hvec\in T.
    \]
    An ambient extension of this potential is given in the proof.
\end{enumerate}
\end{proposition}

\begin{remark}
\Cref{thm:prox-contained,prop:special-cases} cover the Euclidean exact-form deviation classes appearing in \citet{cai2025proximal} including point indicators for external regret, linear functions for projection-based/no-move regret, bounded-displacement local proximal deviations, interpolation deviations, and symmetric affine swaps.  The symmetric affine case is even more direct in the exact-form view since the displacement of a symmetric affine map is already the gradient of a quadratic potential, without first representing the map as a proximal operator.
\end{remark}

\section{\texorpdfstring{The Exact-form Class Strictly Contains the Proximal Class}{The Exact-form Class Strictly Contains the Proximal Class}}\label{subsec:separation}

The inclusion from proximal deviations to exact-form deviations is strict, and the separation already appears in one dimension. Our proof strategy is to use the fact that every weakly convex proximal map with parameter $0\le\rho<1$ is monotone, whereas feasible exact-form deviations need not be monotone.

\begin{proposition}\label{prop:strict-separation}
There exists a smooth feasible exact-form deviation \(\phi:[0,1]\to[0,1]\) that is not \(\prox_F\) for any one-dimensional \(\rho\)-weakly convex function \(F\) with \(0\le\rho<1\).
\end{proposition}

\begin{proof}
Let \(k\ge2\) be an even integer and let \(a\in(0,1]\) satisfy \(a\pi k/2>1\).  Define
\[
  \Dphi(x)\colloneq a x(1-x)\sin(2\pi kx),
  \qquad
  \phi(x)\colloneq x-\Dphi(x).
\]
Because \(\abs{\Dphi(x)}\le a x(1-x)\le \min\{x,1-x\}\), we have \(\phi(x)\in[0,1]\).  Also \(\Dphi=\Psi'\), where
\[
  \Psi(x)\colloneq\int_0^x a s(1-s)\sin(2\pi k s)\,\dd s,
\]
so \(\phi\) is feasible exact-form.

At \(x=1/2\), since \(k\) is even,
\[
  \Dphi'(1/2)=\frac{a\pi k}{2}>1,
  \qquad
  \phi'(1/2)=1-\Dphi'(1/2)<0.
\]
Thus \(\phi\) is not monotone.  By \Cref{lem:weakly-convex-prox-monotone}, every one-dimensional weakly-convex proximal map with parameter \(0\le\rho<1\) is monotone.  Hence \(\phi\) cannot be represented as \(\prox_F\) for any such \(F\).
\end{proof}

Thus while exact-form maps strictly contain weakly-convex proximal maps,
we show that this strict inclusion does not carry over to the corresponding
equilibrium sets, in particular when the approximation error is zero. The following scaling result is the key mechanism behind this equivalence. We show that scaled exact deviations are proximal.

\begin{theorem}\label{thm:scaled-prox}
Let $K\subseteq\R^d$ be compact and convex, and let
$\phi(\xvec)\colloneq\xvec-\nabla\Psi(\xvec)$ be a feasible exact-form deviation with
$\Psi\in C^{1,1}_{\mathrm{loc}}$ on a neighborhood of $K$.  Then there exists
$\alpha_0>0$ such that, for every $\alpha\in(0,\alpha_0]$, the interpolated map
\[
 \phi_\alpha(\xvec)\colloneq(1-\alpha)\xvec+\alpha\phi(\xvec)
 =\xvec-\alpha\nabla\Psi(\xvec)
\]
coincides on $K$ with $\prox_{F_\alpha}(\xvec)$ for some proper lower-semicontinuous
$F_\alpha=f_\alpha+\indicator_K$ that is $\rho_\alpha$-weakly convex for some
$\rho_\alpha<1$.
\end{theorem}

The construction is closely related to the characterization of proximity
operators in terms of convex potentials developed by
\citet{gribonval2020characterization}. The construction and its proof are given in
\Cref{app:scaled-prox-proof}. It reconciles the strict inclusion
\[
  \PhiProx\subsetneq\PhiExactSmooth\subseteq\PhiExact.
\]
Although a smooth exact-form deviation need not itself be proximal, every such
deviation admits a sufficiently small interpolation with the identity that is
proximal.

Write $\PhiProx$ for the constrained weakly-convex proximal maps from
\Cref{sec:containment}, and write $\PhiExactSmooth$ for the feasible
exact-form maps whose potentials are $C^{1,1}_{\rm loc}$ near $K$.  For any
class $\mathcal A$ of feasible maps with a common domain, define its
\emph{identity-radial closure} by
\[
  \operatorname{Rad}_I(\mathcal A)
  \colloneq\left\{\phi\text{ feasible}:\ \text{for some }\alpha\in(0,1],\ 
  (1-\alpha)I+\alpha\phi\in\mathcal A\right\}.
\]

\begin{corollary}
\label{cor:euclidean-radial-closure}
For the radial closure of Euclidean proximal deviations, we have $
  \operatorname{Rad}_I(\PhiProx)
  =\PhiExactSmooth$.
\end{corollary}

Although \Cref{cor:euclidean-radial-closure} does not identify the two map classes, it shows
that their difference disappears after allowing interpolation with the identity.
Indeed, every smooth exact-form deviation admits a sufficiently small positive
interpolation that is proximal.  Thus, while $\PhiProx\subsetneq\PhiExactSmooth$, the strict separation between the two classes does not lead to a separation of
the corresponding equilibrium sets when the approximation error is zero.  The
formal equilibrium statement is deferred to \Cref{sec:games}.

\begin{remark}
\label{rem:scaled-prox-interpretation}
For
$\field_{\phi_\alpha}=\alpha\Dphi$, linearized regret scales exactly as $\Reg_T(\phi_\alpha)=\alpha\Reg_T(\phi)$. Moreover, convexity of a loss $\ell$ gives $
  \ell(\xvec)-\ell(\phi_\alpha(\xvec))
  \ge
  \alpha\bigl(\ell(\xvec)-\ell(\phi(\xvec))\bigr)$.
Thus a profitable smooth exact deviation remains profitable under every
positive interpolation with the identity, while \Cref{thm:scaled-prox} ensures
that a sufficiently small such interpolation is proximal.  The resulting
equivalence of the equilibrium constraints, and in particular $
  \ConCE=\PCE$ for convex games with compact action sets and jointly
continuous losses,
is stated formally in \Cref{thm:conce-pce-equality}.
\end{remark}

\section{Curl Induces Linear Regret for Online Gradient Descent}\label{sec:curl}

The upper bound telescopes because exact one-forms integrate to zero around
closed loops. We now show that nonzero circulation gives the
corresponding obstruction for finite deviation maps.  The construction is the
finite-map analogue of the nonconservative vector-field lower bound of
\citet[Proposition~5.4]{ahunbay2024first}; here we formulate it using a globally
feasible endpoint map and the regret convention of our setting.

\begin{definition}[Circulation]\label{def:circulation}
Let \(\field:K\to\R^d\) be continuous and let \(\gamma:[0,1]\to K\) be a closed piecewise \(C^1\) loop.  The circulation of \(\field\) along \(\gamma\) is
\[
  \oint_\gamma \field(\xvec)^{\transpose}\dd\xvec
  \colloneq\int_0^1 \inner{\field(\gamma(s))}{\gamma'(s)}\,\dd s.
\]
\end{definition}

We recall the classical fact that exact fields have zero circulation, which follows directly from the fundamental theorem of line integrals.

\begin{lemma}\label{lem:zero-circulation}
If \(\field=\nabla\Psi\) on a neighborhood of \(K\), then we have \(\oint_\gamma \field(\xvec)^{\transpose}\dd\xvec=0\) for every closed loop \(\gamma\subseteq K\).
\end{lemma}

Thus exactness rules out circulation.  The next theorem proves the converse algorithmic statement. Nonzero circulation can be converted into a bounded cyclic adversary and hence into linear regret.

\begin{theorem}\label{thm:curl-lower}
Let \(\field:K\to\R^d\) be continuous.  Suppose there exists a closed piecewise \(C^1\) loop \(\gamma:[0,1]\to K\) such that, after possibly reversing orientation,
\[
  -\oint_\gamma \field(\xvec)^{\transpose}\dd\xvec\colloneq c_0>0.
\]
Then for every gradient bound \(G>0\) and every sufficiently small \(\eta>0\), there exists a sequence of losses with \(\norm{\gvec_t}\le G\) such that online gradient descent, initialized at \(\gamma(0)\), suffers
\[
  \sum_{t=1}^T \inner{\gvec_t}{\field(\xvec_t)}\ge cT-O(1/\eta)
\]
for a constant \(c>0\) depending only on \(\field\), \(\gamma\), and \(G\).  In particular the regret is \(\Omega(T)\).  If the algorithm starts elsewhere in the same path-connected component of \(K\), the adversary can first steer it to the loop along any fixed polygonal path; this changes only the \(O(1/\eta)\) transient term.
\end{theorem}

The same lower-bound mechanism also extends to standard anytime learning-rate schedules.
\begin{remark}
\label{rem:stepsizes}
\Cref{thm:curl-lower} is stated for a sufficiently small constant step size,
which may depend on the witnessing loop.  This is slightly different from the
usual $O(\sqrt T)$ upper bound, which uses the known-horizon choice
$\eta\asymp T^{-1/2}$.  The lower-bound construction can nevertheless be made
anytime by using geometric epochs: during an epoch of length $2^k$, use a
constant step size of order $2^{-k/2}$.  Within each epoch, the loop
construction applies after its $O(1/\eta)$ steering transient.  Across epochs,
the positive main terms sum to $\Omega(T)$, whereas the transient terms sum
only to $O(\sqrt T)$.  Thus the linear lower bound is not an artifact of
knowing the horizon in advance. 
\end{remark}

Together, the upper and lower bounds give the following Euclidean curl
criterion for finite deviation maps.  Curl-free fields admit a potential and,
under the feasibility and smoothness conditions below, yield sublinear regret.
By contrast, nonzero curl produces a local loop along which gradient descent
can be forced to incur linear regret.  The two directions involve slightly
different learning-rate assumptions.  The lower bound uses a sufficiently
small constant step size that may depend on the witnessing loop, while the
uniform upper bound over a normalized family is given in
\Cref{cor:uniform-regret}. The following corollary connects the classical
geometric characterization of conservative fields
via the Poincar\'e lemma~\citep{spivak1965calculus}
to regret guarantees. Note that $\relint(K)$ is
simply connected by convexity.

\begin{corollary}[Euclidean curl criterion]\label{cor:euclidean-classification}
Let $E_K\colloneq\spanop(K-K)$, let $U\colloneq\relint(K)$ be open and simply connected in
$\aff(K)$, and let $\field:U\to E_K$ be $C^1$.  In any orthonormal affine
coordinate system on $E_K$, $\partial_j\field_i=\partial_i\field_j$ throughout
$U$ if and only if $\field=\nabla_{E_K}\Psi$ on $U$ for a $C^2$ potential
$\Psi$. If a skew derivative is nonzero somewhere, online gradient descent can be
forced to incur linear regret against $\field$ along a sufficiently small
feasible loop.  When the field is curl-free, assume that $\Psi$ admits an
ambient extension $\widetilde\Psi$, defined on a neighborhood of $K$, that
satisfies the smoothness and bounded-oscillation assumptions of
\Cref{thm:pgd-exact,corr:regret}.  If
$\widetilde\phi(\xvec)\colloneq\xvec-\nabla\widetilde\Psi(\xvec)$ lies in $K$ for
every $\xvec\in K$, then projected GD has $O(\sqrt T)$ regret against
$\widetilde\phi$.
\end{corollary}

\begin{remark}
\citet[Theorem~1.2]{ahunbay2024first} identifies two requirements for
first-order vector-field tests: the field must be conservative and tangent to
the action set.  In the finite-map setting, the second requirement is already
built into feasibility.  Indeed, if $\phi(\xvec)\in K$, convexity of $K$
ensures that the segment from $\xvec$ to $\phi(\xvec)$ remains in $K$.
Exactness is then the finite-map counterpart of the conservative-field
condition.  The circulation obstruction used above is therefore closely
related to the nonconservative-field lower bound of
\citet[Proposition~5.4]{ahunbay2024first}, while here the deviation is specified
by a feasible endpoint map.
\end{remark}

\section{Online Mirror Descent and Exact One-forms}\label{sec:mirror}

Our Euclidean result in \Cref{sec:pgd} says that online gradient descent controls deviations whose displacement is a gradient.  Mirror descent changes the coordinates in which the same statement is true.  Its natural variable is not the primal point \(\xvec\), but the dual point
\[
  \thetavec\colloneq\nabla R(\xvec),
\]
for a distance-generating function $R$.  At an interior iterate, where the
normal-cone term vanishes, mirror descent is the dual update
\[
  \thetavec_{t+1}=\thetavec_t-\eta\gvec_t.
\]
Therefore, as expected we show that the same telescoping proof works whenever the primal displacement \(\Dphi(\xvec)\) is a gradient as a function of the dual coordinate \(\thetavec\).  This way, we are able to generalize our results to mirror-exact deviations.

To make the mirror generalization transparent, we first work in the
Legendre interior $\mathcal D=\relint(K)$.  In this setting, $\nabla R$ provides a smooth change of
coordinates between the primal and dual spaces, and interior mirror-descent
iterates satisfy an additive update in the dual variable.  This removes
boundary normal-cone terms and makes the connection with the Euclidean
argument particularly clear.  The first three subsections use this setting to
introduce mirror-exact deviations, prove the corresponding dual-potential
regret bound, and recover Bregman proximal maps of \citet{cai2025proximal} as a special case.

We then return to standard constrained mirror descent in
\Cref{sec:md-beyond-legendre} and remove the Legendre assumption.  The iterates
may now reach the boundary, and the mirror map need not provide a globally
invertible dual coordinate system.  A normal-cone residual therefore appears
in the mirror-descent optimality condition.  The underlying geometric
condition, however, remains the same.  Exactness of the mirror one-form gives
the upper bound after the boundary residual is controlled, while nonzero
circulation gives the corresponding lower bound.  This shows that the
exactness principle is not a consequence of Legendre duality, even though the
Legendre setting provides the cleanest way to see it initially.

With the regret characterization in place,
\Cref{sec:mirror-interpolation} returns to the identity interpolation of
Theorem~\ref{thm:scaled-prox}.  We keep the original weakly convex Bregman
proximal class and use the scaling identity to transfer proximal regret to its
radial closure.  The Euclidean conclusion does not extend in full generality.
The radial proximal class is always contained in the mirror-exact class, but
the containment can be strict.  In the Legendre setting, convexity of
$R^*-\alpha\Psi$ gives the first obstruction to proximal representation, and
an actual weakly convex proximal generator requires an additional curvature
condition.  This separates the part of the Euclidean interpolation argument
that is purely algebraic from the part that depends on the curvature of the
regularizer.

There is a further distinction once one leaves the Legendre interior.  For
constant-step linearized losses, Legendre mirror descent has an additive dual
recursion, and unrolling that recursion gives the corresponding follow-the-regularized-leader (FTRL) update.
This relationship between mirror descent, FTRL, and closely related
dual-averaging methods is well known in online optimization
\citep{mcmahan2011follow,mcmahan2017survey,xiao2010dual}.  At the boundary,
however, standard constrained mirror descent and FTRL need not generate the
same dynamics.  Mirror descent performs a one-step Bregman update and may
carry a normal-cone residual, whereas FTRL retains the accumulated linear
losses in a single regularized minimization.  The distinction also affects
which deviations the two algorithms control.  In \Cref{sec:ftrl}, we show
that their exact deviation classes can differ and give a smooth powered
$\ell_p$ example for which mirror descent has $O(\sqrt T)$ regret while FTRL
can incur linear regret.  For this reason, we study FTRL separately in
\Cref{sec:ftrl}.

Throughout the first three subsections of this section, all convex-analytic objects are taken in the affine
hull of $K$.  After fixing an origin, we identify this affine hull with a
Euclidean vector space $E$, equipped with the induced inner product.
These three subsections work in the standard Legendre regime.  Let
$R:E\to(-\infty,+\infty]$ be a proper closed Legendre function satisfying
\[
 \overline{\dom R}=K,
 \qquad
 \mathcal D\colloneq\operatorname{int}_E(\dom R)=\relint(K).
\]
Assume that $R$ is $C^2$ on $\mathcal D$ and that $\nabla^2R(\xvec)$ is
positive definite there.  Let
\[
 R^*(\thetavec)\colloneq\sup_{\xvec\in E}
 \{\inner{\thetavec}{\xvec}-R(\xvec)\},
 \qquad
 \ThetaSet\colloneq\operatorname{int}_E(\dom R^*)=\nabla R(\mathcal D).
\]
We use the Bregman divergence generated by \(R\), defined as
\[
  \DR(\yvec,\xvec)
  \colloneq R(\yvec)-R(\xvec)-\inner{\nabla R(\xvec)}{\yvec-\xvec}.
\]
We use analogous notation for the Bregman divergence
generated by other differentiable functions. Legendre duality gives mutually inverse $C^1$ diffeomorphisms
$\nabla R:\mathcal D\to\ThetaSet$ and
$\nabla R^*:\ThetaSet\to\mathcal D$.  In particular,
\[
 D_{R^*}(\thetavec',\thetavec)=\DR(\xvec,\xvec')
 \quad\text{when }\thetavec=\nabla R(\xvec),\quad
 \thetavec'=\nabla R(\xvec').
\]
We initialize mirror descent in $\mathcal D$. Under the standing Legendre
assumptions, attained mirror-descent subproblems remain in the Legendre
interior, so the iterates stay in $\mathcal D$.  We write
$\thetavec\colloneq\nabla R(\xvec)$.

With this notation, online mirror descent takes the form
\[
\numberthis{eq:md}
  \xvec_{t+1}
  \colloneq\argmin_{\xvec\in K}
  \left\{
    \eta\inner{\gvec_t}{\xvec}+\DR(\xvec,\xvec_t)
  \right\}.
\]

\begin{definition}[Mirror-exact deviation]\label{def:mirror-exact}
A deviation \(\phi:\mathcal D\to K\), with displacement
\(\Dphi(\xvec)\colloneq\xvec-\phi(\xvec)\), is \emph{\(R\)-mirror-exact} if there is
a differentiable potential \(\Psi:\ThetaSet\to\R\) such that
\[
\Dphi(\xvec)=\gradtheta\Psi(\nabla R(\xvec))
  \qquad\forall \xvec\in \mathcal D.
\]
Equivalently, the one-form
$\Dphi(\xvec)^{\transpose}\dd\nabla R(\xvec)$ is exact.
\end{definition}
We write
$\PhiMirrorExact{R}(\mathcal D)$ for the class of mirror-exact deviations in
the Legendre interior.

\begin{remark}
If $K=E$ and \(R(\xvec)=\frac12\norm{\xvec}^2\), then
\(\nabla R(\xvec)=\xvec\).  Mirror-exactness becomes
\(\Dphi(\xvec)=\nabla\Psi(\xvec)\), exactly the Euclidean exact-form condition
from \Cref{sec:pgd} on the unconstrained action space.
\end{remark}

The next proposition gives a useful primal-coordinate test for this condition. It shows that mirror-exactness is equivalent to ordinary exactness after multiplying the displacement by the mirror metric \(\nabla^2R(\xvec)\).

\begin{proposition}\label{prop:one-form-equivalence}
A field \(\field:\mathcal D\to E\) is \(R\)-mirror-exact if and only if
there exists a differentiable scalar potential \(V:\mathcal D\to\R\) such that
\[
\numberthis{eq:V-gradient}
  \nabla V(\xvec)=\nabla^2R(\xvec)\field(\xvec).
\]
Equivalently, $\field(\xvec)^{\transpose}\dd\nabla R(\xvec)=\dd V(\xvec)$.
\end{proposition}

\subsection{Examples of Mirror-exact Deviations}\label{subsec:mirror-exact-examples}

Definition~\ref{def:mirror-exact} can look abstract because it is written in dual coordinates.  The quickest way to build examples is simple: ``choose a scalar potential $\Psi(\thetavec)$ in dual space, differentiate it to get $\nabla_{\thetavec} \Psi(\thetavec)$, and then pull the resulting field back through \(\thetavec=\nabla R(\xvec)\) to get $\phi(\xvec) = \xvec - \nabla_{\thetavec} \Psi\mleft(  \nabla R(\xvec)\mright) $.''

In \Cref{app:mirror-exact-examples}, we provide several examples of mirror-exact deviations. They show that the mirror-exact deviation family contains every fixed-comparator deviation. They also give nontrivial deviations for different instances of online mirror descent, such as, for the multiplicative weights update algorithm, multiplicative deviations with arbitrary temperature vector $\avec \in \R^d$,
\[
  \phi_{\avec}(\xvec)[i]
  \colloneq\frac{\xvec[i]\exp(-\avec[i])}
  {\sum_{j=1}^d \xvec[j]\exp(-\avec[j])},
\]
and geometric interpolation deviations anchored at $\uvec\in\relint(\simplex_d)$,
for $\lambda\in[0,1)$,
\[
\phi_{\lambda,\uvec}(\xvec)[i]
  =\frac{\xvec[i]^{1-\lambda}\uvec[i]^\lambda}
  {\sum_{j=1}^d \xvec[j]^{1-\lambda}\uvec[j]^\lambda};
\]
and there exist mirror-exact deviations that are not Euclidean-exact.

\subsection{Online Mirror Descent Controls Mirror-exact Regret}\label{sec:omd}

We now prove no-regret results for mirror-exact deviations under online mirror
descent in the Legendre interior. For a differentiable function \(A\) on dual space, we write its Bregman divergence as
\[
  D_A(\thetavec',\thetavec)
  \colloneq A(\thetavec')-A(\thetavec)
  -\inner{\nabla A(\thetavec)}{\thetavec'-\thetavec},
\]
and, for $L\ge0$, we say that the potential \(\Psi\) is \(L\)-smooth relative to \(R^*\) on a set of ordered dual pairs if
\[
\numberthis{eq:relative-smoothness}
  D_\Psi(\thetavec',\thetavec)
  \le L D_{R^*}(\thetavec',\thetavec)
\]
for every such pair $\thetavec',\thetavec \in \ThetaSet$.  When $K=E$ and
\(R(\xvec)=\frac12\norm{\xvec}^2\), this is the usual upper quadratic
inequality.  For a general mirror map, it
measures the one-step Taylor error of \(\Psi\) in the same geometry used by
mirror descent.

\begin{theorem}\label{thm:constrained-md}
Assume the standing Legendre-domain conditions and that $R$ is $m$-strongly
convex on $\mathcal D$.  Let \(\phi:\mathcal D\to K\) be feasible and
\(R\)-mirror-exact with dual potential \(\Psi\). Let $\{ \xvec_t\}_{t=1}^{T+1}$ be the
iterates generated by online mirror descent~\eqref{eq:md}, initialized in
$\mathcal D$, with iterates remaining in $\mathcal D$. Write
\[
  \thetavec_t\colloneq\nabla R(\xvec_t),
  \qquad
  \Dphi(\xvec_t)=\gradtheta\Psi(\thetavec_t).
\]
Assume that \eqref{eq:relative-smoothness} holds for every consecutive pair
$(\thetavec_{t+1},\thetavec_t)$, $t\in[T]$. A sufficient
trajectory-independent condition is that it hold for every ordered pair in
$\ThetaSet\times\ThetaSet$.
Then
\[
\numberthis{eq:constrained-md-bound}
\sum_{t=1}^T\inner{\gvec_t}{\Dphi(\xvec_t)}
  \le
  \frac{\Psi(\thetavec_1)-\Psi(\thetavec_{T+1})}{\eta}
  +
  \frac{L\eta}{m}
  \sum_{t=1}^T\norm{\gvec_t}^2.
\]
Consequently, suppose \(\norm{\gvec_t}\le G\) and that, for the
horizon-dependent choices of constant step size considered below, there is a
constant $B$, independent of the horizon, such that
\[
 \max_{1\le t\le T+1}\Psi(\thetavec_t)
 -\min_{1\le t\le T+1}\Psi(\thetavec_t)\le B,
\]
Then choosing \(\eta\asymp 1/\sqrt T\) gives \(O(\sqrt T)\) regret against
\(\phi\).  More precisely, when $B,L,G>0$, the choice
$\eta=\sqrt{mB/(LG^2T)}$ gives the bound
$2G\sqrt{LBT/m}$.  The global condition
$\osc_{\ThetaSet}(\Psi)\le B$ is a sufficient trajectory-independent
assumption that avoids any dependence of the oscillation bound on the chosen
step size.
\end{theorem}

The coefficient $L/m$ in this bound uses the standard Legendre-regime
assumption that every iterate remains in $\relint(K)$.  The relative normal
cone then vanishes, so the proof works directly with the unconstrained dual
update and introduces no additional projection term.  Boundary iterates are
treated separately in \Cref{sec:md-beyond-legendre}, where we remove the
Legendre assumption and analyze standard constrained mirror descent directly.

The circulation converse is also developed in
\Cref{sec:md-beyond-legendre}.  In particular,
\Cref{thm:md-mirror-circulation} applies directly to standard constrained
mirror descent, allows boundary iterates, and gives the corresponding lower
bound without relying on the interior dual-coordinate recursion.

\begin{remark}
For the regret proof, it is enough that
\eqref{eq:relative-smoothness} hold along the realized
trajectory, namely for $(\thetavec_{t+1},\thetavec_t)$
at each round. Imposing the condition on all of
$\ThetaSet\times\ThetaSet$ gives a trajectory-independent
sufficient condition, but may be considerably stronger
when $\ThetaSet$ is unbounded. More generally, the same proof yields
$O(\sqrt{T})$ regret under the remaining assumptions of
\Cref{thm:constrained-md} if
\[
D_\Psi(\thetavec_{t+1},\thetavec_t)
\le \frac{L'}{2}
\dualnorm{\thetavec_{t+1}-\thetavec_t}^2
\]
for a constant $L'\ge0$.
By the $1/m$-smoothness of $R^*$,
\eqref{eq:relative-smoothness} implies this weaker
condition with $L'=L/m$.
\end{remark}

\subsection{Bregman Proximal Regret is Mirror-exact}\label{sec:bregman}

We next show that the Bregman proximal deviations of \citet{cai2025proximal}
are mirror-exact.  A related Hessian-weighted integrability observation and discussion appear
in \citet[Appendix~B.2.2]{ahunbay2024first}, but there the argument requires a
steep, smooth regularizer and does not recover the full Bregman proximal
guarantee of \citet{cai2025proximal}.  Here we first work in the Legendre
interior and later, in \Cref{sec:md-beyond-legendre}, remove the Legendre
assumption and allow boundary iterates.  The proof is the mirror version of the
Moreau-envelope identity~\eqref{eq:envelope}.

Let \(F\colloneq f+\indicator_K\) be proper, lower semicontinuous, and
$\rho$-weakly convex for some $\rho\ge0$.  The indicator is included so that all proximal outputs
lie in \(K\), giving pointwise feasibility automatically.  For
$\xvec\in\mathcal D$, define
\[
  \prox_F^R(\xvec)
  \colloneq \argmin_{\yvec\in E}
    \left\{F(\yvec)+\DR(\yvec,\xvec)\right\}, \quad 
  \MRF(\xvec)
  \colloneq \min_{\yvec\in E}
    \left\{F(\yvec)+\DR(\yvec,\xvec)\right\}.
\]
Because \(F\) contains \(\indicator_K\), the minimization is equivalently over \(\yvec\in K\).

\begin{theorem}\label{thm:bregman-contained}
Assume the standing Legendre-domain conditions.  Let
$F\colloneq f+\indicator_K$ be proper and lower semicontinuous, assume that every
Bregman proximal minimum above is attained, and suppose that \(\prox_F^R\) is
single-valued and continuous on $\mathcal D$.  Then the Bregman proximal
deviation $\phi_F^R:\mathcal D\to K$ defined by
$\phi_F^R(\xvec)\colloneq\prox_F^R(\xvec)$
is feasible and \(R\)-mirror-exact.  Equivalently,
\[
\numberthis{eq:bregman-one-form-exact}
  \big(\xvec-
  \prox_F^R(\xvec)\big)^{\transpose}\dd\nabla R(\xvec)
  =\dd\MRF(\xvec).
\]
Thus Bregman proximal regret is a special case of mirror-exact regret.
\end{theorem}

We can now directly invoke \Cref{thm:constrained-md} for this class of deviations in the following proposition.

\begin{proposition}\label{prop:bregman-relative-smooth}
Assume the standing Legendre-domain conditions.  Assume, in the
extended-valued sense on $E$, that $R$ is $m$-strongly convex and that
$F\colloneq f+\indicator_K$ is proper, closed, and $\rho$-weakly convex, where
$0\le\rho<m$ and $H\colloneq F+R$ is proper.  Then the Bregman proximal subproblem has a
unique solution on $\mathcal D$, and its solution map is continuous there.
Let $\{\xvec_t\}_{t=1}^{T+1}\subset\mathcal D$ be generated by online mirror
descent~\eqref{eq:md}.  Then these iterates control Bregman proximal regret with
the bound
\[
\numberthis{eq:bregman-prox-md-bound}
  \sum_{t=1}^T
  \inner{\gvec_t}{\xvec_t-\prox_F^R(\xvec_t)}
  \le
  \frac{\Psi_F(\thetavec_1)-\Psi_F(\thetavec_{T+1})}{\eta}
  +\frac{\eta}{m}\sum_{t=1}^T\norm{\gvec_t}^2,
\]
where $\Psi_F(\thetavec_t) = \MRF(\xvec_t)$ for all $t \in [T+1]$.
\end{proposition}

The key idea behind this result is to show that weakly convex Bregman proximal potentials are $(L=1)$-relatively smooth with respect to $R^*$, as defined in \eqref{eq:relative-smoothness}, and hence \Cref{thm:constrained-md} applies.

\subsection{Full Mirror Characterization Beyond Legendre Regularizers}
\label{sec:md-beyond-legendre}

The preceding subsections use Legendre duality to regard
$\thetavec\colloneq\nabla R(\xvec)$ as a global coordinate and to eliminate the
normal cone from the mirror-descent optimality condition.  The geometric
condition itself does not depend on either property.  It is exactness of the
one-form
\[
  \Dphi(\xvec)^{\transpose}\dd\nabla R(\xvec).
\]
We now keep the standard constrained mirror-descent update
\eqref{eq:md}, allow boundary iterates, and prove almost matching upper and lower bounds without assuming that
$\nabla R$ is onto or that $R$ is Legendre.\footnote{The upper bound requires a Lipschitz displacement and a
relative-curvature bound on the potential; the circulation lower
bound requires only continuity of the field, with the loop
$\gamma$ and its image $\nabla R\circ\gamma$ both piecewise $C^1$.}

Let $K$ be a compact convex subset of a finite-dimensional normed space
$(E,\norm{\cdot})$, with dual norm $\dualnorm{\cdot}$.  Throughout this
subsection, $R$ is differentiable on a neighborhood of $K$ and is
$m$-strongly convex on $K$:
\[
  R(\yvec)\ge R(\xvec)+\inner{\nabla R(\xvec)}{\yvec-\xvec}
  +\frac m2\norm{\yvec-\xvec}^2
  \qquad(\xvec,\yvec\in K).
\]
The gradient is uniformly continuous on $K$.  Its modulus of continuity is
\[
  \omega_R(\delta)
  \colloneq\sup\left\{
  \dualnorm{\nabla R(\xvec)-\nabla R(\yvec)}:
  \xvec,\yvec\in K,\ \norm{\xvec-\yvec}\le\delta
  \right\}.
\]
Thus $\omega_R(\delta)\to0$ as $\delta\downarrow0$.

The same exactness condition extends beyond the Legendre setting by defining
the potential directly on a neighborhood of $\nabla R(K)$.

\begin{definition}
\label{def:mirror-exact-beyond-legendre}
A feasible deviation $\phi:K\to K$ is \emph{$R$-mirror-exact} if there are
an open convex set $\Omega\supseteq\nabla R(K)$ and a differentiable
potential $\Psi:\Omega\to\R$ such that
\[
  \Dphi(\xvec)=\gradtheta\Psi(\nabla R(\xvec))
  \qquad(\xvec\in K).
\]
\end{definition}

We denote this class by $\PhiMirrorExact{R}(K)$.  Under the Legendre
assumptions of the preceding subsections, the definition reduces to
Definition~\ref{def:mirror-exact} on $\mathcal D$.  If $R$ is twice
differentiable and $V=\Psi\circ\nabla R$, then
\[
  \nabla V(\xvec)=\nabla^2R(\xvec)\Dphi(\xvec),
\]
so \Cref{def:mirror-exact-beyond-legendre} is equivalent to
exactness of $\Dphi(\xvec)^{\transpose}\dd\nabla R(\xvec)$.

The discrete potential argument requires a one-sided curvature bound.  We
assume that, for some $L\ge0$,
\[
\numberthis{eq:relative-smoothness-beyond-legendre}
  D_\Psi\bigl(\nabla R(\yvec),\nabla R(\xvec)\bigr)
  \le L \DR(\xvec,\yvec)
  \qquad(\xvec,\yvec\in K).
\]
Equivalently,
\[
\begin{aligned}
  \Psi(\nabla R(\yvec))
  &\le \Psi(\nabla R(\xvec))
  +\inner{\Dphi(\xvec)}
  {\nabla R(\yvec)-\nabla R(\xvec)}
  +L \DR(\xvec,\yvec).
\end{aligned}
\]
In the Legendre interior,
$\DR(\xvec,\yvec)=D_{R^*}(\nabla R(\yvec),\nabla R(\xvec))$, and this
condition becomes the relative-smoothness assumption in
\Cref{thm:constrained-md}.

These ingredients yield the corresponding exact-form regret bound for standard
constrained mirror descent.
\begin{theorem}
\label{thm:md-exact-beyond-legendre}
Assume $R$ is differentiable on a neighborhood of $K$, and let $\{\xvec_t\}_{t=1}^{T+1}$ be generated by mirror
descent \eqref{eq:md} with a constant step size $\eta>0$.  Let $\phi:K\to K$ be
$R$-mirror-exact with potential $\Psi$.  Suppose
\eqref{eq:relative-smoothness-beyond-legendre} holds and $\Dphi$ is
$L_\phi$-Lipschitz on $K$.  Then
\[
\begin{aligned}
  \sum_{t=1}^T\inner{\gvec_t}{\Dphi(\xvec_t)}
  &\le
  \frac{\Psi(\nabla R(\xvec_1))-\Psi(\nabla R(\xvec_{T+1}))}{\eta}
  +\frac{L\eta}{m}\sum_{t=1}^T\dualnorm{\gvec_t}^2\\
  &\quad+
  \frac{L_\phi}{m}\sum_{t=1}^T
  \dualnorm{\gvec_t}
  \left[
  \eta\dualnorm{\gvec_t}
  +\omega_R\!\left(\frac{\eta}{m}\dualnorm{\gvec_t}\right)
  \right].
\end{aligned}
\]
\end{theorem}

The first line is the same potential estimate as in the Legendre interior.
The second line controls the boundary normal-cone residual.  Feasibility
gives the correct sign at $\xvec_{t+1}$, and Lipschitz continuity of the
displacement transfers that sign to the displacement evaluated at
$\xvec_t$.

The preceding bound immediately gives sublinear regret under bounded gradients
and bounded potential oscillation.  When $\nabla R$ is Lipschitz on $K$, it
also recovers the usual $O(\sqrt T)$ rate.

\begin{corollary}
\label{cor:md-exact-beyond-legendre-rate}
Assume the hypotheses of \Cref{thm:md-exact-beyond-legendre},
$\dualnorm{\gvec_t}\le G$, and
$\osc_{\nabla R(K)}(\Psi)\le B$.  For $\eta=T^{-1/2}$,
\[
  \frac1T\sum_{t=1}^T\inner{\gvec_t}{\Dphi(\xvec_t)}
  \le
  \frac{B}{\sqrt T}
  +\frac{(L+L_\phi)G^2}{m\sqrt T}
  +\frac{L_\phi G}{m}
  \omega_R\!\left(\frac{G}{m\sqrt T}\right).
\]
Consequently, mirror descent has $o(T)$ regret against $\phi$.  If
$\nabla R$ is $M$-Lipschitz on $K$, then
\[
  \sum_{t=1}^T\inner{\gvec_t}{\Dphi(\xvec_t)}
  \le \frac{B}{\eta}+C_{R,\phi}\eta G^2T,
  \qquad
  C_{R,\phi}\colloneq\frac{L+L_\phi}{m}+\frac{L_\phi M}{m^2}.
\]
The optimized bound is
$2G\sqrt{B C_{R,\phi}T}$ whenever $B C_{R,\phi}>0$.
\end{corollary}

We next establish necessity of exactness at the geometric level and show how nonzero circulation obstructs
sublinear regret.  Let $\field:K\to E$ be continuous and let
$\gamma:[0,1]\to K$ be a closed piecewise $C^1$ loop such that
$\nabla R\circ\gamma$ is piecewise $C^1$.  The mirror circulation of
$\field$ along $\gamma$ is
\[
\numberthis{eq:mirror-circulation-beyond-legendre}
  \oint_\gamma \field(\xvec)^{\transpose}\dd\nabla R(\xvec)
  \colloneq
  \int_0^1
  \inner{\field(\gamma(s))}
  {\frac{\dd}{\dd s}\nabla R(\gamma(s))}\,\dd s.
\]
By the chain rule, every $R$-mirror-exact field has zero mirror circulation.
Conversely, nonzero circulation along a feasible loop can be exploited by an
adversary to make standard constrained mirror descent repeatedly traverse the
loop and accumulate linear regret.

\begin{theorem}
\label{thm:md-mirror-circulation}
Suppose that, after reversing the orientation of $\gamma$ if necessary,
\[
  -\oint_\gamma
  \field(\xvec)^{\transpose}\dd\nabla R(\xvec)\colloneq c_0>0.
\]
For every $G>0$ and every sufficiently small constant step size $\eta>0$,
there is a sequence with $\dualnorm{\gvec_t}\le G$ such that standard
constrained mirror descent, initialized at $\gamma(0)$, satisfies
\[
  \sum_{t=1}^T\inner{\gvec_t}{\field(\xvec_t)}
  \ge cT-O(1/\eta)
\]
for a constant $c>0$ depending only on $\field$, $\gamma$, $R$, and $G$.
\end{theorem}

The adversary discretizes the loop in the variables $\nabla R(\xvec)$.
For two consecutive points $\zvec_j,\zvec_{j+1}$, it chooses
\[
  \gvec_{j+1}
  =\frac{\nabla R(\zvec_j)-\nabla R(\zvec_{j+1})}{\eta}.
\]
The derivative of the mirror-descent objective vanishes at
$\zvec_{j+1}$.  Strong convexity therefore makes $\zvec_{j+1}$ the unique
constrained minimizer, including when it lies on the boundary.  One
traversal accumulates a Riemann sum for the negative circulation divided
by $\eta$.  The learning-rate quantifiers are the same as in
\Cref{thm:curl-lower} and the geometric-epoch construction in
Remark~\ref{rem:stepsizes} gives the corresponding anytime statement.

Combining the exactness, cross-derivative, and circulation viewpoints gives the
following characterization for standard constrained mirror descent.

\begin{corollary}
\label{cor:md-full-mirror-characterization}
Let $U\colloneq\relint(K)$.  The set $U$ is open and convex in $\aff(K)$,
and hence simply connected there.  Suppose that $R\in C^2(U)$ and
$\nabla^2R(\xvec)$ is positive definite for every $\xvec\in U$, and the coefficient field $\xvec\longmapsto\nabla^2R(\xvec)\field(\xvec)$
is $C^1$.  All derivatives and coordinate components below are taken in
affine coordinates on $\aff(K)$.  Then the following conditions are
equivalent.
\begin{enumerate}[label=\textup{(\roman*)},leftmargin=*]
  \item The one-form
  $\field(\xvec)^{\transpose}\dd\nabla R(\xvec)$ is exact on $U$.
  \item For every coordinate pair $i,j$,
  \[
    \partial_j\!\bigl[\nabla^2R(\xvec)\field(\xvec)\bigr]_i
    =
    \partial_i\!\bigl[\nabla^2R(\xvec)\field(\xvec)\bigr]_j
    \qquad(\xvec\in U).
  \]
  \item The mirror circulation vanishes on every closed piecewise $C^1$
  loop in $U$.
\end{enumerate}
If these conditions fail, standard mirror descent can be forced to incur
linear regret along a sufficiently small loop contained in $U$.  Suppose
instead that they hold, that $\field$ extends to the displacement $\Dphi$ of a
feasible map on $K$, and that the resulting potential and displacement satisfy
the regularity hypotheses of \Cref{thm:md-exact-beyond-legendre}.  Then mirror
descent has sublinear regret.  Under the Lipschitz-gradient assumptions in
Corollary~\ref{cor:md-exact-beyond-legendre-rate}, the regret is $O(\sqrt T)$.
\end{corollary}

Exactness is therefore the sharp geometric boundary for standard
constrained mirror descent under the stated regularity.  Smoothness and
potential oscillation determine the quantitative rate, while nonzero
circulation rules out no regret.  Proximality plays no role in this
characterization.  It is addressed next as a separate representation
problem.

\subsection{Identity Interpolation of Bregman Proximal Deviations}
\label{sec:mirror-interpolation}

Theorem~\ref{thm:scaled-prox} has two consequences in Euclidean geometry.
First, interpolation with the identity scales the displacement exactly.
Second, every smooth exact displacement has a sufficiently small interpolation
that is proximal.  The first fact is purely algebraic and survives for every
mirror map.  The second depends on the curvature of the regularizer and can
fail even for a smooth Legendre geometry.  This subsection separates these
two statements.

The representation and interpolation arguments do not require compactness.
Let $K$ be a nonempty closed convex set, and let
$R:E\to(-\infty,+\infty]$ be proper, closed, and $m$-strongly convex in
the extended-valued sense.  Let $\mathcal X\subseteq K\cap\dom R$ be a
convex input domain on which $R$ is differentiable.  We use
$\mathcal X=K$ when $R$ is differentiable on a neighborhood of $K$, as in
standard constrained mirror descent, and $\mathcal X=\mathcal D$ in the
Legendre specialization below.

For a proper lower-semicontinuous function
$F\colloneq f+\indicator_K$, assume that $F+R$ is proper
(equivalently, $\dom F\cap\dom R\ne\varnothing$), and define
\[
  \prox_F^R(\xvec)
  \colloneq
  \argmin_{\yvec\in E}
  \bigl\{F(\yvec)+\DR(\yvec,\xvec)\bigr\},
  \qquad \xvec\in\mathcal X.
\]
The indicator in $F$ restricts the minimization to $K$.  If $F$ is
$\rho$-weakly convex for some $0\le\rho<m$, the objective is proper and lower semicontinuous as well as
$(m-\rho)$-strongly convex and the minimizer exists and is unique.

\begin{definition}
\label{def:bregman-radial-class}
For a specified input domain $\mathcal X\subseteq K\cap\dom R$, let
\[
\begin{aligned}
  \PhiMirrorProx{R}(\mathcal X)
  &\colloneq
  \Bigl\{\prox_F^R\big|_{\mathcal X}:\;
  \begin{array}{l}
  F=f+\indicator_K\text{ is proper, lower semicontinuous},\\
  \text{ and $\rho$-weakly convex for some $0\le\rho<m$}
  \end{array}
  \Bigr\},\\
  \PhiMirrorRad{R}(\mathcal X)
  &\colloneq
  \operatorname{Rad}_I\bigl(\PhiMirrorProx{R}(\mathcal X)\bigr).
\end{aligned}
\]
Thus $\phi\in\PhiMirrorRad{R}(\mathcal X)$ when
$\phi_\alpha\colloneq(1-\alpha)I+\alpha\phi$ belongs to
$\PhiMirrorProx{R}(\mathcal X)$ for at least one
$\alpha\in(0,1]$.
\end{definition}

For the constrained regularizer, we write $R_K\colloneq R+\indicator_K$. This notation is used only when the constraint is absorbed into the regularizer and the Bregman divergence itself continues to be generated by $R$.

The Bregman proximal-regret guarantee of \citet{cai2025proximal} applies to
standard constrained mirror descent and does not require Legendre duality.
We therefore take that guarantee as given and ask how the corresponding
deviation class relates to mirror exactness. In short, we show that radial proximal deviations are mirror-exact.

\begin{proposition}
\label{prop:bregman-radial-exact}
Assume $\mathcal X=K$.  For every admissible generator $F$ in
Definition~\ref{def:bregman-radial-class}, the function $
  \Psi_F(\thetavec)
  \colloneq
  R_K^*(\thetavec)-(F+R)^*(\thetavec)$
is differentiable and satisfies
\[
  \xvec-\prox_F^R(\xvec)
  =\gradtheta\Psi_F(\nabla R(\xvec))
  \qquad(\xvec\in K).
\]
Consequently, $
  \PhiMirrorProx{R}(K)
  \subseteq
  \PhiMirrorRad{R}(K)
  \subseteq
  \PhiMirrorExact{R}(K)$,
where $\PhiMirrorExact{R}(K)$ is the class from
Definition~\ref{def:mirror-exact-beyond-legendre}.  More precisely, if
$\phi_\alpha=\prox_F^R$, then $\phi$ is mirror-exact with potential
$\Psi_F/\alpha$.
\end{proposition}

The first inclusion follows by taking $\alpha=1$.  For the second, if
$\phi_\alpha=\prox_F^R$, then $
  \xvec-\phi_\alpha(\xvec)
  =\alpha\bigl(\xvec-\phi(\xvec)\bigr)$.
Thus exactness of $\phi_\alpha$ implies exactness of $\phi$.  The same identity
also transfers the Bregman proximal-regret guarantee of
\citet{cai2025proximal}.  For linearized regret, $\Reg_T(\phi)=\alpha^{-1}\Reg_T(\phi_\alpha)$,
and for convex losses, $
  \ell(\xvec)-\ell(\phi_\alpha(\xvec))
  \ge
  \alpha\bigl[\ell(\xvec)-\ell(\phi(\xvec))\bigr]$.
Hence, for each fixed deviation, the proximal-regret guarantee extends to
$\PhiMirrorRad{R}(K)$.

Thus standard mirror descent controls the identity-radial closure of the
Bregman proximal class.  The remaining question is whether this radial class
exhausts the mirror-exact class.  Theorem~\ref{thm:scaled-prox} gives an
affirmative answer for smooth Euclidean exact deviations.  We now show that
the corresponding statement can fail even for Legendre regularizers.  It is
therefore enough to return temporarily to the Legendre setting, where duality
provides a convenient test for when a mirror-exact deviation can admit a
Bregman proximal interpolation.

\paragraph{The Legendre test.}
We now return to the Legendre setting to test when a mirror-exact deviation
admits a Bregman proximal interpolation.  We use the assumptions of the first
three subsections, except that no $C^2$ Hessian assumption is needed.  Thus
$R$ is a proper closed Legendre function with $\overline{\dom R}=K$, and it is
$m$-strongly convex in the norm under consideration.  Write
$\mathcal D\colloneq\operatorname{int}_E(\dom R)$ and
$\ThetaSet\colloneq\operatorname{int}_E(\dom R^*)$.  Legendre duality identifies
these two sets through the inverse maps $\nabla R$ and $\nabla R^*$.

Let $\phi:\mathcal D\to K$ be $R$-mirror-exact with dual potential $\Psi$.
For $\alpha\in(0,1]$, define
\[
  u_\alpha(\thetavec)
  \colloneq
  R^*(\thetavec)-\alpha\Psi(\thetavec).
\]
If $\xvec=\nabla R^*(\thetavec)$, mirror exactness and the definition of
$\phi_\alpha$ give
\[
  \phi_\alpha(\nabla R^*(\thetavec))
  =\nabla u_\alpha(\thetavec).
\]
Thus, in dual coordinates, the interpolated map $\phi_\alpha$ is the gradient
map of $u_\alpha$.  If $\phi_\alpha$ is Bregman proximal, this gradient must
also arise from a convex conjugate.  This gives the first necessary condition
for proximal representability.

\begin{proposition}
\label{prop:legendre-prox-necessary}
Suppose $\phi_\alpha=\prox_{F_\alpha}^R$ on $\mathcal D$ for an admissible
Bregman proximal generator $F_\alpha$.  Then the dual endpoint potential
$u_\alpha$ agrees, up to an additive constant, with the convex conjugate of
$F_\alpha+R$.  More precisely, there is a constant $c_\alpha\in\R$ such that $u_\alpha =(F_\alpha+R)^*+c_\alpha$, on $\ThetaSet$.
Consequently, $u_\alpha$ is convex, and hence we have
\[
  D_\Psi(\thetavec',\thetavec)
  \le
  \frac1\alpha D_{R^*}(\thetavec',\thetavec)
  \qquad(\thetavec,\thetavec'\in\ThetaSet).
\]
\end{proposition}

Under the domain assumptions of
\Cref{prop:legendre-prox-reconstruction}, convexity of $u_\alpha$ yields a
variational Bregman representation.  It also reveals the connection with the relative
smoothness condition of \eqref{eq:relative-smoothness} used in the mirror-exact regret bound, since any proximal
interpolation necessarily satisfies such a condition with constant $1/\alpha$.
To obtain an admissible proximal map, however, the reconstructed generator $F_\alpha$ below must
additionally satisfy the required weak-convexity condition.

\begin{proposition}
\label{prop:legendre-prox-reconstruction}
Suppose $u_\alpha$ is convex on $\ThetaSet$, and let
\[
  \overline u_\alpha
  \colloneq
  \cl\bigl(u_\alpha+\indicator_{\ThetaSet}\bigr).
\]
Assume $\overline u_\alpha$ is proper and that
$\phi_\alpha(\mathcal D)\subseteq\dom R$.  Define the candidate proximal
generator
\[
  F_\alpha(\yvec)
  \colloneq
  \begin{cases}
    \overline u_\alpha^*(\yvec)-R(\yvec),
      & \yvec\in\dom R,\\
    +\infty, & \yvec\notin\dom R.
  \end{cases}
\]
Then, for every $\xvec\in\mathcal D$,
\[
  \phi_\alpha(\xvec)
  \in
  \argmin_{\yvec}
  \bigl\{F_\alpha(\yvec)+\DR(\yvec,\xvec)\bigr\}.
\]
Under these domain assumptions, convexity of $u_\alpha$ yields a
variational Bregman representation of $\phi_\alpha$.  If, in addition, $F_\alpha$ is proper, lower semicontinuous,
and $\rho_\alpha$-weakly convex for some $\rho_\alpha<m$, the minimizer is
unique and $\phi_\alpha=\prox_{F_\alpha}^R$.  In that case
$\phi\in\PhiMirrorRad{R}(\mathcal D)$.
\end{proposition}

\Cref{prop:legendre-prox-necessary,prop:legendre-prox-reconstruction} separate
proximal representability into two requirements.  First,
$R^*-\alpha\Psi$ must be convex.  Second, the reconstruction must
satisfy the stated domain conditions and produce a generator in the
admissible weakly convex proximal class.
The Euclidean construction in \Cref{thm:scaled-prox} satisfies both
requirements for a sufficiently small interpolation scale.  In a general
mirror geometry, either requirement may fail.

We begin with the first requirement.  Define
\[
  L_R(\Psi)
  \colloneq
  \inf\left\{
  L\ge0:
  D_\Psi(\thetavec',\thetavec)
  \le L D_{R^*}(\thetavec',\thetavec)
  \text{ for all }\thetavec,\thetavec'\in\ThetaSet
  \right\}.
\]
This quantity compares the curvature of $\Psi$ with that of $R^*$.

\begin{corollary}
\label{cor:mirror-interpolation-radius}
For $0<\alpha\le1$, the function $u_\alpha=R^*-\alpha\Psi$ is convex
exactly when $\alpha L_R(\Psi)\le1$.  Define the convexity threshold
\[
  \alpha_{\mathrm{cvx}}
  \colloneq
  \min\left\{1,\frac1{L_R(\Psi)}\right\},
\]
with the conventions $1/0\colloneq+\infty$ and
$1/(+\infty)\colloneq0$; a zero threshold means that no positive scale passes
the convexity test.  Every Bregman proximal interpolation must satisfy
$0<\alpha\le\alpha_{\mathrm{cvx}}$.  Conversely, every such $\alpha$ for which
$\overline u_\alpha$ is proper and
$\phi_\alpha(\mathcal D)\subseteq\dom R$ yields the variational representation
in \Cref{prop:legendre-prox-reconstruction}.  It is an admissible Bregman
proximal interpolation if the resulting $F_\alpha$ also satisfies the
weak-convexity and regularity conditions in that proposition.
\end{corollary}

The endpoint-domain condition is automatic for $0<\alpha<1$, since
$\mathcal D=\operatorname{int}_E(\dom R)$, $K=\overline{\dom R}$, and a
strict interpolation of an interior point with a point of $K$ lies in
$\mathcal D$.  At $\alpha=1$, one must additionally require
$\phi(\mathcal D)\subseteq\dom R$.

When $R^*$ and $\Psi$ are twice differentiable, the same condition becomes
\[
  \alpha\nabla^2\Psi(\thetavec)
  \preceq
  \nabla^2R^*(\thetavec)
  \qquad(\thetavec\in\ThetaSet).
\]
Thus the curvature contributed by $\alpha\Psi$ cannot exceed that of $R^*$
in any direction.  This is a relative curvature condition and can be much
stronger than an ordinary Euclidean Lipschitz bound on $\nabla\Psi$.

We next consider the second requirement.  Suppose $R$ is $M$-smooth on the
relevant primal domain and $\overline u_\alpha$ is a finite closed convex
function that is $L_\alpha$-smooth on its affine hull.  Then
$\overline u_\alpha^*$ is $1/L_\alpha$-strongly convex.  Comparing this with
the $M$-smoothness of $R$ shows that the reconstructed $F_\alpha$ is
\[
  \rho_\alpha
  \colloneq
  \left(M-\frac1{L_\alpha}\right)_+
\]
weakly convex.  Hence it is admissible when $\rho_\alpha<m$, provided
the reconstruction domain assumptions hold and $F_\alpha$ is proper and
lower semicontinuous.

In the Euclidean case $M=m=1$, a sufficiently small interpolation scale
provides both the convexity of $u_\alpha$ and the required weak-convexity
margin for $F_\alpha$, as shown in \Cref{thm:scaled-prox}.  For a nonquadratic
regularizer, these two curvature requirements need not follow from one
another.

The next example shows that failure can already occur at the first step.
Although the dual potential is smooth with globally Lipschitz gradient,
$R^*-\alpha\Psi$ is nonconvex for every positive interpolation scale.

\begin{proposition}
\label{prop:mirror-no-scaling}
Let $1<p<2$, let $q\colloneq p/(p-1)>2$, and take the full-domain Legendre
regularizer
\[
  R(\xvec)\colloneq\frac12\norm{\xvec}_p^2
  \qquad(\xvec\in\R^2).
\]
There is a feasible $R$-mirror-exact deviation whose dual potential is
$C^\infty$ with globally Lipschitz gradient, but which does not belong to
$\PhiMirrorRad{R}(\R^2)$.
\end{proposition}

To see why, note that
$R^*(\thetavec)=\frac12\norm{\thetavec}_q^2$.  Take
\[
  \Psi(\thetavec)\colloneq\frac12\thetavec[2]^2,
  \qquad
  \phi(\xvec)
  \colloneq
  \xvec-\gradtheta\Psi(\nabla R(\xvec)).
\]
The action space is all of $\R^2$, so feasibility is automatic, and
$\nabla\Psi$ is globally $1$-Lipschitz.  Along the line
$\thetavec=\evec_1+t\evec_2$,
\[
  u_\alpha(\evec_1+t\evec_2)
  =\frac12(1+\abs{t}^q)^{2/q}-\frac\alpha2t^2.
\]
As $t\to0$,
\[
  u_\alpha(\evec_1+t\evec_2)-u_\alpha(\evec_1)
  =\frac1q\abs{t}^q-\frac\alpha2t^2
  +O(\abs{t}^{2q}).
\]
Because $q>2$, the negative quadratic term dominates for every
$\alpha>0$ and every sufficiently small nonzero $t$.  Hence
$R^*-\alpha\Psi$ is nonconvex for every positive $\alpha$.

The obstruction comes from a mismatch in curvature.  At $\evec_1$, the
conjugate regularizer has no quadratic curvature in the transverse
$\evec_2$ direction, and its first increase is of order $\abs{t}^q$.  The
potential $\Psi$ has quadratic curvature in that same direction.  Scaling
$\Psi$ changes only the coefficient of this quadratic term, so no positive
scale can restore convexity.  In the Euclidean case,
$\frac12\norm{\thetavec}_2^2$ has quadratic curvature in every direction,
which is why the small-scale construction in \Cref{thm:scaled-prox}
succeeds.  Hence
\[
  \PhiMirrorProx{R}(\R^2)
  \subseteq
  \PhiMirrorRad{R}(\R^2)
  \subsetneq
  \PhiMirrorExact{R}(\R^2)
\]
can hold even for a Legendre regularizer with a $C^\infty$ dual potential
whose gradient is globally Lipschitz.  The regularizer is differentiable but
not $C^2$ on the coordinate axes, which is why the interpolation argument is
formulated through conjugacy rather than Hessians.

\section{Follow-the-Regularized-Leader Beyond Legendre Regularizers}
\label{sec:ftrl}

In the Legendre interior, online mirror descent and
follow-the-regularized-leader (FTRL) coincide for constant-step linearized
losses \citep{mcmahan2011follow,mcmahan2017survey}.  Indeed, mirror descent
satisfies
\[
  \nabla R(\xvec_{t+1})=\nabla R(\xvec_t)-\eta\gvec_t,
\]
which unrolls to the FTRL update
\[
  \xvec_t
  \colloneq
  \argmin_{\xvec\in E}
  \left\{
    \eta\sum_{s=1}^{t-1}\inner{\gvec_s}{\xvec}
    +R(\xvec)-\inner{\nabla R(\xvec_1)}{\xvec}
  \right\}.
\]
Beyond the Legendre interior, this equivalence need not persist.  Boundary
constraints introduce normal-cone terms into mirror descent, while FTRL
continues to use the cumulative linear losses in a single regularized
minimization; related distinctions between constrained mirror descent and
cumulative-dual methods are discussed by \citet{fang2022online}.  In \Cref{sec:md-beyond-legendre}, we characterized the full deviation class
controlled by standard constrained mirror descent beyond the Legendre setting.
For FTRL, the relevant geometry is different because the algorithm evolves
through its cumulative dual state rather than the one-step constrained mirror
update.  We therefore study FTRL separately and characterize the deviation
class for which it has sublinear regret, together with the corresponding
dual-circulation obstruction.

\subsection{Coincidence with Mirror Descent in the Legendre Interior}
\label{sec:md-ftrl-legendre}

The equivalence between mirror descent and follow-the-regularized-leader
(FTRL) for constant-step linearized losses in the Legendre interior is
standard; see, for example,
\citet{mcmahan2011follow,mcmahan2017survey}.  We record the precise form used
here for completeness.

\begin{proposition}
\label{prop:md-ftrl-equivalence}
Under the standing Legendre assumptions of \Cref{sec:mirror}, let
$\thetavec_1\colloneq\nabla R(\xvec_1)$ and use a constant step size $\eta$.  The
mirror-descent iterates satisfy
\[
  \xvec_t
  =\argmin_{\xvec\in E}
  \left\{
    \eta\sum_{s=1}^{t-1}\inner{\gvec_s}{\xvec}
    +R(\xvec)-\inner{\thetavec_1}{\xvec}
  \right\}
  =\argmin_{\xvec\in K}
  \left\{
    \eta\sum_{s=1}^{t-1}\inner{\gvec_s}{\xvec}
    +\DR(\xvec,\xvec_1)
  \right\}.
\]
Thus they are the FTRL iterates for the linearized losses and the tilted regularizer
$R-\inner{\thetavec_1}{\cdot}$.
\end{proposition}

Since the two algorithms generate the same iterates in this setting, the
Legendre-interior regret guarantees established above for mirror descent,
including the mirror-exact and Bregman proximal guarantees of \Cref{sec:omd,sec:bregman}, apply verbatim to
FTRL.  The situation changes beyond the Legendre interior.  Standard
constrained mirror descent can acquire a boundary normal-cone residual,
whereas FTRL continues to evolve through the cumulative linear losses.  We
analyzed the resulting deviation class for mirror descent in
\Cref{sec:md-beyond-legendre}.   We now carry out the corresponding analysis
for FTRL and show that, once boundary effects are present, the two algorithms
can control different deviation classes.

\subsection{Exactness and Circulation in the Cumulative Dual State}
\label{sec:ftrl-exactness}

Unlike constrained mirror descent, FTRL continues to evolve additively through
its cumulative dual state beyond the Legendre setting.  This suggests imposing
exactness directly in that variable.

Let $R:E\to(-\infty,+\infty]$ be proper, closed, and convex, set
$K\colloneq\overline{\dom R}$, and assume that $R^*$ is differentiable on the nonempty
open convex set $
  \ThetaSet\colloneq\operatorname{int}_E(\dom R^*).$
Given $\thetavec_1\in\ThetaSet$, define FTRL by
\[
\numberthis{eq:ftrl-update}
  \thetavec_t
  \colloneq\thetavec_1-\eta\sum_{s=1}^{t-1}\gvec_s,
  \qquad
  \xvec_t\colloneq\nabla R^*(\thetavec_t).
\]
Equivalently,
\[
  \xvec_t
  =\argmin_{\xvec\in E}
  \left\{
    \eta\sum_{s=1}^{t-1}\inner{\gvec_s}{\xvec}
    +R(\xvec)-\inner{\thetavec_1}{\xvec}
  \right\}.
\]
Because $R$ may be extended-valued, it can encode constraints and the resulting
iterates may lie on the boundary of $K$.

The relevant displacement is therefore the displacement viewed as a function
of the cumulative dual state.

\begin{definition}[FTRL-exact deviation]
\label{def:ftrl-exact}
A feasible map $\phi:K\to K$ is \emph{$R$-FTRL-exact} if there is a
differentiable potential $\Psi:\ThetaSet\to\R$ such that
\[
  \nabla R^*(\thetavec)
  -\phi(\nabla R^*(\thetavec))
  =\gradtheta\Psi(\thetavec)
  \qquad(\thetavec\in\ThetaSet).
\]
When $R$ is Legendre, this is the mirror-exact condition on the interior choice
range.
\end{definition}

We denote this class by $\PhiFTRL{R}(K)$.  Its dual displacement is
\[
  \Dphidual(\thetavec)
  \colloneq\nabla R^*(\thetavec)
  -\phi(\nabla R^*(\thetavec)).
\]
Thus FTRL-exactness is ordinary exactness of $\Dphidual$ in the cumulative
dual state.

To obtain a quantitative regret bound, we use uniform convexity of the
regularizer.  Let $\norm{\cdot}$ be a norm on $E$, with dual norm
$\dualnorm{\cdot}$.  We say that $R$ is $(\sigma,r)$-uniformly convex if, for
every $\xvec,\yvec\in\dom R$ and every
$\thetavec\in\partial R(\xvec)$,
\[
\numberthis{eq:uniform-convexity}
  R(\yvec)
  \ge R(\xvec)+\inner{\thetavec}{\yvec-\xvec}
  +\frac\sigma r\norm{\yvec-\xvec}^r.
\]
Here $\sigma>0$, $r>1$, and $q\colloneq r/(r-1)$.

Exactness makes the potential differences telescope, while relative smoothness
and uniform convexity control the Bregman error incurred by each dual update. We are now ready to state and prove our exact-form regret result for FTRL.

\begin{theorem}
\label{thm:ftrl-regret}
Assume that $R$ is proper, closed, $(\sigma,r)$-uniformly convex, and that
$R^*$ is finite on $E$.  Let the iterates follow \eqref{eq:ftrl-update}.  If
$\phi$ is $R$-FTRL-exact with potential $\Psi$ and
\[
  D_\Psi(\thetavec_{t+1},\thetavec_t)
  \le L D_{R^*}(\thetavec_{t+1},\thetavec_t)
  \qquad(t\in[T]),
\]
then
\[
  \sum_{t=1}^T
  \inner{\gvec_t}{\xvec_t-\phi(\xvec_t)}
  \le
  \frac{\Psi(\thetavec_1)-\Psi(\thetavec_{T+1})}{\eta}
  +\frac{L\eta^{q-1}}{q\sigma^{q-1}}
   \sum_{t=1}^T\dualnorm{\gvec_t}^{q}.
\]
\end{theorem}

As in the mirror-descent analysis and Condition~\eqref{eq:relative-smoothness} in \Cref{sec:omd}, the relative-smoothness condition is needed
only along the realized trajectory, namely on the consecutive pairs
$(\thetavec_{t+1},\thetavec_t)$.  Imposing it on all of
$\ThetaSet\times\ThetaSet$ gives a trajectory-independent sufficient
condition.

Under bounded gradients and bounded oscillation of the potential, optimizing
the preceding bound over the step size gives the following rate.

\begin{corollary}
\label{cor:ftrl-rate}
Suppose $\dualnorm{\gvec_t}\le G$ and the oscillation of $\Psi$ along the dual
trajectory is at most $B$.  When $B,L,G>0$, choosing
$\eta=\left(\frac{rB\sigma^{q-1}}{LG^qT}\right)^{1/q}$ gives
\[
  \sum_{t=1}^T
  \inner{\gvec_t}{\xvec_t-\phi(\xvec_t)}
  \le
  r^{1/r}\frac{G L^{1/q}B^{1/r}}{\sigma^{1/r}}T^{1/q}.
\]
For $r=2$, this is $G\sqrt{2LBT/\sigma}$.
\end{corollary}

As in the mirror-descent converse of \Cref{thm:md-mirror-circulation}, the
obstruction is nonzero circulation.  Here the relevant loop lives in the
cumulative dual state of FTRL.  If the dual displacement has nonzero
circulation, an adversary can force the iterates to repeatedly traverse such a
loop and incur linear regret.

\begin{theorem}
\label{thm:ftrl-circulation}
Suppose $\Dphidual$ is continuous and there is a closed piecewise
$C^1$ loop $\gamma\subset\ThetaSet$ with
\[
  \oint_\gamma
  \inner{\Dphidual(\thetavec)}{\dd\thetavec}\ne0.
\]
For every $G>0$ and every sufficiently small constant $\eta>0$, bounded
gradients $\dualnorm{\gvec_t}\le G$ can force FTRL
\eqref{eq:ftrl-update} to incur $cT-O(1/\eta)$ regret against $\phi$, for some
$c>0$.  The statement holds from every initial dual state in $\ThetaSet$ since
steering to the loop changes only the transient term.
\end{theorem}

Combining the upper and lower bounds gives the corresponding exactness-versus-
circulation characterization for FTRL.

\begin{corollary}
\label{cor:ftrl-curl}
Assume that $\Dphidual$ is $C^1$ on the open convex set $\ThetaSet$.
It is $R$-FTRL-exact if and only if its Jacobian is symmetric.  If a skew
derivative is nonzero, \Cref{thm:ftrl-circulation} gives linear regret along a
small dual loop.  Under the regularity and normalization assumptions of
\Cref{thm:ftrl-regret}, an exact field has sublinear regret.
\end{corollary}

\subsection{Mirror Descent and FTRL Control Different Deviation Classes}
\label{sec:md-ftrl-deviation-separation}

The equivalence in \Cref{prop:md-ftrl-equivalence} relies on the Legendre
interior.  Once boundary points are allowed, the two algorithms use different
dual information.  Standard constrained mirror descent is governed by the
primal mirror coordinate $\nabla R(\xvec)$ together with a normal-cone
residual, whereas FTRL with the extended regularizer
$R_K\colloneq R+\indicator_K$ evolves through a cumulative dual state that may
also contain normal components.  As a result, FTRL exactness imposes
conditions in dual directions that are not present in mirror exactness. In particular, every FTRL-exact deviation is mirror-exact.

\begin{proposition}
\label{prop:ftrl-exact-implies-md-exact}
Let $K$ be compact and convex, and let $R$ be differentiable and strongly
convex on a neighborhood of $K$.  If a feasible map $\phi:K\to K$ is
$R_K$-FTRL-exact, then it is $R$-mirror-exact in the sense of
\Cref{def:mirror-exact-beyond-legendre}.  Consequently,
\[
  \PhiFTRL{R_K}(K)
  \subseteq
  \PhiMirrorExact{R}(K).
\]
\end{proposition}

We next show that this inclusion can be strict, even for a smooth power-type
regularizer.  Fix $1<p<\infty$ and $0<a<b$, and set
\[
 K\colloneq[a,b]^2,
 \qquad
 R_p(\xvec)\colloneq\frac1p\norm{\xvec}_p^p
 =\frac1p\bigl(\xvec[1]^p+\xvec[2]^p\bigr),
 \qquad
 R_{p,K}\colloneq R_p+\indicator_K.
\]
Let $c\colloneq(a+b)/2$, choose $\beta\in(0,1)$, and write
\[
 S_\beta\colloneq
 \begin{pmatrix}1&\beta\\ \beta&1\end{pmatrix},
 \qquad
 m_p\colloneq(p-1)\min_{s\in[a,b]}s^{p-2}.
\]
Choose
$0<\eps\le m_p(1-\beta)/(1+\beta)$ and define
\[
\numberthis{eq:lp-md-ftrl-separation-map}
 \field_{p,\eps}(\xvec)
 \colloneq\eps
 \bigl[\nabla^2R_p(\xvec)\bigr]^{-1}
 S_\beta(\xvec-c\one),
 \qquad
 \phi_{p,\eps}(\xvec)
 \colloneq\xvec-\field_{p,\eps}(\xvec).
\]

The same deviation has opposite regret behavior under the two algorithms.

\begin{theorem}
\label{thm:md-ftrl-lp-separation}
The map $\phi_{p,\eps}$ is a smooth feasible deviation.  Standard
constrained mirror descent with distance generator $R_p$ has
$O(\sqrt T)$ regret against it for bounded gradients and the usual
$T^{-1/2}$ step size.  In contrast, FTRL with the extended regularizer
$R_{p,K}$ can be forced, for every sufficiently small constant step
size, to incur $cT-O(1/\eta)$ regret against the same deviation.  Hence the
inclusion in \Cref{prop:ftrl-exact-implies-md-exact} is strict and $ \PhiFTRL{R_{p,K}}(K) \subsetneq \PhiMirrorExact{R_p}(K)$.
\end{theorem}

For mirror descent, exactness is visible directly in primal coordinates.  The
one-form satisfies
\[
 \bigl[\nabla^2R_p(\xvec)\field_{p,\eps}(\xvec)\bigr]^{\transpose}
 \dd\xvec
 =\dd\left[
 \frac\eps2
 (\xvec-c\one)^{\transpose}S_\beta(\xvec-c\one)
 \right].
\]
All coefficients are smooth and bounded on the positive box, so the
relative-curvature and Lipschitz conditions in
\Cref{cor:md-exact-beyond-legendre-rate} hold with finite constants.  The
mirror-descent regret bound therefore applies.

For FTRL, the obstruction appears when the cumulative dual state reaches a
region where one coordinate is saturated at the boundary.  For the extended
regularizer, the choice map is coordinatewise
\[
 \nabla R_{p,K}^*(\thetavec)[i]
 =
 \begin{cases}
 a, & \thetavec[i]\le a^{p-1},\\
 \thetavec[i]^{1/(p-1)},
   & a^{p-1}<\thetavec[i]<b^{p-1},\\
 b, & \thetavec[i]\ge b^{p-1}.
 \end{cases}
\]
On the open strip, we have
\[
 \thetavec[1]>b^{p-1},
 \qquad
 a^{p-1}<\thetavec[2]<b^{p-1}.
\]
The dual displacement
$\widetilde{\field}_{p,\eps}$ satisfies
\[
 \frac{\partial\widetilde{\field}_{p,\eps}[1]}
 {\partial\thetavec[2]}
 =\frac{\eps\beta}{(p-1)^2}
 b^{2-p}\thetavec[2]^{(2-p)/(p-1)}>0,
 \qquad
 \frac{\partial\widetilde{\field}_{p,\eps}[2]}
 {\partial\thetavec[1]}=0.
\]
The unequal cross derivatives give nonzero dual circulation on a small
rectangle in this strip.  Hence \Cref{thm:ftrl-circulation} gives linear
regret for FTRL, even though the same deviation is mirror-exact and has
sublinear regret under standard constrained mirror descent.  At $p=2$, the
displacement in \eqref{eq:lp-md-ftrl-separation-map} is affine and the example
reduces to a symmetric cross-coordinate deviation on a box.

\subsection{Examples and Applications of the FTRL Bound}
\label{sec:ftrl-lp-rates}

We conclude the FTRL analysis with two standard power-type regularizers that
illustrate how \Cref{thm:ftrl-regret} specializes in concrete geometries.  In
both examples, $K\subseteq\R^d$ is nonempty, closed, and convex, the constraint
is encoded in $R$, $\phi$ is FTRL-exact, the relative-smoothness condition
holds with constant $L$, and the oscillation of $\Psi$ along the dual
trajectory is at most $B$.

For squared $\ell_p$ geometry with $1<p\le2$, take
\[
  R(\xvec)\colloneq\frac12\norm{\xvec}_p^2+\indicator_K(\xvec).
\]
Since this regularizer is $(p-1)$-strongly convex with respect to
$\norm{\cdot}_p$, \Cref{thm:ftrl-regret} gives the following specialization.

\begin{corollary}
\label{cor:ftrl-lp-squared}
Let $1<p\le2$ and let $q\colloneq p/(p-1)$.  Then
\[
  \sum_{t=1}^T
  \inner{\gvec_t}{\xvec_t-\phi(\xvec_t)}
  \le
  \frac{B}{\eta}
  +\frac{L\eta}{2(p-1)}
   \sum_{t=1}^T\norm{\gvec_t}_q^2.
\]
If $\norm{\gvec_t}_q\le G$, the optimized bound is
$G\sqrt{2LBT/(p-1)}$.
\end{corollary}

For $p\ge2$, the natural power-type regularizer is instead
\[
  R(\xvec)\colloneq\frac1p\norm{\xvec}_p^p+\indicator_K(\xvec).
\]
Its uniform convexity has exponent $p$, which leads to a different dependence
on the horizon.

\begin{corollary}
\label{cor:ftrl-lp-powered}
Let $2\le p<\infty$ and let $q\colloneq p/(p-1)$.  Then
\[
  \sum_{t=1}^T
  \inner{\gvec_t}{\xvec_t-\phi(\xvec_t)}
  \le
  \frac{B}{\eta}
  +\frac{L\eta^{q-1}}
  {q(2^{2-p})^{q-1}}
  \sum_{t=1}^T\norm{\gvec_t}_q^q.
\]
If $\norm{\gvec_t}_q\le G$, the optimized bound is $
  p^{1/p}2^{(p-2)/p}G L^{1/q}B^{1/p}T^{1/q}$,
which is $O(T^{1-1/p})$.
\end{corollary}

These examples show how the geometry of the regularizer determines the
quantitative FTRL rate through its uniform-convexity exponent.  The
one-homogeneous norm $\norm{\xvec}_p$ itself is not suitable for
\Cref{thm:ftrl-regret}, since it is affine along positive rays and its
conjugate is an indicator of the dual unit ball.  At $p=1$ or $p=\infty$, one
can instead use an entropic regularizer, a strongly convex perturbation, or a
nearby $\ell_p$ geometry with $1<p<\infty$.

\Cref{sec:md-beyond-legendre,sec:ftrl} give parallel characterizations for different algorithms.
For standard constrained mirror descent, exactness of
$\Dphi(\xvec)^{\transpose}\dd\nabla R(\xvec)$ is the sharp geometric
condition, and the boundary residual determines the additional analytic term in
the upper bound.  For FTRL, exactness and circulation live directly in the
cumulative dual state, which also permits nonsmooth and extended-valued
regularizers.  The theories coincide on Legendre interior trajectories.  Once
normal directions appear, FTRL exactness is a stricter extension requirement,
and \Cref{thm:md-ftrl-lp-separation} shows that the corresponding no-regret
deviation classes can differ even for the same powered $\ell_p$ geometry.

\section{Conservative Correlated Equilibria in Convex Games}
\label{sec:games}

The regret guarantees developed above translate directly into equilibrium
constraints in convex games.  It is useful to distinguish three questions.
The first concerns inclusion between deviation classes, the second concerns
quantitative approximate equilibria for normalized families, and the third
concerns equilibrium constraints when the approximation error is zero.
In particular, a strict inclusion between deviation classes, such as the one
in \Cref{prop:strict-separation}, does not by itself imply a strict inclusion
between the corresponding equilibrium sets.

Consider an $N$-player game.  Player $i$ has compact convex action set $K_i$
and a loss $\ell_i(\xvec_i,\xvec_{-i})$ convex and differentiable in its own
action.  Let $K\colloneq\prod_{i=1}^N K_i$.
Assume also that each loss $\ell_i:K\to\R$ is jointly continuous.
Throughout this section, every distribution is a Borel probability measure
on $K$, and every feasible deviation is Borel measurable.  Joint continuity
and compactness make each loss bounded, so all the expected gains below are
well defined and finite, including those of identity interpolations. For such a distribution $\mu$ and a feasible deviation $\phi_i$, define its
gain
\[
 \Gamma_i(\mu,\phi_i)
 \colloneq\E_{\xvec\sim\mu}\!
 \left[\ell_i(\xvec_i,\xvec_{-i})
 -\ell_i(\phi_i(\xvec_i),\xvec_{-i})\right].
\]

The interpolation results in \Cref{subsec:separation,sec:mirror-interpolation}
have the same consequence when the approximation error is zero.  Passing from
a deviation class to its identity-radial closure does not change the resulting
equilibrium constraints.

\begin{proposition}
\label{prop:radial-equilibrium}
Let $\mathcal A_i$ be any family of feasible Borel-measurable
deviations for player $i$.  A
distribution satisfies
$\Gamma_i(\mu,\phi_i)\le0$ for every $\phi_i\in\mathcal A_i$ and every player
if and only if it satisfies the same inequalities for every
$\phi_i\in\operatorname{Rad}_I(\mathcal A_i)$.
\end{proposition}

This equivalence concerns the equilibrium constraints themselves.
Quantitative bounds over a radial closure additionally require a common lower
bound on the admissible interpolation scales.

For each player $i$, let $\Phi_{i,\mathrm{prox}}$ be the unrestricted union of
all constrained proximal maps $\prox_F|_{K_i}$, where
$F\colloneq f+\indicator_{K_i}$ is proper, lower-semicontinuous, and
$\rho$-weakly convex for some map-dependent $0\le\rho<1$.  We use the
following equilibrium notions.
\begin{itemize}[leftmargin=*,itemsep=1mm]
  \item A distribution $\mu$ is a \emph{coarse correlated equilibrium} (CCE)
  if $\Gamma_i(\mu,\tau_{i,\uvec_i})\le0$ for every player $i$ and every
  constant deviation $\tau_{i,\uvec_i}(\xvec_i)\equiv\uvec_i\in K_i$.
  \item It is a \emph{proximal correlated equilibrium} (PCE) if
  $\Gamma_i(\mu,\tau_i)\le0$ for every player $i$ and every
  $\tau_i\in\Phi_{i,\mathrm{prox}}$.
  \item It is a \emph{correlated equilibrium} (CE) if
  $\Gamma_i(\mu,\tau_i)\le0$ for every player $i$ and every Borel-measurable
  deviation $\tau_i:K_i\to K_i$.
\end{itemize}
Thus, under the unrestricted proximal convention used here,
\[
  \CE\subseteq\PCE\subseteq\CCE.
\]
The CE notion originates with \citet{aumann1974subjectivity}; here we use its
standard measurable-deviation extension to compact convex action spaces.  The
original PCE work introduced PCE as a tractable refinement of CCE
\citep{cai2025proximal}.  We determine below how PCE compares with
measurable-deviation CE.  The two notions coincide on finite support, while
$\CE\subsetneq\PCE$ can occur on infinite support.

We next define the equilibrium notion associated with exact-form deviations.
The quantitative formulation and the case of zero approximation error require
slightly different conventions because uniform regret bounds require common
bounds on the potentials.

\begin{definition}[ConCE for a specified deviation family]
\label{def:conce}
Fix for every player a specified family $\Phi_i$ of feasible Euclidean
exact-form deviations.  A distribution $\mu$ is an $\eps$-conservative
correlated equilibrium relative to $(\Phi_i)_i$ if
\[
  \sup_{\phi_i\in\Phi_i}\Gamma_i(\mu,\phi_i)\le\eps
  \qquad\text{for every player }i.
\]
For approximate, quantitative conclusions, we use a family such as
$\Phi_i\colloneq\Phi_{i,\mathrm{exact}}(B_i,L_i)$, where
$\Phi_{i,\mathrm{exact}}(B_i,L_i)$ denotes
$\PhiExact(B_i,L_i)$ on the action set $K_i$.  In this family all potentials
satisfy the same oscillation bound $B_i$ and smoothness bound $L_i$. By contrast, when the approximation error is zero, ConCE tests every smooth
feasible exact-form deviation in the sense of \Cref{def:exact}
(equivalently, every feasible exact-form deviation admitting a
$C^{1,1}_{\mathrm{loc}}$ potential on a neighborhood of $K_i$).  Thus each
such deviation must have nonpositive gain.  Since this larger class has no
common bounds on the size or smoothness of its potentials, we do not claim a
single uniform quantitative estimate over the entire class.  Unless stated
otherwise, ConCE in this section refers to this Euclidean exact-form notion.
\end{definition}

The standard regret-to-equilibrium argument gives the following quantitative connection.

\begin{theorem}
\label{thm:equilibrium}
Let $\bar\mu_T\colloneq T^{-1}\sum_{t=1}^T\delta_{\xvec_t}$ and define
$\gvec_{i,t}\colloneq\nabla_{\xvec_i}\ell_i(\xvec_{i,t},\xvec_{-i,t})$.  If every player $i$
guarantees
\[
 \sup_{\phi_i\in\Phi_i}
 \sum_{t=1}^T\inner{\gvec_{i,t}}
 {\xvec_{i,t}-\phi_i(\xvec_{i,t})}\le R_{i,T},
\]
then $\bar\mu_T$ is an $\eps_T$-ConCE relative to $(\Phi_i)_i$, where
$\eps_T=\max_i R_{i,T}/T$.  In particular, OGD on
$\Phi_{i,\mathrm{exact}}(B_i,L_i)$ with $\norm{\gvec_{i,t}}\le G_i$ gives $
 \eps_T\le \max_i 2G_i\sqrt{L_iB_i/T}$.
\end{theorem}

The proof uses only convexity of the losses and the linearized regret bound.
The same argument therefore applies to the regularizer-dependent mirror and
FTRL deviation families.  We introduce their corresponding equilibrium notions
in \Cref{sec:mirror-equilibria}.  The distinction between them matters at the
boundary, as shown by \Cref{thm:md-ftrl-lp-separation}.

\subsection{Exact and Proximal Equilibrium Constraints}

The class $\Phi_{i,\mathrm{prox}}$ is the unrestricted proximal class used in
\citet{cai2025proximal}.  We impose no fixed $\rho$, normalization, locality,
or computational restriction.  Let $\Phi_{i,\mathrm{exact}}^{\mathrm{sm}}$
denote the smooth exact-form class tested by ConCE.  By
\Cref{cor:euclidean-radial-closure},
\[
  \operatorname{Rad}_I(\Phi_{i,\mathrm{prox}})
  =\Phi_{i,\mathrm{exact}}^{\mathrm{sm}}.
\]
Together with the radial invariance in
\Cref{prop:radial-equilibrium}, this immediately gives the equality of the
corresponding equilibrium constraints.

\begin{theorem}
\label{thm:conce-pce-equality}
Under the unrestricted conventions above and at zero approximation error $\eps = 0$, every convex game with compact
convex action sets and jointly continuous losses satisfies
\[
  \ConCE=\PCE.
\]
\end{theorem}

The proof is given in \Cref{app:proofs-games}.  This equality relies on the
unrestricted deviation classes and does not directly extend to classes with a
fixed weak-convexity parameter, normalization, locality restriction, or
computational restriction.  For normalized $\eps$-equilibria, the admissible
interpolation scale may depend on the deviation, so the strict difference
between the map classes can still affect quantitative guarantees.

\subsection{Mirror-exact, Mirror-radial, and FTRL Equilibrium Constraints}
\label{sec:mirror-equilibria}

We now record the equilibrium consequences of the regularizer-dependent
deviation classes studied in \Cref{sec:mirror,sec:ftrl}.  We begin with
standard constrained mirror descent.  Fix for every player $i$ a
differentiable $m_i$-strongly convex regularizer $R_i$ on $K_i$, as in
\Cref{sec:md-beyond-legendre}.  The full geometric class controlled by the
exactness characterization is $\PhiMirrorExact{R_i}(K_i)$.  For quantitative
equilibrium guarantees, we restrict to a subfamily with common regularity
parameters.  Let
\[
  \Phi_i^{\mathrm{MD}}
  \subseteq
  \PhiMirrorExact{R_i}(K_i)
\]
be such a family.  Assume that every $\phi_i\in\Phi_i^{\mathrm{MD}}$ has
a potential with oscillation at most $B_i^{\mathrm{MD}}$, satisfies
\eqref{eq:relative-smoothness-beyond-legendre} with the same constant
$L_i^{\mathrm{MD}}$, and has an $S_i$-Lipschitz displacement.  Suppose also
that $\nabla R_i$ is $M_i$-Lipschitz on $K_i$ and
$\dualnorm{\gvec_{i,t}}\le G_i$.  Set
\[
  C_i^{\mathrm{MD}}
  \colloneq
  \frac{L_i^{\mathrm{MD}}+S_i}{m_i}
  +\frac{S_iM_i}{m_i^2}.
\]
If player $i$ uses standard mirror descent with
\[
  \eta_i^{\mathrm{MD}}
  \colloneq
  \sqrt{\frac{B_i^{\mathrm{MD}}}
  {C_i^{\mathrm{MD}}G_i^2T}},
\]
whenever $B_i^{\mathrm{MD}}C_i^{\mathrm{MD}}>0$, then
\Cref{cor:md-exact-beyond-legendre-rate,thm:equilibrium} give
\[
  \sup_{\phi_i\in\Phi_i^{\mathrm{MD}}}
  \Gamma_i(\bar\mu_T,\phi_i)
  \le
  2G_i\sqrt{\frac{B_i^{\mathrm{MD}}C_i^{\mathrm{MD}}}{T}}.
\]
The degenerate cases follow from the corresponding unoptimized bound.  If
$\nabla R_i$ is only uniformly continuous, the modulus estimate in
\Cref{cor:md-exact-beyond-legendre-rate} still gives an $o(1)$ equilibrium
error for every family with common values of
$B_i^{\mathrm{MD}},L_i^{\mathrm{MD}}$, and $S_i$.

The identity-interpolation results of \Cref{sec:mirror-interpolation} give a
second equilibrium notion associated with the same regularizer.  Let
$\PhiMirrorProx{R_i}(K_i)$ and $\PhiMirrorRad{R_i}(K_i)$ be the Bregman
proximal and radial classes from \Cref{def:bregman-radial-class}.  A
distribution is a \emph{Bregman $R$-PCE} when every deviation in
$\PhiMirrorProx{R_i}(K_i)$ has nonpositive gain for every player.  It is a
\emph{radial $R$-equilibrium} when the same condition holds for every
deviation in $\PhiMirrorRad{R_i}(K_i)$.  Denote the two sets by
$\PCE^{\mathrm{Breg}}_{(R_i)_i}$ and
$\mathrm{Eq}^{\mathrm{rad}}_{(R_i)_i}$, respectively.

By radial invariance, these two equilibrium notions coincide. We record this result in the following corollary.

\begin{corollary}
\label{cor:mirror-equilibrium-equality}
For any set of regularizers $R_i$ for $i \in [N]$, we have
\[
  \PCE^{\mathrm{Breg}}_{(R_i)_i}
  =
  \mathrm{Eq}^{\mathrm{rad}}_{(R_i)_i}.
\]
\end{corollary}

This equality concerns the proximal class and its identity-radial closure; it
does not identify the radial class with the full mirror-exact class.  To make
this distinction explicit, define
\[
  \mathrm{Eq}^{\mathrm{exact}}_{(R_i)_i}
  \colloneq
  \left\{
  \mu:
  \Gamma_i(\mu,\phi_i)\le0
  \text{ for every player $i$ and every }
  \phi_i\in\PhiMirrorExact{R_i}(K_i)
  \right\}.
\]
\Cref{prop:bregman-radial-exact} gives
\[
  \mathrm{Eq}^{\mathrm{exact}}_{(R_i)_i}
  \subseteq
  \mathrm{Eq}^{\mathrm{rad}}_{(R_i)_i}
  =
  \PCE^{\mathrm{Breg}}_{(R_i)_i}.
\]
The first inclusion can be strict at the level of deviation maps.  In the
Legendre regime, \Cref{prop:legendre-prox-necessary} shows that a mirror-exact
deviation can belong to the radial class only if it passes the relative
curvature test in \Cref{cor:mirror-interpolation-radius}.
\Cref{prop:legendre-prox-reconstruction} gives the converse when the
reconstructed generator is an admissible weakly convex Bregman proximal
function, while \Cref{prop:mirror-no-scaling} shows that no positive
interpolation may exist even for a smooth dual potential with globally
Lipschitz gradient.  As in the Euclidean setting, strict inclusion of
deviation classes does not by itself imply strict inclusion of the
corresponding equilibrium sets.

FTRL produces a third, distinct family of equilibrium constraints.  For each
player, let
$\mathcal R_i:\R^{d_i}\to(-\infty,+\infty]$ be a proper closed
$(\sigma_i,r_i)$-uniformly convex FTRL regularizer with
$K_i=\overline{\dom\mathcal R_i}$, and let
\[
  \Phi_i^{\mathrm{FTRL}}
  \subseteq
  \PhiFTRL{\mathcal R_i}(K_i)
\]
be a specified family.  Player $i$ uses \eqref{eq:ftrl-update} with $R$
replaced by $\mathcal R_i$.  When FTRL and mirror descent use the same base
geometry and constraint, one takes
$\mathcal R_i\colloneq R_i+\indicator_{K_i}$.  Let
$q_i\colloneq r_i/(r_i-1)$.  Assume that every potential in
$\Phi_i^{\mathrm{FTRL}}$ has trajectory oscillation at most
$B_i^{\mathrm{FTRL}}$, satisfies the consecutive relative bound with the
same constant $L_i^{\mathrm{FTRL}}$, and that
$\dualnorm{\gvec_{i,t}}\le G_i$.  With the learning rate from
\Cref{cor:ftrl-rate}, the empirical distribution is an
$\eps_T^{\mathrm{FTRL}}$-equilibrium relative to
$(\Phi_i^{\mathrm{FTRL}})_i$, where
\[
  \eps_T^{\mathrm{FTRL}}
  \le
  \max_i\left\{
  r_i^{1/r_i}
  \frac{G_i
  (L_i^{\mathrm{FTRL}})^{1/q_i}
  (B_i^{\mathrm{FTRL}})^{1/r_i}}
  {\sigma_i^{1/r_i}}
  T^{-1/r_i}
  \right\}.
\]
For $r_i=2$, the rate is $O(T^{-1/2})$.  This guarantee applies to the
FTRL-exact deviation family rather than the mirror-exact family.  As shown in
\Cref{thm:md-ftrl-lp-separation}, the two classes can differ once boundary
effects are present, even when the algorithms use the same underlying
geometry and action set.

The remaining subsections return to the Euclidean ConCE and PCE notions.
Their comparison with unrestricted measurable CE does not depend on a
selected mirror regularizer.

\subsection{Finite Support and Boundary-compatible Interpolation}

On a finite support, only finitely many recommended actions and deviation
values matter.  The following result shows that any such finite collection can
be realized, after a sufficiently small common interpolation with the
identity, by a single proximal map.  Its proof appears in
\Cref{app:finite-interpolation}.

\begin{theorem}
\label{thm:finite-prox}
Let $K\subseteq\R^d$ be nonempty, closed, and convex.  Let
$\xvec^1,\ldots,\xvec^m\in K$ be distinct and choose arbitrary targets
$\yvec^1,\ldots,\yvec^m\in K$.  Fix $\rho\in(0,1)$.  There is $\alpha_0>0$ such
that for every $\alpha\in(0,\alpha_0]$ one can find a proper
lower-semicontinuous $F_\alpha=f_\alpha+\indicator_K$ that is
$\rho$-weakly convex and satisfies
\[
 \prox_{F_\alpha}(\xvec^j)
 =(1-\alpha)\xvec^j+\alpha\yvec^j,
 \qquad j=1,\ldots,m.
\]
The proximal minimizer is unique.
\end{theorem}

It follows that the three equilibrium notions cannot be separated on a finite
support.

\begin{corollary}\label{cor:finite-support}
For every finitely supported distribution in a convex game,
\[
  \mu\in\CE\quad\Longleftrightarrow\quad
  \mu\in\ConCE\quad\Longleftrightarrow\quad
  \mu\in\PCE.
\]
In particular the three notions coincide in every finite game.
\end{corollary}

For the finite-game statement, we use the standard mixed extension, with each
pure-action set embedded as the vertices of its mixed-strategy simplex and the
distribution supported on pure-action profiles.  Since the extended losses are
affine in a player's own mixed action, any profitable deviation to a mixed
action has a profitable pure action in its support.  The proof of
\Cref{cor:finite-support} is given in \Cref{app:proofs-games}.

Indeed, \citet[Proposition~6.1]{monnot2017limits} show that
$\CE=\CCE$ when every player has two actions: in the CCE constraint for the
constant deviation to one action, the terms corresponding to recommendations
of that same action cancel, leaving exactly the CE constraint for switching
from the other action.  The remaining equalities follow from
\Cref{cor:finite-support}.  This discrete binary-action statement is distinct
from \Cref{thm:one-dimensional-ce} below: with a continuum of actions on an interval,
the same cancellation does not identify CCE with CE.

\subsection{Full Correlated Equilibrium on Intervals}

On an interval, every smooth scalar displacement has an antiderivative.  Under
the joint continuity assumption on the losses, smooth feasible deviations are
also sufficiently rich to approximate all measurable CE deviations.

\begin{theorem}[Full CE on one-dimensional action spaces]
\label{thm:one-dimensional-ce}
Suppose that every player has a compact interval action set
$K_i\colloneq[a_i,b_i]\subseteq\R$ and, in addition to the standing assumptions, every
loss $\ell_i:K\to\R$ is continuous on
$K\colloneq\prod_{j=1}^N K_j$.  Then
\[
  \CE=\ConCE=\PCE
\]
under the unrestricted definitions used in this section when the approximation
error is zero.
\end{theorem}

The proof of \Cref{thm:one-dimensional-ce} is given in
\Cref{app:proofs-games}.  In particular, infinite support alone does not
separate these equilibrium notions.  Under continuous losses, such a
separation requires at least one player's individual action set to have
affine dimension at least two.

Two related results help place the statement above.  In finite two-action
games, \citet[Proposition~6.1]{monnot2017limits} show that $\CE=\CCE$ through
a cancellation specific to the two available actions.  Although mixed
strategies in that setting admit a one-dimensional parameterization as
$(q,1-q)$, the argument relies on the underlying binary action structure and
does not extend to interval action spaces, where each player has a continuum
of possible deviations.  In particular, we do not obtain an analogous
equality with CCE. Separately, \citet[Theorem~5.2]{piliouras2022evolutionary} study
continuous-time replicator dynamics with two actions and obtain a localized
regret guarantee under a condition on how the resulting one-dimensional
trajectory crosses an interval.  Their main game-theoretic results likewise
focus on two-player \(2\times2\) games.  The result below differs in both scope
and method.  It applies to any number of players with continuous interval
action spaces and does not depend on continuous-time dynamics or on any
particular learning trajectory.  Instead, it is an equilibrium-set statement
for arbitrary Borel distributions, proved by approximating measurable
deviations with smooth feasible exact-form maps.

\subsection{Strict Separation from CE on Continuous Support}

Finite interpolation shows that any separation from CE must involve infinite
support.  The following result establishes such a separation and gives it a
simple variational interpretation.

The construction uses the diagonal law $(X,X)$ with $X$ uniform on a disk.
Thus the marginals are absolutely continuous with respect to two-dimensional
Lebesgue measure and the joint support is uncountable, while the joint law
itself is singular in the product action space.  Accordingly, the result does
not establish a separation for a joint distribution with a full-dimensional
density.

\begin{theorem}\label{thm:ce-separation}
There is a two-player convex game and a distribution $\mu$ with uncountable
support---indeed, with marginals absolutely continuous with respect to
two-dimensional Lebesgue measure---for which
\[
   \mu\in\ConCE=\PCE
   \qquad\text{but}\qquad
   \mu\notin\CE.
\]
Consequently $\CE\subsetneq\ConCE=\PCE$ under the definitions above.
\end{theorem}

The proof of \Cref{thm:ce-separation} is given in
\Cref{app:proofs-games}.  It constructs a conditional loss-gradient field that
is invisible to exact-form deviations but can be exploited by unrestricted
deviations.  The construction also illustrates why a simple localization
argument cannot establish CE equality on general continuous supports.  We
defer the question of whether a similar strict separation can be obtained with
an absolutely continuous, full-dimensional joint law to future work.

\section{Conclusion} \label{sec:conclusion}

This paper develops a geometric theory of the action-dependent comparators
controlled by projected first-order methods.  Its central principle is that
exactness determines the interior behavior of a deviation, while feasibility
controls its interaction with the boundary.  In Euclidean geometry, exact
feasible displacement fields yield explicit regret bounds for projected
gradient descent, whereas nonzero circulation can be exploited to force linear
regret.  The same exactness-versus-circulation principle extends to mirror
geometry, where the regularizer determines the relevant one-form and the
additional regularity needed to control boundary residuals.

This viewpoint also clarifies several distinctions among deviation classes,
algorithms, and equilibrium notions.  Proximal deviations and their identity
interpolations form a subclass of exact-form deviations, but the Euclidean
radial-closure representation does not extend to general Bregman geometry:
convexity of the reconstructed dual potential provides an additional
obstruction.  Mirror descent and FTRL likewise agree in the Legendre interior
but can control different deviation classes once constraints create
nontrivial boundary effects.  At the equilibrium level, strict inclusion of
deviation maps need not produce strict inclusion of zero-tolerance equilibrium
sets. In particular, we showed that unrestricted ConCE and PCE coincide, and they coincide with CE on finite
supports and on one-dimensional action spaces with continuous losses, while
a two-dimensional continuous-support example yields
\[
  \CE\subsetneq\ConCE=\PCE.
\]

Two questions appear particularly natural.  First, can the relative-curvature,
Lipschitz, and ambient regularity assumptions in the positive mirror-descent
result be weakened while retaining control of the boundary residual and hence
sublinear regret?  A sharp answer would bring the sufficient conditions closer
to the circulation-based obstruction, which requires only continuity of the
field.  Second, can the strict continuous-support equilibrium separation be
realized by a full-dimensional absolutely continuous joint distribution?  The
present construction has absolutely continuous marginals but is supported on
a lower-dimensional diagonal, so resolving this question would determine
whether the separation persists for genuinely diffuse outcome
distributions.

More broadly, the exact-form viewpoint suggests connections with Blackwell
approachability~\citep{blackwell1956analog,abernethy2011blackwell} and gradient equilibrium~\citep{angelopoulos2025gradient,lee2026blackwell}.  All three frameworks exploit
geometric structure to identify quantities that first-order dynamics can drive
toward a target set or make telescope over time.  Understanding these
connections may lead to alternative characterizations of exact-form regret and
to broader classes of deviations controlled by simple first-order methods. A related direction is to combine exactness with lifting~\citep{farina2022near,soleymani2025cautious}.  Deviations that are
not exact in the original action space may admit an exact representation in a
higher-dimensional feature space, potentially allowing gradient-based
algorithms to control richer families such as general linear endomorphisms.
Developing such exact lifts, and relating them to approachability and gradient
equilibrium, could provide a bridge between the geometric simplicity of the
present framework and more general $\Phi$-regret guarantees.

\section*{Acknowledgments}

The authors thank Mete \c{S}eref Ahunbay for helpful discussions and suggestions.

This research was partially supported by the Office of Naval Research (ONR) grants N00014-24-1-2470 and N00014-25-1-2296, National Science Foundation awards CCF-2443068 and IIS-2552046, and a Schmidt Sciences AI2050 Early Career Fellowship. Part of this research was performed while A.S. was visiting the Institute for Mathematical and Statistical Innovation (IMSI), which is supported by the National Science Foundation (Grant No. DMS-2425650).

\bibliographystyle{unsrtnat}
\bibliography{refs} 

\newpage
\appendix

\section{Preliminaries}\label{app:preliminaries}

\begin{lemma}\label{lem:projection-residual}
Let \(\uvec\in\R^d\) and \(\xvec^+\colloneq\Proj_K(\uvec)\).  Then
\[
  \uvec-\xvec^+\in \normal_K(\xvec^+),
\]
\end{lemma}
\begin{proof}
The point \(\xvec^+\) minimizes \(\yvec\mapsto \frac12\norm{\yvec-\uvec}^2\) over \(K\).  The first-order optimality condition is
\[
  \inner{\xvec^+-\uvec}{\yvec-\xvec^+}\ge 0\qquad\forall \yvec\in K.
\]
Equivalently, \(\inner{\uvec-\xvec^+}{\yvec-\xvec^+}\le 0\) for all \(\yvec\in K\), which is the normal-cone condition.
\end{proof}

\begin{definition}
For $\rho\ge0$, a proper lower-semicontinuous function
\(F:\R^d\to(-\infty,+\infty]\) is \(\rho\)-weakly convex if
\[
  \xvec\mapsto F(\xvec)+\frac{\rho}{2}\norm{\xvec}^2
\]
 is convex.  We use the weakly convex subdifferential convention
\[
 \partial F(\xvec)
 \colloneq\partial_{\rm cvx}\!\left(F+\frac\rho2\norm{\cdot}^2\right)(\xvec)
   -\rho\xvec,
\]
where $\partial_{\rm cvx}$ is the convex subdifferential.  With this
convention, $\partial F$ is \(\rho\)-hypomonotone: for all
\(\uvec\in\partial F(\pvec)\) and \(\vvec\in\partial F(\qvec)\),
\[
\numberthis{eq:hypomonotone}
  \inner{\uvec-\vvec}{\pvec-\qvec}\ge -\rho\norm{\pvec-\qvec}^2.
\]
\end{definition}

\begin{lemma}\label{lem:weak-moreau}
Assume \(F\colloneq f+\indicator_K\) is proper, lower semicontinuous, and \(\rho\)-weakly convex with \(0\le\rho<1\).  Then, for every $\xvec\in\R^d$, the proximal objective attains its minimum at a unique point. Moreover, \(\MF\) is differentiable, and
\[
\numberthis{eq:weak-moreau-gradient}
  \nabla\MF(\xvec)=\xvec-\prox_F(\xvec).
\]
Moreover,
\[
\numberthis{eq:prox-lipschitz}
  \norm{\prox_F(\xvec)-\prox_F(\yvec)}\le \frac{1}{1-\rho}\norm{\xvec-\yvec},
\]
so \(\nabla\MF\) is Lipschitz with the crude bound
\[
\numberthis{eq:moreau-smooth-crude}
  \norm{\nabla\MF(\xvec)-\nabla\MF(\yvec)}
  \le \frac{2-\rho}{1-\rho}\norm{\xvec-\yvec}.
\]
\end{lemma}

\begin{proof}
For fixed \(\xvec\), the function \(\yvec\mapsto F(\yvec)+\frac12\norm{\yvec-\xvec}^2\) is proper, lower semicontinuous, coercive, and \((1-\rho)\)-strongly convex.  It therefore attains its minimum at a unique point.  Let
\[
  \pvec=\prox_F(\xvec),\qquad \qvec=\prox_F(\yvec).
\]
The first-order optimality conditions are
\[
  \xvec-\pvec\in\partial F(\pvec),
  \qquad
  \yvec-\qvec\in\partial F(\qvec).
\]
Applying hypomonotonicity \eqref{eq:hypomonotone} gives
\[
  \inner{(\xvec-\pvec)-(\yvec-\qvec)}{\pvec-\qvec}
  &\ge -\rho\norm{\pvec-\qvec}^2,\\
  \inner{\xvec-\yvec}{\pvec-\qvec}-\norm{\pvec-\qvec}^2
  &\ge -\rho\norm{\pvec-\qvec}^2.
\]
Thus
\[
  (1-\rho)\norm{\pvec-\qvec}^2
  \le \inner{\xvec-\yvec}{\pvec-\qvec}
  \le \norm{\xvec-\yvec}\,\norm{\pvec-\qvec},
\]
which proves \eqref{eq:prox-lipschitz}.

For fixed \(\yvec\), the derivative with respect to \(\xvec\) of \(F(\yvec)+\frac12\norm{\yvec-\xvec}^2\) is \(\xvec-\yvec\).  Since the minimizer is unique, Danskin's theorem gives \(\nabla\MF(\xvec)=\xvec-\prox_F(\xvec)\).  Finally,
\[
  \norm{\nabla\MF(\xvec)-\nabla\MF(\yvec)}
  &=\norm{(\xvec-\yvec)-\big(\prox_F(\xvec)-\prox_F(\yvec)\big)}\\
  &\le \norm{\xvec-\yvec}+\norm{\prox_F(\xvec)-\prox_F(\yvec)}\\
  &\le \left(1+\frac{1}{1-\rho}\right)\norm{\xvec-\yvec}.
\]
This is \eqref{eq:moreau-smooth-crude}.
\end{proof}

In turn, we recall that weakly-convex proximal maps are monotone.

\begin{lemma} \label{lem:weakly-convex-prox-monotone}
Let $F:\R^d\to(-\infty,+\infty]$ be a proper lower-semicontinuous $\rho$-weakly convex function with $0\le\rho<1$. Then $\prox_F$ is monotone in the sense that, for all $\xvec,\yvec\in\R^d$,
\[
    \left\langle \xvec-\yvec, \prox_F(\xvec)-\prox_F(\yvec)\right\rangle \ge 0.
\]
\end{lemma}

\begin{proof}
Fix $\xvec,\yvec\in\R^d$. The proximal optimality conditions give
\[
  \xvec-\prox_F(\xvec)\in\partial F(\prox_F(\xvec)),
  \qquad
  \yvec-\prox_F(\yvec)\in\partial F(\prox_F(\yvec)).
\]
Since $F$ is $\rho$-weakly convex, its subdifferential is $\rho$-hypomonotone. Hence
\[
  \left\langle
    \big(\xvec-\prox_F(\xvec)\big)-\big(\yvec-\prox_F(\yvec)\big),
    \prox_F(\xvec)-\prox_F(\yvec)
  \right\rangle
  \ge
  -\rho\|\prox_F(\xvec)-\prox_F(\yvec)\|^2.
\]
Rearranging gives
\[
  \left\langle \xvec-\yvec, \prox_F(\xvec)-\prox_F(\yvec)\right\rangle
  \ge
  (1-\rho)\|\prox_F(\xvec)-\prox_F(\yvec)\|^2
  \ge 0.
\]
Thus $\prox_F$ is monotone.
\end{proof}

\section{Scaled Exact Deviations and Proximal Representation}
\label{app:scaled-prox-proof}

\begin{lemma}\label{lem:smooth-extension}
Let $K$ be compact and let $\Psi$ have locally Lipschitz gradient on an open
neighborhood of $K$.  Then $\Psi$ agrees on a neighborhood of $K$ with a
globally defined differentiable function $\widetilde\Psi$ whose gradient is
globally Lipschitz.
\end{lemma}

\begin{proof}
Choose open sets $K\subset V\Subset W$ with $\overline W$ inside the domain of
$\Psi$, and a smooth cutoff $\chi$ equal to one on $V$ and supported in $W$.
The function $\widetilde\Psi\colloneq\chi\Psi$ on the original domain, extended by zero
outside it, has the required properties.  Compactness of $\overline W$ gives a
finite Lipschitz constant for its gradient.
\end{proof}

\begin{repeatthm}{thm:scaled-prox}
Let $K\subseteq\R^d$ be compact and convex, and let
$\phi(\xvec)\colloneq\xvec-\nabla\Psi(\xvec)$ be a feasible exact-form deviation with
$\Psi\in C^{1,1}_{\mathrm{loc}}$ on a neighborhood of $K$.  Then there exists
$\alpha_0>0$ such that, for every $\alpha\in(0,\alpha_0]$, the interpolated map
\[
 \phi_\alpha(\xvec)\colloneq(1-\alpha)\xvec+\alpha\phi(\xvec)
 =\xvec-\alpha\nabla\Psi(\xvec)
\]
coincides on $K$ with $\prox_{F_\alpha}(\xvec)$ for some proper lower-semicontinuous
$F_\alpha=f_\alpha+\indicator_K$ that is $\rho_\alpha$-weakly convex for some
$\rho_\alpha<1$.
\end{repeatthm}

\begin{proof}
By \Cref{lem:smooth-extension}, replace $\Psi$ outside a neighborhood of $K$ so
that $\nabla\Psi$ is globally $L$-Lipschitz.  If $L=0$, set
$\alpha_0\colloneq 1/2$; otherwise set
$\alpha_0\colloneq\frac12\min\{1,1/L\}$.  Fix any
$\alpha\in(0,\alpha_0]$ and set
\[
 h_\alpha(\xvec)\colloneq\frac12\norm{\xvec}^2-\alpha\Psi(\xvec),
 \qquad
 f_\alpha(\yvec)\colloneq h_\alpha^*(\yvec)-\frac12\norm{\yvec}^2.
\]
The function $h_\alpha$ is $(1-\alpha L)$-strongly convex and
$(1+\alpha L)$-smooth.  Hence $h_\alpha^*$ is
$1/(1+\alpha L)$-strongly convex and $f_\alpha$ is
$\rho_\alpha$-weakly convex for
$\rho_\alpha\colloneq\alpha L/(1+\alpha L)<1$.  Put
$F_\alpha\colloneq f_\alpha+\indicator_K$.

For fixed $\xvec$, the unconstrained proximal objective equals, up to the
constant $\norm{\xvec}^2/2$,
\[
 h_\alpha^*(\yvec)-\inner{\xvec}{\yvec}.
\]
Its unique minimizer is
$\yvec=\nabla h_\alpha(\xvec)=\xvec-\alpha\nabla\Psi(\xvec)
=\phi_\alpha(\xvec)$.  If $\xvec\in K$, convexity of $K$ and feasibility of
$\phi$ give $\phi_\alpha(\xvec)\colloneq(1-\alpha)\xvec+\alpha\phi(\xvec)\in K$.
The constraint therefore does not change the minimizer, proving
$\prox_{F_\alpha}=\phi_\alpha$ on $K$.
\end{proof}

\begin{repeatcorollary}{cor:euclidean-radial-closure}
For the radial closure of Euclidean proximal deviations, we have $
  \operatorname{Rad}_I(\PhiProx)
  =\PhiExactSmooth$.
\end{repeatcorollary}

\begin{proof}
Let $\phi\in\PhiExactSmooth$.  By
\Cref{thm:scaled-prox}, there is $\alpha>0$ for which
$\phi_\alpha=(1-\alpha)I+\alpha\phi$ is proximal.  Hence
$\phi\in\operatorname{Rad}_I(\PhiProx)$.

Conversely, let
$\phi\in\operatorname{Rad}_I(\PhiProx)$.  Choose
$\alpha\in(0,1]$ and a proper lower-semicontinuous weakly convex function $F$
with $\phi_\alpha=\prox_F$ on $K$.  By
\Cref{thm:prox-contained},
\[
  \xvec-\phi_\alpha(\xvec)=\nabla M_F(\xvec),
\]
and $\nabla M_F$ is Lipschitz.  Since
$\xvec-\phi_\alpha(\xvec)=\alpha(\xvec-\phi(\xvec))$, we have
\[
  \xvec-\phi(\xvec)
  =\nabla\left(\frac1\alpha M_F\right)(\xvec).
\]
Thus $\phi$ is smooth exact-form, proving the reverse inclusion.
\end{proof}

\section{Boundary-Compatible Finite Interpolation}
\label{app:finite-interpolation}

\begin{lemma}
\label{lem:finite-convex-interpolation}
For finite data $(\zvec^j,\svec^j)_{j=1}^m$, there is a proper closed convex
function $H$ with $\svec^j\in\partial H(\zvec^j)$ for every $j$ if and only if
the data are cyclically monotone, namely
\[
 \sum_{r=1}^q
 \inner{\svec^{j_r}}{\zvec^{j_{r+1}}-\zvec^{j_r}}\le0
 \quad\text{for every cycle }j_1,\ldots,j_q,j_{q+1}=j_1.
\]
It is enough to check nontrivial simple cycles.
\end{lemma}

\begin{proof}
Necessity follows by summing the subgradient inequalities around a cycle.
For sufficiency, form the complete directed graph on $[m]$ with edge weight
$w(j,k)=\inner{\svec^j}{\zvec^k-\zvec^j}$.  Add a source vertex $0$ with a
zero-weight edge to every $j$, and let
\[
 a_k=\max\left\{\sum_{r=0}^{q-1}w(j_r,j_{r+1}):
 j_q=k,\ j_0\in[m],\text{ and the path is simple}\right\}.
\]
The maximum is finite because there are finitely many simple paths.  Any walk
can be reduced to a simple path by deleting closed subwalks; cyclic
monotonicity says each deleted closed subwalk has nonpositive weight.  Hence
$a_k$ is also the supremum over all directed walks ending at $k$.  Appending
the edge $(j,k)$ to a walk ending at $j$ therefore gives
\[
 a_k\ge a_j+w(j,k)
 =a_j+\inner{\svec^j}{\zvec^k-\zvec^j}
 \qquad(j,k\in[m]).
\]
Define the finite polyhedral function
\[
 H(\zvec)\colloneq\max_{j\in[m]}
 \left\{a_j+\inner{\svec^j}{\zvec-\zvec^j}\right\}.
\]
It is proper, closed, and convex.  At $\zvec^k$, the preceding inequality
shows that the $k$th affine function has value $a_k$ and dominates every
$j$th affine function.  Thus $H(\zvec^k)=a_k$ and
$\svec^k\in\partial H(\zvec^k)$.  Finally, every closed walk decomposes into
simple cycles, so checking nontrivial simple cycles is sufficient.
\end{proof}

\begin{repeatthm}{thm:finite-prox}
Let $K\subseteq\R^d$ be nonempty, closed, and convex.  Let
$\xvec^1,\ldots,\xvec^m\in K$ be distinct and choose arbitrary targets
$\yvec^1,\ldots,\yvec^m\in K$.  Fix $\rho\in(0,1)$.  There is $\alpha_0>0$ such
that for every $\alpha\in(0,\alpha_0]$ one can find a proper
lower-semicontinuous $F_\alpha=f_\alpha+\indicator_K$ that is
$\rho$-weakly convex and satisfies
\[
 \prox_{F_\alpha}(\xvec^j)
 =(1-\alpha)\xvec^j+\alpha\yvec^j,
 \qquad j=1,\ldots,m.
\]
The proximal minimizer is unique.
\end{repeatthm}

\begin{proof}
Write $\dvec^j\colloneq\xvec^j-\yvec^j$ and define
\[
 \zvec_\alpha^j\colloneq\xvec^j-\alpha\dvec^j,
 \qquad
 \svec_\alpha^j\colloneq\xvec^j-(1-\rho)\zvec_\alpha^j.
\]
Convexity of $K$ gives $\zvec_\alpha^j\in K$ for $\alpha\in[0,1]$.
At $\alpha=0$, $\zvec_0^j=\xvec^j$ and $\svec_0^j=\rho\xvec^j$.  For every
nontrivial simple cycle $C$,
\[
 \sum_{(j,k)\in C}\inner{\svec_0^j}{\zvec_0^k-\zvec_0^j}
 =-\frac\rho2\sum_{(j,k)\in C}\norm{\xvec^k-\xvec^j}^2<0.
\]
There are finitely many simple cycles and their cycle sums vary continuously
with $\alpha$.  Thus, for all sufficiently small common $\alpha>0$, the data
$(\zvec_\alpha^j,\svec_\alpha^j)$ are cyclically monotone.  By
\Cref{lem:finite-convex-interpolation}, choose a proper closed convex $H_\alpha$
with $\svec_\alpha^j\in\partial H_\alpha(\zvec_\alpha^j)$.

Set
\[
 f_\alpha(\zvec)\colloneq H_\alpha(\zvec)-\frac\rho2\norm{\zvec}^2,
 \qquad F_\alpha\colloneq f_\alpha+\indicator_K.
\]
Then $F_\alpha$ is $\rho$-weakly convex.  Moreover, using the zero vector from
$\normal_K(\zvec_\alpha^j)$,
\[
 0=\svec_\alpha^j-\rho\zvec_\alpha^j
   +\zvec_\alpha^j-\xvec^j
 \in\partial\left(F_\alpha(\zvec)
       +\frac12\norm{\zvec-\xvec^j}^2\right)_{\zvec=\zvec_\alpha^j}.
\]
The proximal objective is $(1-\rho)$-strongly convex, so this minimizer is
unique and equals
$\zvec_\alpha^j=(1-\alpha)\xvec^j+\alpha\yvec^j$.
\end{proof}

\section{\texorpdfstring{Other Proofs and Details}{Other Proofs and Details}}\label{app:proofs}

\subsection{Proofs for \texorpdfstring{\Cref{sec:pgd}}{Section 2}}
\label{app:proofs-pgd}

First, we show that the feasibility of exact-form deviations controls the normal vectors that are introduced due to the projection step in projected online gradient descent~\eqref{eq:pgd}.

\begin{lemma}\label{lem:normal-feasible}
Let \(\phi:K\to K\), let \(\Dphi(\xvec)\colloneq\xvec-\phi(\xvec)\), and let \(\nvec\in \normal_K(\xvec)\).  Then
\[
  \inner{\nvec}{\Dphi(\xvec)}\ge 0.
\]
\end{lemma}

\begin{proof}
Since \(\phi(\xvec)\in K\) and \(\nvec\in\normal_K(\xvec)\),
\[
  \inner{\nvec}{\phi(\xvec)-\xvec}\le 0.
\]
Multiplying by \(-1\) gives the statement.
\end{proof}

\begin{repeatthm}{thm:pgd-exact}
Let $\mleft\{ \xvec_t \mright\}_{t=1}^{T+1}$ be the iterates generated by online gradient descent~\eqref{eq:pgd}. Let \(\phi:K\to K\) be an exact-form deviation $
  \Dphi(\xvec)=\xvec-\phi(\xvec)=\nabla\Psi(\xvec)$ where $\Psi$ is $L$-smooth with respect to the Euclidean norm. Then, for every sequence \(\gvec_1,\ldots,\gvec_T\),
\[
\tag{\ref{eq:pgd-exact-bound}}
  \sum_{t=1}^T \inner{\gvec_t}{\xvec_t-\phi(\xvec_t)}
  \le
  \frac{\Psi(\xvec_1)-\Psi(\xvec_{T+1})}{\eta}
  +L\eta\sum_{t=1}^T\norm{\gvec_t}^2.
\]
\end{repeatthm}

\begin{proof}
Fix a round \(t\) and define
\[
  \svec_t \colloneq \xvec_t-\xvec_{t+1},
  \qquad
  \nvec_t \colloneq \xvec_t-\eta\gvec_t-\xvec_{t+1}.
\]
By \Cref{lem:projection-residual}, \(\nvec_t\in \normal_K(\xvec_{t+1})\).  Also
\[
  \eta\gvec_t=\svec_t-\nvec_t.
\]
Thus,
\[
\numberthis{eq:split-pgd}
  \eta\inner{\gvec_t}{\Dphi(\xvec_t)}
  =\inner{\svec_t}{\Dphi(\xvec_t)}-\inner{\nvec_t}{\Dphi(\xvec_t)}.
\]
We bound the two terms separately.
First, $L$-smoothness of \(\Psi\) gives
\[
  \Psi(\xvec_{t+1})
  &\le \Psi(\xvec_t)+\inner{\nabla\Psi(\xvec_t)}{\xvec_{t+1}-\xvec_t}
     +\frac{L}{2}\norm{\xvec_{t+1}-\xvec_t}^2 \\
  &= \Psi(\xvec_t)-\inner{\Dphi(\xvec_t)}{\svec_t}
     +\frac{L}{2}\norm{\svec_t}^2.
\]
Therefore
\[
\numberthis{eq:s-term}
  \inner{\svec_t}{\Dphi(\xvec_t)}
  \le \Psi(\xvec_t)-\Psi(\xvec_{t+1})+\frac{L}{2}\norm{\svec_t}^2.
\]

Second, by \Cref{lem:normal-feasible} applied at \(\xvec_{t+1}\),
\[
  \inner{\nvec_t}{\Dphi(\xvec_{t+1})}\ge 0.
\]
Thus
\[
  -\inner{\nvec_t}{\Dphi(\xvec_t)}
  &= -\inner{\nvec_t}{\Dphi(\xvec_{t+1})}
     -\inner{\nvec_t}{\Dphi(\xvec_t)-\Dphi(\xvec_{t+1})} \\
  &\le \norm{\nvec_t}\,\norm{\Dphi(\xvec_t)-\Dphi(\xvec_{t+1})} \\
  &\le L\norm{\nvec_t}\,\norm{\svec_t}.
\]
Projection is nonexpansive and \(\xvec_t=\Proj_K(\xvec_t)\), so
\begin{equation}
  \norm{\svec_t}
  =\norm{\Proj_K(\xvec_t)-\Proj_K(\xvec_t-\eta\gvec_t)}
  \le \eta\norm{\gvec_t}.
\end{equation}
Moreover, since \(\nvec_t\in\normal_K(\xvec_{t+1})\) and \(\xvec_t\in K\),
\(\inner{\nvec_t}{\svec_t}\le0\), so
\[
  \norm{\svec_t}^2+\norm{\nvec_t}^2
  \le \norm{\svec_t-\nvec_t}^2
  =\eta^2\norm{\gvec_t}^2.
\]
Therefore, by Young's inequality,
\[
\numberthis{eq:n-term}
  -\inner{\nvec_t}{\Dphi(\xvec_t)}\le \frac{L\eta^2}{2}\norm{\gvec_t}^2.
\]
Combining \eqref{eq:split-pgd}, \eqref{eq:s-term}, and \eqref{eq:n-term}, and using \(\norm{\svec_t}\le\eta\norm{\gvec_t}\), yields
\[
  \eta\inner{\gvec_t}{\Dphi(\xvec_t)}
  \le \Psi(\xvec_t)-\Psi(\xvec_{t+1})+L\eta^2\norm{\gvec_t}^2.
\]
Dividing by \(\eta\) and summing over \(t\in[T]\) telescopes the potential
terms and proves \eqref{eq:pgd-exact-bound}.
\end{proof}

\begin{repeatcorollary}{cor:uniform-regret}
Suppose $\norm{\gvec_t}\le G$ and fix $B,L>0$.  Then OGD with
$\eta=\sqrt{B/(LG^2T)}$ satisfies
\[
 \sup_{\phi\in\PhiExact(B,L)}\Reg_T(\phi)
 \le 2G\sqrt{LBT}.
\]
Thus the algorithm controls the whole normalized class $\PhiExact(B,L)$ with one learning rate.
\end{repeatcorollary}

\begin{proof}
For every $\phi\in\PhiExact(B,L)$, the proof of
\Cref{thm:pgd-exact} applies the Lipschitz-gradient condition in
\Cref{def:normalized-exact}.  Its telescoping term is bounded by $B$, uniformly
in $\phi$, while $\sum_{t=1}^T\norm{\gvec_t}^2\le G^2T$.  Hence
\[
 \Reg_T(\phi)\le \frac B\eta+L\eta G^2T
 \quad\text{for every }\phi\in\PhiExact(B,L).
\]
Taking the supremum and substituting the displayed value of $\eta$ proves the
claim.  This common choice of $\eta$ is precisely why the normalization is
needed.
\end{proof}

\subsection{Proofs for \texorpdfstring{\Cref{sec:containment}}{Section 3}}

\begin{repeatthm}{thm:prox-contained}
Let \(F\colloneq f+\indicator_K\) be proper, lower semicontinuous, and \(\rho\)-weakly convex with \(0\le\rho<1\).  Define $\phi_F(\xvec)\colloneq\prox_F(\xvec)$. Then \(\phi_F:K\to K\) is feasible and Euclidean exact-form, with potential \(\MF\), i.e., $
  \xvec-\phi_F(\xvec)=\nabla\MF(\xvec)$, whose gradient is
  \(L_F=(2-\rho)/(1-\rho)\)-Lipschitz.
Consequently projected gradient descent controls every such proximal deviation by \Cref{thm:pgd-exact}.  In particular, if \(K\) is compact, \(\norm{\gvec_t}\le G\), \(\osc_K(\MF)\le B_F\), and $B_F,G>0$, then choosing \(\eta=\sqrt{B_F/(L_F G^2T)}\) gives
\[
  \sum_{t=1}^T \inner{\gvec_t}{\xvec_t-\prox_F(\xvec_t)}
  \le \frac{B_F}{\eta}+L_F\eta G^2T = \mathcal{O}(\sqrt{T}).
\]
If $B_F=0$ or $G=0$, the corresponding regret bound is trivial and the
displayed optimizing step size need not be used.
\end{repeatthm}

\begin{proof}
Feasibility holds because \(F=f+\indicator_K\) forces \(\prox_F(\xvec)\in K\).  Exactness is exactly \eqref{eq:weak-moreau-gradient}.  The regret bound follows by applying \Cref{thm:pgd-exact} with \(\Psi=\MF\) and the smoothness bound from \Cref{lem:weak-moreau}.
\end{proof}

\begin{repeatprop}{prop:special-cases}
The exact-form class contains the following deviations.
\begin{enumerate}[leftmargin=*,itemsep=1mm]
  \item \textbf{External regret.}  For \(\uvec\in K\), take \(F\colloneq\indicator_{\{\uvec\}}\).  Then \(\prox_F(\xvec)=\uvec\) and \(\xvec-\uvec=\nabla\frac12\norm{\xvec-\uvec}^2\).
  \item \textbf{Projection to a convex subset.}  For a nonempty closed convex set \(S\subseteq K\), take \(F\colloneq\indicator_S\).  Then \(\prox_F(\xvec)=\Proj_S(\xvec)\).
  \item \textbf{Projection-based/no-move regret.}  For a vector \(\vvec\), take \(F(\yvec)\colloneq\inner{\vvec}{\yvec}+\indicator_K(\yvec)\).  Then \(\prox_F(\xvec)=\Proj_K(\xvec-\vvec)\).
  \item \textbf{Interpolation deviations.}  For \(\uvec\in K\) and \(\alpha\in(0,1)\), the map \(\phi(\xvec)\colloneq(1-\alpha)\xvec+\alpha\uvec\) is the prox map of
  \[
    F(\yvec)\colloneq\frac{\alpha}{2(1-\alpha)}\norm{\yvec-\uvec}^2+\indicator_K(\yvec).
  \]
  The case \(\alpha=1\) is external regret.
  \item \textbf{Local proximal deviations.}  If \(\norm{\xvec-\prox_F(\xvec)}\le\delta\) on \(K\), then the deviation belongs to the local exact-form subclass \(\norm{\nabla\MF(\xvec)}\le\delta\).
  \item \textbf{Symmetric affine swaps.}  Let \(\phi(\xvec)\colloneq A\xvec+\bvec\) map \(K\) into \(K\).  Let \(T\colloneq\spanop(K-K)\) be the direction space of \(\aff(K)\), fix any \(\avec\in K\), and assume \(A T\subseteq T\) and the restriction of \(I-A\) to \(T\) is self-adjoint.  Then \(\phi\) is exact-form in the ambient sense of our definition. On \(\aff(K)\), a potential is
  \[
  \tag{\ref{eq:symmetric-affine-potential}}
    \Psi(\avec+\hvec)\colloneq\inner{\Dphi(\avec)}{\hvec}+\frac12\inner{\hvec}{(I-A)\hvec},
    \qquad \hvec\in T.
  \]
An ambient extension of this potential is given in the proof.
\end{enumerate}
\end{repeatprop}

\begin{proof}
Items 1 and 2 are immediate from \Cref{thm:prox-contained}.  For item 3, complete the square:
\[
  \argmin_{\yvec\in K}\left\{\inner{\vvec}{\yvec}+\frac12\norm{\yvec-\xvec}^2\right\}
  =\argmin_{\yvec\in K}\frac12\norm{\yvec-(\xvec-\vvec)}^2
  =\Proj_K(\xvec-\vvec).
\]
For item 4, the unconstrained first-order condition for
\[
  \frac{\alpha}{2(1-\alpha)}\norm{\yvec-\uvec}^2+\frac12\norm{\yvec-\xvec}^2
\]
 gives \(\yvec=(1-\alpha)\xvec+\alpha\uvec\), which lies in \(K\) by convexity.  Hence it is also the constrained minimizer.  Item 5 is just the Moreau identity.  For item 6, write every point in the affine hull as \(\xvec=\avec+\hvec\) with \(\hvec\in T\).  Since \(\avec\in K\) and \(\phi(\avec)\in K\), the vector \(\Dphi(\avec)=\avec-\phi(\avec)\) lies in \(T\).  For every tangent direction \(\rvec\in T\),
\[
  \Dphi(\avec+\hvec+\rvec)-\Dphi(\avec+\hvec)=(I-A)\rvec.
\]
Differentiating \eqref{eq:symmetric-affine-potential} in direction \(\rvec\),
\[
  \dd\Psi(\avec+\hvec)[\rvec]
  &=\inner{\Dphi(\avec)}{\rvec}+\inner{\hvec}{(I-A)\rvec} \\
  &=\inner{\Dphi(\avec)+(I-A)\hvec}{\rvec} \\
  &=\inner{\Dphi(\avec+\hvec)}{\rvec}.
\]
The second equality uses self-adjointness on \(T\).  Hence the affine-hull gradient of \(\Psi\) equals \(\Dphi\).  To obtain the ambient potential
required by \Cref{def:exact}, let $P_T$ be the orthogonal projection onto $T$
and define, for $\zvec\in\R^d$,
\[
 \widetilde\Psi(\zvec)
 \colloneq\inner{\Dphi(\avec)}{P_T(\zvec-\avec)}
 +\frac12\inner{P_T(\zvec-\avec)}
 {(I-A)P_T(\zvec-\avec)}.
\]
Self-adjointness on $T$ and $\Dphi(\avec)\in T$ imply
$\nabla\widetilde\Psi(\avec+\hvec)
=\Dphi(\avec)+(I-A)\hvec=\Dphi(\avec+\hvec)$ for every $\hvec\in T$.
Thus $\widetilde\Psi$ is a smooth ambient extension with the required
gradient on $K$.  Feasibility is the assumption \(\phi(K)\subseteq K\).
\end{proof}

\subsection{Proofs for \texorpdfstring{\Cref{sec:curl}}{Section 4}}

\begin{repeatlemma}{lem:zero-circulation}
If \(\field=\nabla\Psi\) on a neighborhood of \(K\), then we have \(\oint_\gamma \field(\xvec)^{\transpose}\dd\xvec=0\) for every closed loop \(\gamma\subseteq K\).
\end{repeatlemma}

\begin{proof}
By the chain rule,
\[
  \int_0^1 \inner{\nabla\Psi(\gamma(s))}{\gamma'(s)}\,\dd s
  =\int_0^1 \frac{\dd}{\dd s}\Psi(\gamma(s))\,\dd s
  =\Psi(\gamma(1))-\Psi(\gamma(0))=0.
\]
\end{proof}

\begin{lemma}\label{lem:cyclic-pgd}
Let \(\zvec_0,\zvec_1,\ldots,\zvec_m=\zvec_0\) be a closed polygon in \(K\).  Suppose
\begin{equation}
  C\colloneq\sum_{i=0}^{m-1}\inner{\field(\zvec_i)}{\zvec_i-\zvec_{i+1}}>0.
\end{equation}
Run projected GD with step size \(\eta\) from \(\xvec_1=\zvec_0\).  For one traversal, choose
\begin{equation}
  \gvec_{i+1}\colloneq\frac{\zvec_i-\zvec_{i+1}}{\eta},\qquad i=0,\ldots,m-1.
\end{equation}
Then the iterates follow the polygon exactly and the regret accumulated in one traversal is \(C/\eta\).
\end{lemma}

\begin{proof}
If \(\xvec_{i+1}=\zvec_i\), then the unprojected update is
\[
  \xvec_{i+1}-\eta\gvec_{i+1}
  =\zvec_i-(\zvec_i-\zvec_{i+1})
  =\zvec_{i+1}\in K.
\]
Projection leaves it unchanged, so \(\xvec_{i+2}=\zvec_{i+1}\).  This proves the trajectory claim by induction.  The regret over the traversal is
\[
  \sum_{i=0}^{m-1}\inner{\gvec_{i+1}}{\field(\zvec_i)}
  =\frac{1}{\eta}\sum_{i=0}^{m-1}\inner{\field(\zvec_i)}{\zvec_i-\zvec_{i+1}}
  =\frac{C}{\eta}.
\]
\end{proof}

\begin{repeatthm}{thm:curl-lower}
Let \(\field:K\to\R^d\) be continuous.  Suppose there exists a closed piecewise \(C^1\) loop \(\gamma:[0,1]\to K\) such that, after possibly reversing orientation,
\[
  -\oint_\gamma \field(\xvec)^{\transpose}\dd\xvec\colloneq c_0>0.
\]
Then for every gradient bound \(G>0\) and every sufficiently small \(\eta>0\), there exists a sequence of losses with \(\norm{\gvec_t}\le G\) such that online gradient descent, initialized at \(\gamma(0)\), suffers
\[
  \sum_{t=1}^T \inner{\gvec_t}{\field(\xvec_t)}\ge cT-O(1/\eta)
\]
for a constant \(c>0\) depending only on \(\field\), \(\gamma\), and \(G\).  In particular the regret is \(\Omega(T)\).  If the algorithm starts elsewhere in the same path-connected component of \(K\), the adversary can first steer it to the loop along any fixed polygonal path; this changes only the \(O(1/\eta)\) transient term.
\end{repeatthm}

\begin{proof}
Let \(L_\gamma\) be the length of \(\gamma\), which is positive because the circulation is nonzero.  The quantity
\[
  -\oint_\gamma \field(\xvec)^{\transpose}\dd\xvec
\]
 is the limit of left-endpoint Riemann sums
\[
  \sum_{i=0}^{m-1}\inner{\field(\zvec_i)}{\zvec_i-\zvec_{i+1}},
  \qquad \zvec_i=\gamma(s_i),
\]
as the mesh of the partition \(0=s_0<s_1<\cdots<s_m=1\) goes to zero.  Hence, for all sufficiently small \(\eta\), we may choose a partition with every chord length at most \(\eta G\), with
\[
  \sum_{i=0}^{m-1}\inner{\field(\zvec_i)}{\zvec_i-\zvec_{i+1}}\ge \frac{c_0}{2},
\]
and with \(m\le 2L_\gamma/(\eta G)\).  For one traversal, define
\[
  \gvec_{i+1}\colloneq\frac{\zvec_i-\zvec_{i+1}}{\eta},
  \qquad i=0,\ldots,m-1.
\]
Then \(\norm{\gvec_{i+1}}\le G\), and \Cref{lem:cyclic-pgd} gives regret at least \(c_0/(2\eta)\) over one traversal.  Therefore the average regret per step during complete cycles is at least
\[
  \frac{c_0/(2\eta)}{2L_\gamma/(\eta G)}=\frac{c_0G}{4L_\gamma}.
\]
Repeat the cycle for \(\lfloor T/m\rfloor\) complete traversals and set the remaining gradients, if any, to zero.  This gives
\[
  \sum_{t=1}^T\inner{\gvec_t}{\field(\xvec_t)}
  \ge \frac{c_0G}{4L_\gamma}T-O(1/\eta),
\]
which proves the claim.
\end{proof}

\begin{repeatcorollary}{cor:euclidean-classification}
Let $E_K\colloneq\spanop(K-K)$, let $U\colloneq\relint(K)$ be open and simply connected in
$\aff(K)$, and let $\field:U\to E_K$ be $C^1$.  In any orthonormal affine
coordinate system on $E_K$, $\partial_j\field_i=\partial_i\field_j$ throughout
$U$ if and only if $\field=\nabla_{E_K}\Psi$ on $U$ for a $C^2$ potential
$\Psi$. If a skew derivative is nonzero somewhere, online gradient descent can be
forced to incur linear regret against $\field$ along a sufficiently small
feasible loop.  When the field is curl-free, assume that $\Psi$ admits an
ambient extension $\widetilde\Psi$, defined on a neighborhood of $K$, that
satisfies the smoothness and bounded-oscillation assumptions of
\Cref{thm:pgd-exact,corr:regret}.  If
$\widetilde\phi(\xvec)\colloneq\xvec-\nabla\widetilde\Psi(\xvec)$ lies in $K$ for
every $\xvec\in K$, then projected GD has $O(\sqrt T)$ regret against
$\widetilde\phi$.
\end{repeatcorollary}

\begin{proof}
The equivalence between vanishing skew derivatives and existence of a
potential is the Poincar\'e lemma for $C^1$ one-forms on the simply connected
open set $U$, expressed in orthonormal affine coordinates on $E_K$.  Nonzero
curl gives, by Stokes' theorem, a sufficiently small loop contained in $U$
with nonzero circulation.  Applying the cyclic construction from the proof of
\Cref{thm:curl-lower} to this loop gives the claimed linear regret.  Under the additional
extension, feasibility, smoothness, and oscillation assumptions, the positive
statement follows from \Cref{thm:pgd-exact,corr:regret}.
\end{proof}

\subsection{Proofs for \texorpdfstring{\Cref{sec:mirror}}{Section 5}}

\begin{repeatprop}{prop:one-form-equivalence}
A field \(\field:\mathcal D\to E\) is \(R\)-mirror-exact if and only if
there exists a differentiable scalar potential \(V:\mathcal D\to\R\) such that
\[
\tag{\ref{eq:V-gradient}}
  \nabla V(\xvec)=\nabla^2R(\xvec)\field(\xvec).
\]
Equivalently, $\field(\xvec)^{\transpose}\dd\nabla R(\xvec)=\dd V(\xvec)$.
\end{repeatprop}

\begin{proof}
First suppose \(\field(\xvec)=\gradtheta\Psi(\nabla R(\xvec))\).  Define
\[
  V(\xvec)\colloneq\Psi(\nabla R(\xvec)).
\]
By the chain rule,
\[
  \nabla V(\xvec)
  =\nabla^2R(\xvec)\gradtheta\Psi(\nabla R(\xvec))
  =\nabla^2R(\xvec)\field(\xvec).
\]
Conversely, suppose \eqref{eq:V-gradient} holds.  For \(\thetavec\in\ThetaSet\), write \(\xvec\colloneq(\nabla R)^{-1}(\thetavec)\) and define
\[
  \Psi(\thetavec)\colloneq V\mleft((\nabla R)^{-1}(\thetavec) \mright).
\]
Another application of the chain rule gives
\[
  \gradtheta\Psi(\thetavec)
  =\big(\nabla^2R(\xvec)\big)^{-1}\nabla V(\xvec)
  =\field(\xvec),
\]
which is the desired dual-gradient representation.
\end{proof}

\subsection{Examples of Mirror-exact Deviations for
\texorpdfstring{\Cref{subsec:mirror-exact-examples}}{Section 5.1}}
\label{app:mirror-exact-examples}

Throughout this subsection, $R^*$ is the Fenchel conjugate defined in the
standing Legendre setup, so that
\[
  \nabla R^*(\thetavec)=(\nabla R)^{-1}(\thetavec)
\]
whenever the inverse is defined.  Each displayed map is a valid deviation once
it is pointwise feasible, i.e. once it maps \(\mathcal D\) into \(K\).

\paragraph{1. Fixed comparators: external regret is always mirror-exact.}
Fix \(\uvec\in K\) and set
\[
  \phi_{\uvec}(\xvec)\colloneq\uvec,
  \qquad
  \field_{\uvec}(\xvec)\colloneq\xvec-\uvec.
\]
Take the dual potential
\[
  \Psi_{\uvec}(\thetavec)
  \colloneq R^*(\thetavec)-\inner{\uvec}{\thetavec}.
\]
Then, with \(\thetavec=\nabla R(\xvec)\),
\[
  \gradtheta\Psi_{\uvec}(\thetavec)
  =\nabla R^*(\thetavec)-\uvec
  =\xvec-\uvec
  =\field_{\uvec}(\xvec).
\]
Thus every fixed-comparator deviation is \(R\)-mirror-exact for every mirror map.  The identity/no-move deviation \(\phi(\xvec)=\xvec\) is the special case \(\Dphi\equiv0\), with potential \(\Psi\equiv0\).

\paragraph{2. Dual translations: one mirror step on a linear loss.}
Let \(\avec\in E\).  Assume that \(\thetavec-\avec\in\ThetaSet\) for every \(\thetavec\in\ThetaSet\), and define
\[
\numberthis{eq:dual-translation-map}
  \phi_{\avec}(\xvec)
  \colloneq\nabla R^*(\nabla R(\xvec)-\avec).
\]
This is the unconstrained Bregman proximal map for the linear function \(\yvec\mapsto\inner{\avec}{\yvec}\).  It is mirror-exact with potential
\[
  \Psi_{\avec}(\thetavec)
  \colloneq R^*(\thetavec)-R^*(\thetavec-\avec),
\]
 since
\[
  \gradtheta\Psi_{\avec}(\thetavec)
  =\nabla R^*(\thetavec)-\nabla R^*(\thetavec-\avec)
  =\xvec-\phi_{\avec}(\xvec).
\]
So a fixed translation in dual space is an exact-form deviation for mirror descent.  This is often the most concrete way to remember the definition.

\paragraph{3. Mirror interpolation toward a comparator.}
Fix \(\uvec\in\mathcal D\), write \(\thetavec_{\uvec}\colloneq\nabla R(\uvec)\), and choose \(\lambda\in[0,1)\).  Assume that
\[
  (1-\lambda)\thetavec+\lambda\thetavec_{\uvec}\in\ThetaSet
  \qquad\forall\thetavec\in\ThetaSet.
\]
Define

\[
\numberthis{eq:dual-interpolation-map}
  \phi_{\lambda,\uvec}^R(\xvec)
  \colloneq\nabla R^*\big((1-\lambda)\nabla R(\xvec)+\lambda\nabla R(\uvec)\big).
\]
This moves \(\xvec\) a fraction \(\lambda\) of the way toward \(\uvec\), but the straight line is taken in dual coordinates.  The potential
\[
  \Psi_{\lambda,\uvec}(\thetavec)
  \colloneq R^*(\thetavec)
  -\frac{1}{1-\lambda}
   R^*\big((1-\lambda)\thetavec+\lambda\thetavec_{\uvec}\big)
\]
works because
\[
  \gradtheta\Psi_{\lambda,\uvec}(\thetavec)
  &=\nabla R^*(\thetavec)
    -\nabla R^*\big((1-\lambda)\thetavec+\lambda\thetavec_{\uvec}\big) \\
  &=\xvec-\phi_{\lambda,\uvec}^R(\xvec).
\]
When \(\lambda=0\), this is the identity deviation.  As \(\lambda\uparrow1\), it approaches the external deviation \(\phi_{\uvec}\), whose potential was given above.

\paragraph{4. Entropy on the simplex: multiplicative deviations.}
Let \(K\colloneq\simplex_d\) and work on \(\relint(\simplex_d)\).  Take the entropic mirror map
\[
  R(\xvec)\colloneq\sum_{i=1}^d \xvec[i]\log \xvec[i].
\]
All derivatives are understood on the affine hull of the simplex.  Equivalently,
one may identify dual vectors that differ by a multiple of $\mathbf 1$, or fix
the gauge $\sum_i\thetavec[i]=0$.  In these logit coordinates,
\(\nabla R^*(\thetavec)=\softmax(\thetavec)\).  Hence the dual translation \eqref{eq:dual-translation-map} becomes
\[
  \phi_{\avec}(\xvec)[i]
  \colloneq\frac{\xvec[i]\exp(-\avec[i])}
  {\sum_{j=1}^d \xvec[j]\exp(-\avec[j])}.
\]
This is multiplicative-weights reweighting by \(\exp(-\avec)\), and it is pointwise feasible on the simplex.  The potential is the log-sum-exp difference
\[
  \Psi_{\avec}(\thetavec)
  \colloneq\log\sum_{i=1}^d\exp(\thetavec[i])
  -\log\sum_{i=1}^d\exp(\thetavec[i]-\avec[i]),
\]
whose gradient is
\[
  \gradtheta\Psi_{\avec}(\thetavec)
  =\softmax(\thetavec)-\softmax(\thetavec-\avec).
\]
Since \(\softmax(\nabla R(\xvec))=\xvec\), this equals \(\xvec-\phi_{\avec}(\xvec)\).

The dual interpolation \eqref{eq:dual-interpolation-map} gives another familiar entropy example.  For \(\uvec\in\relint(\simplex_d)\),
\[
\phi_{\lambda,\uvec}^R(\xvec)[i]
  \colloneq\frac{\xvec[i]^{1-\lambda}\uvec[i]^\lambda}
  {\sum_{j=1}^d \xvec[j]^{1-\lambda}\uvec[j]^\lambda}.
\]
Thus mirror-exact deviations include geometric interpolation on the simplex, not only Euclidean arithmetic interpolation.

\paragraph{5. Metric-symmetric affine swaps.}
Let $K\colloneq E\colloneq\R^d$ and
\(R_H(\xvec)\colloneq\frac12\xvec^{\transpose}H\xvec\), where \(H\succ0\).  This is
a Legendre function with full domain.  Then \(\nabla R_H(\xvec)=H\xvec\), and
\Cref{prop:one-form-equivalence} says that \(\field\) is
\(R_H\)-mirror-exact exactly when \(H\field(\xvec)\) is an ordinary gradient field.

Consider an affine deviation
\[
  \phi(\xvec)\colloneq A\xvec+\bvec,
  \qquad
  \field(\xvec)\colloneq\xvec-\phi(\xvec)=(I-A)\xvec-\bvec.
\]
If
\[
\numberthis{eq:H-symmetric-swap}
  H(I-A)=(I-A)^{\transpose}H,
\]
then \(\field\) is \(R_H\)-mirror-exact; feasibility is automatic because
$K=\R^d$.  Indeed,
\[
    V(\xvec)\colloneq\frac12\xvec^{\transpose}H(I-A)\xvec
  -\inner{H\bvec}{\xvec}
\]
satisfies \(\nabla V(\xvec)=H\field(\xvec)\).  For \(H=I\), condition \eqref{eq:H-symmetric-swap} says that \(I-A\) is symmetric, recovering the symmetric affine swaps in Euclidean geometry.  For general \(H\), the correct symmetry is symmetry in the mirror metric.

\paragraph{6. A mirror-exact field that is not Euclidean-exact.}
The previous example gives a feasible separation on an unconstrained action
space.  Let
\[
H\colloneq\begin{pmatrix}1&0\\0&2\end{pmatrix},
  \qquad K\colloneq\R^2,
\]
use \(R_H(\xvec)\colloneq\frac12\xvec^{\transpose}H\xvec\), and for \(\eps\ge0\) define
\[
  \field(\xvec)[1]\colloneq\eps(\xvec[1]+\xvec[2]),
  \qquad
  \field(\xvec)[2]\colloneq\frac{\eps}{2}(\xvec[1]+\xvec[2]),
  \qquad
  \phi(\xvec)\colloneq\xvec-\field(\xvec).
\]
The map is feasible because it maps $\R^2$ to itself.  Writing
$S\colloneq\begin{psmallmatrix}1&1\\1&1\end{psmallmatrix}$, it is
\(R_H\)-mirror-exact because
\[
  H\field(\xvec)=\eps S\xvec
  =\nabla\left(\frac{\eps}{2}(\xvec[1]+\xvec[2])^2\right).
\]
But \(\field\) is not Euclidean-exact when \(\eps>0\), since
\[
  \partial_2 \field(\xvec)[1]=\eps,
  \qquad
  \partial_1 \field(\xvec)[2]=\frac{\eps}{2}.
\]
So mirror-exact regret is not just Euclidean exact-form regret written in different notation.  The mirror map changes which deviations telescope.

\subsection{Proofs for \texorpdfstring{\Cref{sec:omd}}{Section 5.2}}
\label{app:proof_appendix}

\begin{repeatthm}{thm:constrained-md}
Assume the standing Legendre-domain conditions and that $R$ is $m$-strongly
convex on $\mathcal D$.  Let \(\phi:\mathcal D\to K\) be feasible and
\(R\)-mirror-exact with dual potential \(\Psi\). Let $\{ \xvec_t\}_{t=1}^{T+1}$ be the
iterates generated by online mirror descent~\eqref{eq:md}, initialized in
$\mathcal D$, with iterates remaining in $\mathcal D$. Write
\[
  \thetavec_t\colloneq\nabla R(\xvec_t),
  \qquad
  \Dphi(\xvec_t)=\gradtheta\Psi(\thetavec_t).
\]
Assume that \eqref{eq:relative-smoothness} holds for every consecutive pair
$(\thetavec_{t+1},\thetavec_t)$, $t\in[T]$. A sufficient
trajectory-independent condition is that it hold for every ordered pair in
$\ThetaSet\times\ThetaSet$.
Then
\[
\tag{\ref{eq:constrained-md-bound}}
\sum_{t=1}^T\inner{\gvec_t}{\Dphi(\xvec_t)}
  \le
  \frac{\Psi(\thetavec_1)-\Psi(\thetavec_{T+1})}{\eta}
  +
  \frac{L\eta}{m}
  \sum_{t=1}^T\norm{\gvec_t}^2.
\]
Consequently, suppose \(\norm{\gvec_t}\le G\) and that, for the
horizon-dependent choices of constant step size considered below, there is a
constant $B$, independent of the horizon, such that
\[
 \max_{1\le t\le T+1}\Psi(\thetavec_t)
 -\min_{1\le t\le T+1}\Psi(\thetavec_t)\le B,
\]
Then choosing \(\eta\asymp 1/\sqrt T\) gives \(O(\sqrt T)\) regret against
\(\phi\).  More precisely, when $B,L,G>0$, the choice
$\eta=\sqrt{mB/(LG^2T)}$ gives the bound
$2G\sqrt{LBT/m}$.  The global condition
$\osc_{\ThetaSet}(\Psi)\le B$ is a sufficient trajectory-independent
assumption that avoids any dependence of the oscillation bound on the chosen
step size.
\end{repeatthm}

\begin{proof}
Because $\xvec_{t+1}\in\mathcal D=\relint(K)$ and normal cones are taken in
the affine hull of $K$, the relative normal cone at $\xvec_{t+1}$ is
$\{0\}$.  The first-order optimality condition for \eqref{eq:md} therefore
reduces to the unconstrained dual update
\[
  \thetavec_t-\thetavec_{t+1}=\eta\gvec_t.
\]
Since \(\Dphi(\xvec_t)=\gradtheta\Psi(\thetavec_t)\), relative smoothness and
Legendre conjugacy give
\[
  \eta\inner{\gvec_t}{\Dphi(\xvec_t)}
  &=
  \Psi(\thetavec_t)-\Psi(\thetavec_{t+1})
  +
  D_\Psi(\thetavec_{t+1},\thetavec_t) \\
  &\le
  \Psi(\thetavec_t)-\Psi(\thetavec_{t+1})
  +
  L\,D_{R^*}(\thetavec_{t+1},\thetavec_t).
\]
By the Legendre conjugacy identity stated with the standing assumptions,
\[
  D_{R^*}(\thetavec_{t+1},\thetavec_t)
  =
  \DR(\xvec_t,\xvec_{t+1}).
\]
Therefore
\[
  \eta\inner{\gvec_t}{\Dphi(\xvec_t)}
  \le
  \Psi(\thetavec_t)-\Psi(\thetavec_{t+1})
  +
  L\DR(\xvec_t,\xvec_{t+1}). \numberthis{eq:mirror-interior-potential}
\]
Let $\svec_t\colloneq\xvec_t-\xvec_{t+1}$.  The symmetrized Bregman identity and the
dual update imply
\[
  \DR(\xvec_t,\xvec_{t+1})+
  \DR(\xvec_{t+1},\xvec_t)
  =
  \inner{\thetavec_t-\thetavec_{t+1}}{\svec_t}
  =\eta\inner{\gvec_t}{\svec_t}.
\]
Strong convexity of $R$ therefore yields
\[
  m\norm{\svec_t}^2
  \le
  \eta\inner{\gvec_t}{\svec_t}
  \le
  \eta\norm{\gvec_t}\norm{\svec_t}.
\]
Hence \(\norm{\svec_t}\le(\eta/m)\norm{\gvec_t}\), with the conclusion trivial when
\(\svec_t=0\).  Consequently,
\[
  \DR(\xvec_t,\xvec_{t+1})
  \le \eta\inner{\gvec_t}{\svec_t}
  \le
  \frac{\eta^2}{m}\norm{\gvec_t}^2.
\]
Combining this estimate with \eqref{eq:mirror-interior-potential} gives the
slightly stronger one-step bound
\[
  \eta\inner{\gvec_t}{\Dphi(\xvec_t)}
  \le
  \Psi(\thetavec_t)-\Psi(\thetavec_{t+1})
  +
  \frac{L\eta^2}{m}\norm{\gvec_t}^2.
\]
Dividing by \(\eta\) and summing over \(t\) proves
\eqref{eq:constrained-md-bound}.  The \(O(\sqrt T)\) statement follows
by bounding the telescoping term by \(B/\eta\) and optimizing \(\eta\).
\end{proof}

\subsection{Proofs for \texorpdfstring{\Cref{sec:bregman}}{Section 5.3}}
\label{app:bregman_proximal}

\begin{lemma}\label{lem:bregman-derivative}
For fixed \(\yvec\in\dom R\),
\[
\nabla_{\xvec}\DR(\yvec,\xvec)
  =\nabla^2R(\xvec)(\xvec-\yvec).
\]
\end{lemma}

\begin{proof}
By definition,
\[
  \DR(\yvec,\xvec)
  \colloneq R(\yvec)-R(\xvec)-\inner{\nabla R(\xvec)}{\yvec-\xvec}.
\]
Only the last two terms depend on \(\xvec\).  Let
\[
  h(\xvec)=\inner{\nabla R(\xvec)}{\yvec-\xvec}.
\]
Then
\[
  \nabla h(\xvec)
  =\nabla^2R(\xvec)(\yvec-\xvec)-\nabla R(\xvec).
  \]
Thus
\[
\nabla_{\xvec}\DR(\yvec,\xvec)
  &=-\nabla R(\xvec)-\nabla h(\xvec) \\
  &=\nabla^2R(\xvec)(\xvec-\yvec).
\]
\end{proof}

\begin{lemma}[Bregman-Moreau identity]\label{lem:bregman-moreau}
Assume the standing Legendre-domain conditions. Let
$F\colloneq f+\indicator_K$ be proper and lower semicontinuous, assume that every
Bregman proximal minimum is attained, and suppose that \(\prox_F^R\) is
single-valued and continuous on \(\mathcal D\). Then \(\MRF\) is
differentiable on \(\mathcal D\) and
\[
\numberthis{eq:bregman-moreau-gradient}
  \nabla \MRF(\xvec)
  =\nabla^2R(\xvec)\big(\xvec-\prox_F^R(\xvec)\big).
\]
Consequently, if \(\thetavec=\nabla R(\xvec)\) and
\[
  \Psi_F(\thetavec)
  =\MRF\big((\nabla R)^{-1}(\thetavec)\big),
\]
then
\[
\numberthis{eq:bregman-mirror-gradient}
  \gradtheta\Psi_F(\thetavec)
  =\xvec-\prox_F^R(\xvec),
  \qquad
  \xvec=(\nabla R)^{-1}(\thetavec).
\]
\end{lemma}

\begin{proof}
The key point is that we do not differentiate the minimizer.  Continuity of
$\prox_F^R$ makes the minimizing points locally bounded as $\xvec$ varies, so
the standard local form of Danskin's theorem applies even though the ambient
minimization domain need not be bounded.  Uniqueness then implies
\[
  \nabla \MRF(\xvec)
  =
  \nabla_{\xvec}\DR(\prox_F^R(\xvec),\xvec).
\]
Using \Cref{lem:bregman-derivative}, we obtain
\[
  \nabla \MRF(\xvec)
  =
  \nabla^2R(\xvec)\big(\xvec-\prox_F^R(\xvec)\big),
\]
which proves \eqref{eq:bregman-moreau-gradient}.

For the dual identity, write \(\xvec\colloneq(\nabla R)^{-1}(\thetavec)\).  Since
\[
  \frac{\partial \xvec}{\partial\thetavec}
  =
  \big(\nabla^2R(\xvec)\big)^{-1},
\]
the chain rule gives
\[
  \gradtheta\Psi_F(\thetavec)
  =
  \big(\nabla^2R(\xvec)\big)^{-1}\nabla \MRF(\xvec)
  =
  \xvec-\prox_F^R(\xvec).
\]
\end{proof}

\begin{repeatthm}{thm:bregman-contained}
Assume the standing Legendre-domain conditions.  Let
$F\colloneq f+\indicator_K$ be proper and lower semicontinuous, assume that every
Bregman proximal minimum above is attained, and suppose that \(\prox_F^R\) is
single-valued and continuous on $\mathcal D$.  Then the Bregman proximal
deviation $\phi_F^R:\mathcal D\to K$ defined by
$\phi_F^R(\xvec)\colloneq\prox_F^R(\xvec)$
is feasible and \(R\)-mirror-exact.  Equivalently,
\[
\tag{\ref{eq:bregman-one-form-exact}}
  \big(\xvec-
  \prox_F^R(\xvec)\big)^{\transpose}\dd\nabla R(\xvec)
  =\dd\MRF(\xvec).
\]
Thus Bregman proximal regret is a special case of mirror-exact regret.
\end{repeatthm}

\begin{proof}
Feasibility follows from the indicator \(\indicator_K\), because every minimizer of the Bregman proximal problem lies in \(K\).  By \eqref{eq:bregman-mirror-gradient} in \Cref{lem:bregman-moreau}, the displacement
\[
  \xvec-\phi_F^R(\xvec)=\xvec-\prox_F^R(\xvec)
\]
 equals the dual gradient \(\gradtheta\Psi_F(\nabla R(\xvec))\).  This is precisely \(R\)-mirror-exactness.  The differential form \eqref{eq:bregman-one-form-exact} is the same identity written in primal coordinates using \eqref{eq:bregman-moreau-gradient} in \Cref{lem:bregman-moreau}.
\end{proof}

\begin{repeatprop}{prop:bregman-relative-smooth}
Assume the standing Legendre-domain conditions.  Assume, in the
extended-valued sense on $E$, that $R$ is $m$-strongly convex and that
$F\colloneq f+\indicator_K$ is proper, closed, and $\rho$-weakly convex, where
$0\le\rho<m$ and $H\colloneq F+R$ is proper.  Then the Bregman proximal subproblem has a
unique solution on $\mathcal D$, and its solution map is continuous there.
Let $\{\xvec_t\}_{t=1}^{T+1}\subset\mathcal D$ be generated by online mirror
descent~\eqref{eq:md}.  Then these iterates control Bregman proximal regret with
the bound
\[
\tag{\ref{eq:bregman-prox-md-bound}}
  \sum_{t=1}^T
  \inner{\gvec_t}{\xvec_t-\prox_F^R(\xvec_t)}
  \le
  \frac{\Psi_F(\thetavec_1)-\Psi_F(\thetavec_{T+1})}{\eta}
  +\frac{\eta}{m}\sum_{t=1}^T\norm{\gvec_t}^2,
\]
where $\Psi_F(\thetavec_t) = \MRF(\xvec_t)$ for all $t \in [T+1]$.
\end{repeatprop}

\begin{proof}
Since $F$ is $\rho$-weakly convex and $R$ is $m$-strongly convex in the
extended-valued sense, the proper closed function $H=F+R$ is
$(m-\rho)$-strongly convex.  For $\thetavec\in\ThetaSet$, the Bregman proximal
problem differs by an $\xvec$-dependent constant from
\[
 \min_{\yvec\in E}\{H(\yvec)-\inner{\thetavec}{\yvec}\}.
\]
It therefore has the unique solution $\nabla H^*(\thetavec)$.  Moreover,
$H^*$ is differentiable with $1/(m-\rho)$-Lipschitz gradient, so the proximal
solution map is continuous on $\mathcal D$.

Let \(\xvec\colloneq\nabla R^*(\thetavec)\).  By the definition of the Bregman envelope,
\[
  \Psi_F(\thetavec)
  &=\MRF(\xvec) \\
  &=\min_{\yvec}
    \left\{F(\yvec)+R(\yvec)-R(\xvec)
    -\inner{\nabla R(\xvec)}{\yvec-\xvec}\right\} \\
  &=\min_{\yvec}
    \left\{F(\yvec)+R(\yvec)-\inner{\thetavec}{\yvec}\right\}
    +\inner{\thetavec}{\xvec}-R(\xvec) \\
  &=-(F+R)^*(\thetavec)+R^*(\thetavec).
\]
This proves $\Psi_F(\thetavec)=R^*(\thetavec)-H^*(\thetavec)$ for the function $H \colloneq F + R$.

Because $H^*$ is convex and differentiable on the dual region,
\[
  D_{H^*}(\thetavec',\thetavec)\ge0.
\]
Taking the Bregman divergence of the identity \(\Psi_F=R^*-H^*\) gives
\[
  D_{\Psi_F}(\thetavec',\thetavec)
  =D_{R^*}(\thetavec',\thetavec)-D_{H^*}(\thetavec',\thetavec)
  \le D_{R^*}(\thetavec',\thetavec),
\]
which proves the relative-smoothness with $L=1$ for the potential function $\Psi_F$.  Finally, \Cref{thm:bregman-contained} identifies the Bregman proximal displacement with \(\gradtheta\Psi_F(\thetavec_t)\).  Applying \Cref{thm:constrained-md} with \(L=1\) gives \eqref{eq:bregman-prox-md-bound}.
\end{proof}

\subsection{Proofs for \texorpdfstring{\Cref{sec:md-beyond-legendre}}{Section 8.4}}
\label{app:md-beyond-legendre}

\begin{repeatthm}{thm:md-exact-beyond-legendre}
Assume $R$ is differentiable on a neighborhood of $K$, and let $\{\xvec_t\}_{t=1}^{T+1}$ be generated by mirror
descent \eqref{eq:md} with a constant step size $\eta>0$.  Let $\phi:K\to K$ be
$R$-mirror-exact with potential $\Psi$.  Suppose
\eqref{eq:relative-smoothness-beyond-legendre} holds and $\Dphi$ is
$L_\phi$-Lipschitz on $K$.  Then
\[
\begin{aligned}
  \sum_{t=1}^T\inner{\gvec_t}{\Dphi(\xvec_t)}
  &\le
  \frac{\Psi(\nabla R(\xvec_1))-\Psi(\nabla R(\xvec_{T+1}))}{\eta}
  +\frac{L\eta}{m}\sum_{t=1}^T\dualnorm{\gvec_t}^2\\
  &\quad+
  \frac{L_\phi}{m}\sum_{t=1}^T
  \dualnorm{\gvec_t}
  \left[
  \eta\dualnorm{\gvec_t}
  +\omega_R\!\left(\frac{\eta}{m}\dualnorm{\gvec_t}\right)
  \right].
\end{aligned}
\]
\end{repeatthm}

\begin{proof}
Set
\[
  \thetavec_t\colloneq\nabla R(\xvec_t),
  \qquad
  \svec_t=\xvec_t-\xvec_{t+1}.
\]
The first-order optimality condition for the mirror-descent subproblem
provides $\nvec_{t+1}\in\normal_K(\xvec_{t+1})$ such that
\[
\numberthis{eq:md-normal-decomposition}
  \eta\gvec_t
  =\thetavec_t-\thetavec_{t+1}-\nvec_{t+1}.
\]
The symmetrized Bregman identity gives
\[
\begin{aligned}
  \DR(\xvec_t,\xvec_{t+1})
  +\DR(\xvec_{t+1},\xvec_t)
  &=\inner{\thetavec_t-\thetavec_{t+1}}{\svec_t}\\
  &=\eta\inner{\gvec_t}{\svec_t}
    +\inner{\nvec_{t+1}}{\svec_t}.
\end{aligned}
\]
Since $\xvec_t\in K$ and
$\nvec_{t+1}\in\normal_K(\xvec_{t+1})$,
\[
  \inner{\nvec_{t+1}}{\svec_t}\le0.
\]
Strong convexity of $R$ therefore implies
\[
  m\norm{\svec_t}^2
  \le\eta\dualnorm{\gvec_t}\norm{\svec_t}.
\]
Consequently,
\[
\numberthis{eq:md-step-bound-beyond-legendre}
  \norm{\svec_t}
  \le\frac{\eta}{m}\dualnorm{\gvec_t},
  \qquad
  \DR(\xvec_t,\xvec_{t+1})
  \le\frac{\eta^2}{m}\dualnorm{\gvec_t}^2.
\]
The normal decomposition and the definition of $\omega_R$ also yield
\[
\numberthis{eq:md-normal-bound-beyond-legendre}
  \dualnorm{\nvec_{t+1}}
  \le
  \eta\dualnorm{\gvec_t}
  +\omega_R\!\left(\frac{\eta}{m}\dualnorm{\gvec_t}\right).
\]

Write $\field_t\colloneq\Dphi(\xvec_t)$.  By
\eqref{eq:md-normal-decomposition},
\[
  \eta\inner{\gvec_t}{\field_t}
  =
  \inner{\thetavec_t-\thetavec_{t+1}}{\field_t}
  -\inner{\nvec_{t+1}}{\field_t}.
\]
Mirror exactness and
\eqref{eq:relative-smoothness-beyond-legendre} give
\[
\begin{aligned}
  \inner{\thetavec_t-\thetavec_{t+1}}{\field_t}
  &=\Psi(\thetavec_t)-\Psi(\thetavec_{t+1})
    +D_\Psi(\thetavec_{t+1},\thetavec_t)\\
  &\le
  \Psi(\thetavec_t)-\Psi(\thetavec_{t+1})
  +L \DR(\xvec_t,\xvec_{t+1}).
\end{aligned}
\]

It remains to control the normal term.  Feasibility of $\phi$ gives
$\phi(\xvec_{t+1})\in K$, hence
\[
  \inner{\nvec_{t+1}}{\Dphi(\xvec_{t+1})}\ge0.
\]
Therefore
\[
\begin{aligned}
  -\inner{\nvec_{t+1}}{\field_t}
  &=
  -\inner{\nvec_{t+1}}{\Dphi(\xvec_{t+1})}
  -\inner{\nvec_{t+1}}
  {\Dphi(\xvec_t)-\Dphi(\xvec_{t+1})}\\
  &\le
  L_\phi\dualnorm{\nvec_{t+1}}\norm{\svec_t}.
\end{aligned}
\]
Substituting
\eqref{eq:md-step-bound-beyond-legendre} and
\eqref{eq:md-normal-bound-beyond-legendre} gives
\[
\begin{aligned}
  \eta\inner{\gvec_t}{\field_t}
  &\le
  \Psi(\thetavec_t)-\Psi(\thetavec_{t+1})
  +\frac{L\eta^2}{m}\dualnorm{\gvec_t}^2\\
  &\quad+
  \frac{L_\phi\eta}{m}\dualnorm{\gvec_t}
  \left[
  \eta\dualnorm{\gvec_t}
  +\omega_R\!\left(\frac{\eta}{m}\dualnorm{\gvec_t}\right)
  \right].
\end{aligned}
\]
Divide by $\eta$ and sum over $t$.  The potential differences telescope.
\end{proof}

\begin{repeatcorollary}{cor:md-exact-beyond-legendre-rate}
Assume the hypotheses of \Cref{thm:md-exact-beyond-legendre},
$\dualnorm{\gvec_t}\le G$, and
$\osc_{\nabla R(K)}(\Psi)\le B$.  For $\eta=T^{-1/2}$,
\[
  \frac1T\sum_{t=1}^T\inner{\gvec_t}{\Dphi(\xvec_t)}
  \le
  \frac{B}{\sqrt T}
  +\frac{(L+L_\phi)G^2}{m\sqrt T}
  +\frac{L_\phi G}{m}
  \omega_R\!\left(\frac{G}{m\sqrt T}\right).
\]
Consequently, mirror descent has $o(T)$ regret against $\phi$.  If
$\nabla R$ is $M$-Lipschitz on $K$, then
\[
  \sum_{t=1}^T\inner{\gvec_t}{\Dphi(\xvec_t)}
  \le \frac{B}{\eta}+C_{R,\phi}\eta G^2T,
  \qquad
  C_{R,\phi}\colloneq\frac{L+L_\phi}{m}+\frac{L_\phi M}{m^2}.
\]
The optimized bound is
$2G\sqrt{B C_{R,\phi}T}$ whenever $B C_{R,\phi}>0$.
\end{repeatcorollary}

\begin{proof}
Apply \Cref{thm:md-exact-beyond-legendre} and bound the telescoping term
by $B$.  Uniform continuity of $\nabla R$ on $K$ gives
\[
  \omega_R\!\left(\frac{G}{m\sqrt T}\right)\longrightarrow0,
\]
so the average regret converges to zero.  If
$\omega_R(\delta)\le M\delta$, substitution gives the stated value of
$C_{R,\phi}$.  The minimum of
$B/\eta+C_{R,\phi}\eta G^2T$ is
$2G\sqrt{BC_{R,\phi}T}$ when $BC_{R,\phi}>0$.
\end{proof}

\begin{repeatthm}{thm:md-mirror-circulation}
Suppose that, after reversing the orientation of $\gamma$ if necessary,
\[
  -\oint_\gamma
  \field(\xvec)^{\transpose}\dd\nabla R(\xvec)\colloneq c_0>0.
\]
For every $G>0$ and every sufficiently small constant step size $\eta>0$,
there is a sequence with $\dualnorm{\gvec_t}\le G$ such that standard
constrained mirror descent, initialized at $\gamma(0)$, satisfies
\[
  \sum_{t=1}^T\inner{\gvec_t}{\field(\xvec_t)}
  \ge cT-O(1/\eta)
\]
for a constant $c>0$ depending only on $\field$, $\gamma$, $R$, and $G$.
\end{repeatthm}

\begin{proof}
Let
\[
  \thetavec(s)=\nabla R(\gamma(s)).
\]
The curve $\thetavec$ is closed and rectifiable.  Its length is positive
because the circulation is nonzero.  The negative circulation is the
limit of the left Riemann sums
\[
  \sum_{j=0}^{q-1}
  \inner{\field(\zvec_j)}{\thetavec_j-\thetavec_{j+1}},
  \qquad
  \zvec_j=\gamma(s_j),\quad
  \thetavec_j=\thetavec(s_j).
\]
For every sufficiently small $\eta$, choose a partition
$0=s_0<\cdots<s_q=1$ for which
\[
  \dualnorm{\thetavec_j-\thetavec_{j+1}}\le\eta G,
  \qquad
  q\le\frac{C_\gamma}{\eta},
\]
and
\[
\numberthis{eq:md-circulation-riemann-sum}
  \sum_{j=0}^{q-1}
  \inner{\field(\zvec_j)}{\thetavec_j-\thetavec_{j+1}}
  \ge\frac{c_0}{2}.
\]
Such partitions exist because $\thetavec$ is piecewise $C^1$ and
$\field\circ\gamma$ is continuous.

For one traversal, initialize at $\zvec_0$ and set
\[
  \gvec_{j+1}
  \colloneq\frac{\thetavec_j-\thetavec_{j+1}}{\eta},
  \qquad j=0,\ldots,q-1.
\]
The dual norm of every gradient is at most $G$.  The mirror-descent
objective at this step is
\[
  \yvec\longmapsto
  \eta\inner{\gvec_{j+1}}{\yvec}+\DR(\yvec,\zvec_j).
\]
Its gradient at $\zvec_{j+1}$ is
\[
  \eta\gvec_{j+1}
  +\nabla R(\zvec_{j+1})-\nabla R(\zvec_j)=0.
\]
The objective is $m$-strongly convex on $K$, so
$\zvec_{j+1}$ is its unique minimizer.  The constrained
mirror-descent iterates therefore follow the loop exactly, including at
boundary points.

By \eqref{eq:md-circulation-riemann-sum}, the regret during one
traversal is at least
\[
  \sum_{j=0}^{q-1}
  \inner{\gvec_{j+1}}{\field(\zvec_j)}
  \ge\frac{c_0}{2\eta}.
\]
Since $q\le C_\gamma/\eta$, complete traversals have average regret at
least $c_0/(2C_\gamma)$.  Repeating the traversal and assigning zero
gradients to the remaining rounds gives
\[
  \sum_{t=1}^T\inner{\gvec_t}{\field(\xvec_t)}
  \ge\frac{c_0}{2C_\gamma}T-O(1/\eta).
\]
\end{proof}

\begin{repeatcorollary}{cor:md-full-mirror-characterization}
Let $U\colloneq\relint(K)$.  The set $U$ is open and convex in $\aff(K)$,
and hence simply connected there.  Suppose that $R\in C^2(U)$ and
$\nabla^2R(\xvec)$ is positive definite for every $\xvec\in U$, and the coefficient field $\xvec\longmapsto\nabla^2R(\xvec)\field(\xvec)$
is $C^1$.  All derivatives and coordinate components below are taken in
affine coordinates on $\aff(K)$.  Then the following conditions are
equivalent.
\begin{enumerate}[label=\textup{(\roman*)},leftmargin=*]
  \item The one-form
  $\field(\xvec)^{\transpose}\dd\nabla R(\xvec)$ is exact on $U$.
  \item For every coordinate pair $i,j$,
  \[
    \partial_j\!\bigl[\nabla^2R(\xvec)\field(\xvec)\bigr]_i
    =
    \partial_i\!\bigl[\nabla^2R(\xvec)\field(\xvec)\bigr]_j
    \qquad(\xvec\in U).
  \]
  \item The mirror circulation vanishes on every closed piecewise $C^1$
  loop in $U$.
\end{enumerate}
If these conditions fail, standard mirror descent can be forced to incur
linear regret along a sufficiently small loop contained in $U$.  Suppose
instead that they hold, that $\field$ extends to the displacement $\Dphi$ of a
feasible map on $K$, and that the resulting potential and displacement satisfy
the regularity hypotheses of \Cref{thm:md-exact-beyond-legendre}.  Then mirror
descent has sublinear regret.  Under the Lipschitz-gradient assumptions in
Corollary~\ref{cor:md-exact-beyond-legendre-rate}, the regret is $O(\sqrt T)$.
\end{repeatcorollary}

\begin{proof}
In affine coordinates on $U=\relint(K)$, the coefficient field of the
one-form is
\[
  \betavec(\xvec)=\nabla^2R(\xvec)\field(\xvec).
\]
The Poincar\'e lemma gives the equivalence between exactness of
$\betavec(\xvec)^{\transpose}\dd\xvec$ and symmetry of the cross
derivatives of $\betavec$.  Exact forms have zero integral around every
closed loop.  Conversely, if every closed-loop integral vanishes, the
line integral from a fixed base point is path independent and defines a
potential.

If a cross derivative is nonzero at some point, continuity and Stokes'
theorem give a sufficiently small rectangular loop in $U$ with nonzero
mirror circulation.  The linear lower bound follows from
\Cref{thm:md-mirror-circulation}.

Suppose instead that the form is exact.  Strong convexity makes
$\nabla R$ injective on $U$, and positive definiteness of
$\nabla^2R$ makes it a local diffeomorphism.  Hence
$\nabla R:U\to\nabla R(U)$ is a diffeomorphism onto an open set.  If
$V$ is a potential for the one-form, define
\[
  \Psi(\thetavec)\colloneq V\bigl((\nabla R)^{-1}(\thetavec)\bigr).
\]
The chain rule gives
\[
  \gradtheta\Psi(\nabla R(\xvec))=\field(\xvec).
\]
The extension, feasibility, and regularity assumptions in the
corollary then allow
\Cref{thm:md-exact-beyond-legendre} to be applied on $K$.
\end{proof}

\subsection{Proofs for \texorpdfstring{\Cref{sec:mirror-interpolation}}{Section 8.5}}
\label{app:mirror-interpolation}

\begin{repeatprop}{prop:bregman-radial-exact}
Assume $\mathcal X=K$.  For every admissible generator $F$ in
Definition~\ref{def:bregman-radial-class}, the function $
  \Psi_F(\thetavec)
  \colloneq
  R_K^*(\thetavec)-(F+R)^*(\thetavec)$
is differentiable and satisfies
\[
  \xvec-\prox_F^R(\xvec)
  =\gradtheta\Psi_F(\nabla R(\xvec))
  \qquad(\xvec\in K).
\]
Consequently, $
  \PhiMirrorProx{R}(K)
  \subseteq
  \PhiMirrorRad{R}(K)
  \subseteq
  \PhiMirrorExact{R}(K)$,
where $\PhiMirrorExact{R}(K)$ is the class from
Definition~\ref{def:mirror-exact-beyond-legendre}.  More precisely, if
$\phi_\alpha=\prox_F^R$, then $\phi$ is mirror-exact with potential
$\Psi_F/\alpha$.
\end{repeatprop}

\begin{proof}
Since $R_K$ is proper, closed, and $m$-strongly convex, and $F+R$ is proper,
closed, and $(m-\rho)$-strongly convex, their conjugates are differentiable.
For every $\xvec\in K$,
\[
  \nabla R(\xvec)\in\partial R_K(\xvec),
\]
because $0\in\normal_K(\xvec)$.  Fenchel duality therefore gives
\[
  \nabla R_K^*(\nabla R(\xvec))=\xvec.
\]
If $\pvec=\prox_F^R(\xvec)$, the proximal optimality condition gives
\[
  \nabla R(\xvec)\in\partial(F+R)(\pvec),
\]
and hence
\[
  \nabla(F+R)^*(\nabla R(\xvec))=\pvec.
\]
Subtracting the two identities yields
\[
  \gradtheta\Psi_F(\nabla R(\xvec))
  =\xvec-\prox_F^R(\xvec),
\]
so every map in $\PhiMirrorProx{R}(K)$ is mirror-exact.

The inclusion
$\PhiMirrorProx{R}(K)\subseteq\PhiMirrorRad{R}(K)$ follows by taking
$\alpha=1$.  If $\phi_\alpha=\prox_F^R$, then
\[
  \xvec-\phi(\xvec)
  =\frac1\alpha\bigl(\xvec-\phi_\alpha(\xvec)\bigr)
  =\gradtheta\left(\frac{\Psi_F}{\alpha}\right)(\nabla R(\xvec)).
\]
Thus $\phi$ is mirror-exact with potential $\Psi_F/\alpha$, which proves
$\PhiMirrorRad{R}(K)\subseteq\PhiMirrorExact{R}(K)$ and completes the proof.
\end{proof}

\begin{repeatprop}{prop:legendre-prox-necessary}
Suppose $\phi_\alpha=\prox_{F_\alpha}^R$ on $\mathcal D$ for an admissible
Bregman proximal generator $F_\alpha$.  Then the dual endpoint potential
$u_\alpha$ agrees, up to an additive constant, with the convex conjugate of
$F_\alpha+R$.  More precisely, there is a constant $c_\alpha\in\R$ such that $u_\alpha =(F_\alpha+R)^*+c_\alpha$, on $\ThetaSet$.
Consequently, $u_\alpha$ is convex, and hence we have
\[
  D_\Psi(\thetavec',\thetavec)
  \le
  \frac1\alpha D_{R^*}(\thetavec',\thetavec)
  \qquad(\thetavec,\thetavec'\in\ThetaSet).
\]
\end{repeatprop}

\begin{proof}
Let $H_\alpha\colloneq F_\alpha+R$.  Since $F_\alpha$ is
$\rho_\alpha$-weakly convex with $\rho_\alpha<m$, the function $H_\alpha$ is
strongly convex, and $H_\alpha^*$ is differentiable.  For
$\xvec=\nabla R^*(\thetavec)$, the Bregman proximal problem differs by a
constant from
\[
  \min_{\yvec}\{H_\alpha(\yvec)-\inner{\thetavec}{\yvec}\}.
\]
Hence
\[
  \phi_\alpha(\nabla R^*(\thetavec))
  =\nabla H_\alpha^*(\thetavec).
\]
Mirror exactness gives the independent identity
\[
  \phi_\alpha(\nabla R^*(\thetavec))
  =\nabla(R^*-\alpha\Psi)(\thetavec)
  =\nabla u_\alpha(\thetavec).
\]
The open convex set $\ThetaSet$ is connected, so two differentiable functions
with the same gradient differ by a constant.  Thus
\[
  u_\alpha=H_\alpha^*+c_\alpha
\]
on $\ThetaSet$.  Convexity follows.  Finally,
\[
  D_{u_\alpha}(\thetavec',\thetavec)
  =D_{R^*}(\thetavec',\thetavec)
  -\alpha D_\Psi(\thetavec',\thetavec)
  \ge0,
\]
which is the stated relative-curvature inequality.
\end{proof}

\begin{repeatprop}{prop:legendre-prox-reconstruction}
Suppose $u_\alpha$ is convex on $\ThetaSet$, and let
\[
  \overline u_\alpha
  \colloneq
  \cl\bigl(u_\alpha+\indicator_{\ThetaSet}\bigr).
\]
Assume $\overline u_\alpha$ is proper and that
$\phi_\alpha(\mathcal D)\subseteq\dom R$.  Define the candidate proximal
generator
\[
  F_\alpha(\yvec)
  \colloneq
  \begin{cases}
    \overline u_\alpha^*(\yvec)-R(\yvec),
      & \yvec\in\dom R,\\
    +\infty, & \yvec\notin\dom R.
  \end{cases}
\]
Then, for every $\xvec\in\mathcal D$,
\[
  \phi_\alpha(\xvec)
  \in
  \argmin_{\yvec}
  \bigl\{F_\alpha(\yvec)+\DR(\yvec,\xvec)\bigr\}.
\]
Under these domain assumptions, convexity of $u_\alpha$ yields a
variational Bregman representation of $\phi_\alpha$.  If, in addition, $F_\alpha$ is proper, lower semicontinuous,
and $\rho_\alpha$-weakly convex for some $\rho_\alpha<m$, the minimizer is
unique and $\phi_\alpha=\prox_{F_\alpha}^R$.  In that case
$\phi\in\PhiMirrorRad{R}(\mathcal D)$.
\end{repeatprop}

\begin{proof}
A finite differentiable convex function on the open set $\ThetaSet$ is
continuous there, so
\[
  \overline u_\alpha
  =\cl(u_\alpha+\indicator_{\ThetaSet})
\]
agrees with $u_\alpha$ on $\ThetaSet$.  Fix
$\thetavec\in\ThetaSet$ and set
\[
  \zvec\colloneq\nabla u_\alpha(\thetavec)
  =\phi_\alpha(\nabla R^*(\thetavec)).
\]
Convexity gives
$\zvec\in\partial\overline u_\alpha(\thetavec)$.  Fenchel duality then gives
\[
  \thetavec\in\partial\overline u_\alpha^*(\zvec).
\]
The endpoint-domain hypothesis gives $\zvec\in\dom R$, and Fenchel
equality gives $\overline u_\alpha^*(\zvec)<+\infty$.  In particular,
$F_\alpha+R$ is proper.
Consequently, $\zvec$ minimizes
\[
  \yvec\longmapsto
  \overline u_\alpha^*(\yvec)-\inner{\thetavec}{\yvec}.
\]
For $\xvec=\nabla R^*(\thetavec)$ and $\yvec\in\dom R$, the objective in the
proposition satisfies
\[
\begin{aligned}
  F_\alpha(\yvec)+\DR(\yvec,\xvec)
  &=\overline u_\alpha^*(\yvec)-R(\yvec)
    +R(\yvec)-R(\xvec)-\inner{\thetavec}{\yvec-\xvec}\\
  &=\overline u_\alpha^*(\yvec)-\inner{\thetavec}{\yvec}
    +\inner{\thetavec}{\xvec}-R(\xvec).
\end{aligned}
\]
The last two terms do not depend on $\yvec$.  Since
$\zvec\in\dom R$ and the objective is infinite outside $\dom R$,
$\zvec=\phi_\alpha(\xvec)$ is a minimizer.  If $F_\alpha$ is $\rho_\alpha$-weakly convex with
$\rho_\alpha<m$, then $F_\alpha+R$ is strongly convex.  The minimizer is
unique and equals $\prox_{F_\alpha}^R(\xvec)$.
\end{proof}

\begin{repeatcorollary}{cor:mirror-interpolation-radius}
For $0<\alpha\le1$, the function $u_\alpha=R^*-\alpha\Psi$ is convex
exactly when $\alpha L_R(\Psi)\le1$.  Define the convexity threshold
\[
  \alpha_{\mathrm{cvx}}
  \colloneq
  \min\left\{1,\frac1{L_R(\Psi)}\right\},
\]
with the conventions $1/0\colloneq+\infty$ and
$1/(+\infty)\colloneq0$; a zero threshold means that no positive scale passes
the convexity test.  Every Bregman proximal interpolation must satisfy
$0<\alpha\le\alpha_{\mathrm{cvx}}$.  Conversely, every such $\alpha$ for which
$\overline u_\alpha$ is proper and
$\phi_\alpha(\mathcal D)\subseteq\dom R$ yields the variational representation
in \Cref{prop:legendre-prox-reconstruction}.  It is an admissible Bregman
proximal interpolation if the resulting $F_\alpha$ also satisfies the
weak-convexity and regularity conditions in that proposition.
\end{repeatcorollary}

\begin{proof}
For every ordered pair in $\ThetaSet$,
\[
  D_{R^*-\alpha\Psi}(\thetavec',\thetavec)
  =D_{R^*}(\thetavec',\thetavec)
  -\alpha D_\Psi(\thetavec',\thetavec).
\]
A differentiable function on an open convex set is convex if and only if its
Bregman divergence is nonnegative on every ordered pair.  Hence
$u_\alpha$ is convex if and only if
\[
  D_\Psi(\thetavec',\thetavec)
  \le\frac1\alpha D_{R^*}(\thetavec',\thetavec)
  \qquad\text{for all }\thetavec,\thetavec'\in\ThetaSet.
\]
If $L_R(\Psi)<+\infty$, take a decreasing sequence of valid constants tending
to the infimum.  Passing to the pointwise limit shows that the infimum is
valid.  If $L_R(\Psi)=+\infty$, no positive interpolation scale
satisfies the inequality.  Otherwise, the largest $\alpha\in(0,1]$
satisfying it is $\min\{1,1/L_R(\Psi)\}$.  Necessity for a Bregman proximal
interpolation follows from \Cref{prop:legendre-prox-necessary}.  For the
converse, the additional assumptions that $\overline u_\alpha$ is proper and
$\phi_\alpha(\mathcal D)\subseteq\dom R$ permit application of
\Cref{prop:legendre-prox-reconstruction}; its remaining hypotheses ensure
admissibility of the reconstructed generator.
\end{proof}

\begin{repeatprop}{prop:mirror-no-scaling}
Let $1<p<2$, let $q\colloneq p/(p-1)>2$, and take the full-domain Legendre
regularizer
\[
  R(\xvec)\colloneq\frac12\norm{\xvec}_p^2
  \qquad(\xvec\in\R^2).
\]
There is a feasible $R$-mirror-exact deviation whose dual potential is
$C^\infty$ with globally Lipschitz gradient, but which does not belong to
$\PhiMirrorRad{R}(\R^2)$.
\end{repeatprop}

\begin{proof}
Let $q\colloneq p/(p-1)>2$.  The conjugate regularizer is
\[
  R^*(\thetavec)=\frac12\norm{\thetavec}_q^2.
\]
Set
\[
  \Psi(\thetavec)\colloneq\frac12\thetavec[2]^2,
  \qquad
  \phi(\xvec)
  \colloneq
  \xvec-\gradtheta\Psi(\nabla R(\xvec)).
\]
The map is feasible because the action space is $\R^2$.  It is mirror-exact by
construction, and $\nabla\Psi$ is globally $1$-Lipschitz.

Fix $\alpha>0$ and restrict
$u_\alpha=R^*-\alpha\Psi$ to the line
$\thetavec=\evec_1+t\evec_2$.  Then
\[
  u_\alpha(\evec_1+t\evec_2)
  =\frac12(1+\abs{t}^q)^{2/q}-\frac\alpha2t^2.
\]
The expansion $(1+s)^{2/q}=1+(2/q)s+O(s^2)$ gives
\[
  u_\alpha(\evec_1+t\evec_2)-u_\alpha(\evec_1)
  =\frac1q\abs{t}^q-\frac\alpha2t^2+O(\abs{t}^{2q}).
\]
Since $q>2$, this difference is negative for every sufficiently small
nonzero $t$.  The restricted function is even, so convexity would imply
\[
  u_\alpha(\evec_1)
  \le
  \frac12u_\alpha(\evec_1+t\evec_2)
  +\frac12u_\alpha(\evec_1-t\evec_2)
  =u_\alpha(\evec_1+t\evec_2),
\]
a contradiction.  Thus $u_\alpha$ is not convex for any $\alpha>0$.
By \Cref{prop:legendre-prox-necessary}, no positive interpolation can be an
admissible Bregman proximal map.
\end{proof}

\subsection{\texorpdfstring{Proofs for \Cref{sec:md-ftrl-legendre}}{Proofs for Section 9.1}}
\label{app:proofs-md-ftrl-legendre}

\begin{repeatprop}{prop:md-ftrl-equivalence}
Under the standing Legendre assumptions of \Cref{sec:mirror}, let
$\thetavec_1\colloneq\nabla R(\xvec_1)$ and use a constant step size $\eta$.  The
mirror-descent iterates satisfy
\[
  \xvec_t
  =\argmin_{\xvec\in E}
  \left\{
    \eta\sum_{s=1}^{t-1}\inner{\gvec_s}{\xvec}
    +R(\xvec)-\inner{\thetavec_1}{\xvec}
  \right\}
  =\argmin_{\xvec\in K}
  \left\{
    \eta\sum_{s=1}^{t-1}\inner{\gvec_s}{\xvec}
    +\DR(\xvec,\xvec_1)
  \right\}.
\]
Thus they are the FTRL iterates for the linearized losses and the tilted regularizer
$R-\inner{\thetavec_1}{\cdot}$.
\end{repeatprop}

\begin{proof}
The interior optimality condition for mirror descent is
\[
  \nabla R(\xvec_{t+1})
  =\nabla R(\xvec_t)-\eta\gvec_t.
\]
Writing $\thetavec_t\colloneq\nabla R(\xvec_t)$ and unrolling gives
$\thetavec_t=\thetavec_1-\eta\sum_{s<t}\gvec_s$.  Legendre duality therefore
gives
\[
  \xvec_t=\nabla R^*(\thetavec_t)
  =\argmin_{\xvec\in E}
  \{R(\xvec)-\inner{\thetavec_t}{\xvec}\}.
\]
This is the first FTRL formula.  The second follows because
$R(\xvec)-\inner{\thetavec_1}{\xvec}$ differs from
$\DR(\xvec,\xvec_1)$ by a constant independent of $\xvec$.
\end{proof}

\subsection{\texorpdfstring{Proofs for \Cref{sec:ftrl-exactness}}{Proofs for Section 9.2}}
\label{app:proofs-ftrl-exactness}

\begin{repeatthm}{thm:ftrl-regret}
Assume that $R$ is proper, closed, $(\sigma,r)$-uniformly convex, and that
$R^*$ is finite on $E$.  Let the iterates follow \eqref{eq:ftrl-update}.  If
$\phi$ is $R$-FTRL-exact with potential $\Psi$ and
\[
  D_\Psi(\thetavec_{t+1},\thetavec_t)
  \le L D_{R^*}(\thetavec_{t+1},\thetavec_t)
  \qquad(t\in[T]),
\]
then
\[
  \sum_{t=1}^T
  \inner{\gvec_t}{\xvec_t-\phi(\xvec_t)}
  \le
  \frac{\Psi(\thetavec_1)-\Psi(\thetavec_{T+1})}{\eta}
  +\frac{L\eta^{q-1}}{q\sigma^{q-1}}
   \sum_{t=1}^T\dualnorm{\gvec_t}^{q}.
\]
\end{repeatthm}

\begin{proof}
Uniform convexity makes $\partial R^*(\thetavec)$ a singleton whenever it is
nonempty.  Since $R^*$ is finite on $E$, it has a nonempty subdifferential
everywhere and is therefore differentiable on $E$.

FTRL-exactness and the additive dual update give
\[
\begin{aligned}
  \eta\inner{\gvec_t}{\xvec_t-\phi(\xvec_t)}
  &=\inner{\thetavec_t-\thetavec_{t+1}}
    {\gradtheta\Psi(\thetavec_t)}\\
  &=\Psi(\thetavec_t)-\Psi(\thetavec_{t+1})
    +D_\Psi(\thetavec_{t+1},\thetavec_t)\\
  &\le\Psi(\thetavec_t)-\Psi(\thetavec_{t+1})
    +L D_{R^*}(\thetavec_{t+1},\thetavec_t).
\end{aligned}
\numberthis{eq:ftrl-one-step}
\]

Fix $\thetavec\in E$, let $\xvec\colloneq\nabla R^*(\thetavec)$, and let
$\dvec\in E$.  Uniform convexity gives, for every $\yvec\in\dom R$,
\[
  R(\yvec)
  \ge R(\xvec)+\inner{\thetavec}{\yvec-\xvec}
  +\frac\sigma r\norm{\yvec-\xvec}^r.
\]
Hence
\[
\begin{aligned}
  R^*(\thetavec+\dvec)
  &\le R^*(\thetavec)+\inner{\dvec}{\xvec}
  +\sup_{\zvec\in E}
   \left\{\inner{\dvec}{\zvec}-\frac\sigma r\norm{\zvec}^r\right\}\\
  &=R^*(\thetavec)+\inner{\dvec}{\nabla R^*(\thetavec)}
  +\frac1{q\sigma^{q-1}}\dualnorm{\dvec}^q.
\end{aligned}
\]
Thus
\[
  D_{R^*}(\thetavec+\dvec,\thetavec)
  \le\frac1{q\sigma^{q-1}}\dualnorm{\dvec}^q.
\]
Apply this with
$\dvec\colloneq\thetavec_{t+1}-\thetavec_t=-\eta\gvec_t$ in
\eqref{eq:ftrl-one-step}, divide by $\eta$, and sum.
\end{proof}

\begin{repeatcorollary}{cor:ftrl-rate}
Suppose $\dualnorm{\gvec_t}\le G$ and the oscillation of $\Psi$ along the dual
trajectory is at most $B$.  When $B,L,G>0$, choosing
$\eta=\left(\frac{rB\sigma^{q-1}}{LG^qT}\right)^{1/q}$ gives
\[
  \sum_{t=1}^T
  \inner{\gvec_t}{\xvec_t-\phi(\xvec_t)}
  \le
  r^{1/r}\frac{G L^{1/q}B^{1/r}}{\sigma^{1/r}}T^{1/q}.
\]
For $r=2$, this is $G\sqrt{2LBT/\sigma}$.
\end{repeatcorollary}

\begin{proof}
The theorem gives
\[
  \frac B\eta
  +\frac{LG^qT}{q\sigma^{q-1}}\eta^{q-1}.
\]
Differentiation yields the displayed optimizer and bound.
\end{proof}

\begin{repeatthm}{thm:ftrl-circulation}
Suppose $\Dphidual$ is continuous and there is a closed piecewise
$C^1$ loop $\gamma\subset\ThetaSet$ with
\[
  \oint_\gamma
  \inner{\Dphidual(\thetavec)}{\dd\thetavec}\ne0.
\]
For every $G>0$ and every sufficiently small constant $\eta>0$, bounded
gradients $\dualnorm{\gvec_t}\le G$ can force FTRL
\eqref{eq:ftrl-update} to incur $cT-O(1/\eta)$ regret against $\phi$, for some
$c>0$.  The statement holds from every initial dual state in $\ThetaSet$ since
steering to the loop changes only the transient term.
\end{repeatthm}

\begin{proof}
Reverse the orientation if necessary so that
\[
  -\oint_\gamma
  \inner{\Dphidual(\thetavec)}{\dd\thetavec}
  =c_0>0.
\]
Because $\ThetaSet$ is open and convex, join the initial state to a point of
the loop by a compact polygonal path in $\ThetaSet$.  Discretizing that path
into dual increments of norm at most $\eta G$ steers FTRL to the loop in
$O(1/\eta)$ rounds.  Continuity bounds the displacement on the path, so the
magnitude of the steering regret is $O(1/\eta)$.

Let $L_\gamma$ be the length of the loop in the dual norm.  For every
sufficiently small $\eta$, choose a partition
$\thetavec_0,\ldots,\thetavec_s=\thetavec_0$ with
\[
  \dualnorm{\thetavec_j-\thetavec_{j+1}}\le\eta G,
  \qquad
  s\le\frac{2L_\gamma}{\eta G},
\]
and whose left Riemann sum is at least $c_0/2$.  Set
\[
  \gvec_{j+1}
  \colloneq\frac{\thetavec_j-\thetavec_{j+1}}{\eta}.
\]
The FTRL dual state follows the polygon exactly, and one traversal earns at
least $c_0/(2\eta)$ regret.  Its average regret per round is at least
$c_0G/(4L_\gamma)$.  Repeating the loop gives
\[
  \sum_{t=1}^T
  \inner{\gvec_t}{\xvec_t-\phi(\xvec_t)}
  \ge\frac{c_0G}{4L_\gamma}T-O(1/\eta).
\]
\end{proof}

\begin{repeatcorollary}{cor:ftrl-curl}
Assume that $\Dphidual$ is $C^1$ on the open convex set $\ThetaSet$.
It is $R$-FTRL-exact if and only if its Jacobian is symmetric.  If a skew
derivative is nonzero, \Cref{thm:ftrl-circulation} gives linear regret along a
small dual loop.  Under the regularity and normalization assumptions of
\Cref{thm:ftrl-regret}, an exact field has sublinear regret.
\end{repeatcorollary}

\begin{proof}
The equivalence between a symmetric Jacobian and a potential follows from the
Poincar\'e lemma on the open convex set $\ThetaSet$.  A nonzero skew derivative
has nonzero circulation on a sufficiently small rectangle, so
\Cref{thm:ftrl-circulation} applies.  The positive statement follows
from \Cref{cor:ftrl-rate}; its horizon-independent bounds give
$O(T^{1/q})$ regret, which is sublinear because $q>1$.
\end{proof}

\subsection{\texorpdfstring{Proofs for \Cref{sec:md-ftrl-deviation-separation}}{Proofs for Section 9.3}}
\label{app:proofs-md-ftrl-separation}

\begin{repeatprop}{prop:ftrl-exact-implies-md-exact}
Let $K$ be compact and convex, and let $R$ be differentiable and strongly
convex on a neighborhood of $K$.  If a feasible map $\phi:K\to K$ is
$R_K$-FTRL-exact, then it is $R$-mirror-exact in the sense of
\Cref{def:mirror-exact-beyond-legendre}.  Consequently,
\[
  \PhiFTRL{R_K}(K)
  \subseteq
  \PhiMirrorExact{R}(K).
\]
\end{repeatprop}

\begin{proof}
Strong convexity makes $R_K^*$ differentiable.  For every $\xvec\in K$,
\[
  \nabla R(\xvec)
  \in \nabla R(\xvec)+\normal_K(\xvec)
  =\partial R_K(\xvec),
\]
because $0\in \normal_K(\xvec)$.  Fenchel duality and uniqueness give
\[
  \nabla R_K^*(\nabla R(\xvec))=\xvec.
\]
If $\phi$ is $R_K$-FTRL-exact with potential $\Psi$, substitute
$\thetavec=\nabla R(\xvec)$ in Definition~\ref{def:ftrl-exact} to obtain
\[
  \xvec-\phi(\xvec)
  =\gradtheta\Psi(\nabla R(\xvec)).
\]
This is Definition~\ref{def:mirror-exact-beyond-legendre}.
\end{proof}

\begin{repeatthm}{thm:md-ftrl-lp-separation}
The map $\phi_{p,\eps}$ is a smooth feasible deviation.  Standard
constrained mirror descent with distance generator $R_p$ has
$O(\sqrt T)$ regret against it for bounded gradients and the usual
$T^{-1/2}$ step size.  In contrast, FTRL with the extended regularizer
$R_{p,K}$ can be forced, for every sufficiently small constant step
size, to incur $cT-O(1/\eta)$ regret against the same deviation.  Hence the
inclusion in \Cref{prop:ftrl-exact-implies-md-exact} is strict and $ \PhiFTRL{R_{p,K}}(K) \subsetneq \PhiMirrorExact{R_p}(K)$.
\end{repeatthm}

\begin{proof}
Let $c\colloneq(a+b)/2$, $d\colloneq(b-a)/2$, choose $\beta\in(0,1)$, and use the map in
\eqref{eq:lp-md-ftrl-separation-map}.  The Hessian is
\[
  \nabla^2R_p(\xvec)
  =(p-1)\operatorname{diag}
  \bigl(\xvec[1]^{p-2},\xvec[2]^{p-2}\bigr),
\]
and every diagonal entry is at least
$m_p\colloneq(p-1)\min_{s\in[a,b]}s^{p-2}>0$.
Write $\zvec\colloneq\xvec-c\one$.  For each coordinate $i$ and the other coordinate
$j$,
\[
\begin{aligned}
  \field_{p,\eps}(\xvec)[i]
  &=\frac{\eps}{(p-1)\xvec[i]^{p-2}}
    \bigl(\zvec[i]+\beta\zvec[j]\bigr),\\
  \abs{\field_{p,\eps}(\xvec)[i]}
  &\le \frac{\eps(1+\beta)d}{m_p}.
\end{aligned}
\]
If $\field_{p,\eps}(\xvec)[i]>0$, then
$\zvec[i]>-\beta\zvec[j]\ge-\beta d$, and hence
$\xvec[i]-a>(1-\beta)d$.  The assumed upper bound on $\eps$ gives
$\field_{p,\eps}(\xvec)[i]\le\xvec[i]-a$.  Therefore
$\phi_{p,\eps}(\xvec)[i]\ge a$.  Its upper bound is immediate because
subtracting a positive displacement moves the coordinate downward.  If
$\field_{p,\eps}(\xvec)[i]<0$, the same argument gives
$b-\xvec[i]>(1-\beta)d$ and
$\abs{\field_{p,\eps}(\xvec)[i]}\le b-\xvec[i]$.  Thus
$\phi_{p,\eps}(K)\subseteq K$.

Define
\[
  V(\xvec)
  \colloneq\frac\eps2
   (\xvec-c\one)^{\transpose}S_\beta(\xvec-c\one).
\]
Then
\[
  \nabla V(\xvec)
  =\eps S_\beta(\xvec-c\one)
  =\nabla^2R_p(\xvec)\field_{p,\eps}(\xvec).
\]
Hence
$\field_{p,\eps}(\xvec)^{\transpose}\dd\nabla R_p(\xvec)=\dd V(\xvec)$.
Equivalently, on the positive dual orthant,
\[
  \Psi(\thetavec)
  \colloneq\frac\eps2
  \left(
  \begin{pmatrix}
  \thetavec[1]^{1/(p-1)}\\
  \thetavec[2]^{1/(p-1)}
  \end{pmatrix}
  -c\one
  \right)^{\transpose}
  S_\beta
  \left(
  \begin{pmatrix}
  \thetavec[1]^{1/(p-1)}\\
  \thetavec[2]^{1/(p-1)}
  \end{pmatrix}
  -c\one
  \right)
\]
satisfies
$\field_{p,\eps}(\xvec)=\gradtheta\Psi(\nabla R_p(\xvec))$.
Because the box is compact and bounded away from the coordinate hyperplanes,
$R_p$, $\field_{p,\eps}$, and $\Psi$ are smooth on neighborhoods of the
relevant primal and dual sets.  The Hessian of $R_p^*$ is positive definite
there.  Compactness therefore gives finite constants for the relative
curvature, displacement Lipschitz modulus, mirror-gradient Lipschitz modulus,
and potential oscillation required by
\Cref{cor:md-exact-beyond-legendre-rate}.  Standard mirror descent consequently
has $O(\sqrt T)$ regret against the map under bounded gradients.

For FTRL, the minimization defining $\nabla R_{p,K}^*$ is separable, and its $i$th coordinate is
\[
  \nabla R_{p,K}^*(\thetavec)[i]
  =
  \begin{cases}
  a, & \thetavec[i]\le a^{p-1},\\
  \thetavec[i]^{1/(p-1)},
    & a^{p-1}<\thetavec[i]<b^{p-1},\\
  b, & \thetavec[i]\ge b^{p-1}.
  \end{cases}
\]
Choose numbers
$u<v$ with $b^{p-1}<u<v$ and
$s<t$ with $a^{p-1}<s<t<b^{p-1}$.  On the rectangle
$[u,v]\times[s,t]$, the first primal coordinate is $b$ and the second is
$\thetavec[2]^{1/(p-1)}$.  Consequently the first component of the dual
displacement is strictly increasing in $\thetavec[2]$:
\[
  \frac{\partial\widetilde{\field}_{p,\eps}[1]}
  {\partial\thetavec[2]}
  =\frac{\eps\beta}{(p-1)^2}
   b^{2-p}\thetavec[2]^{(2-p)/(p-1)}>0,
\]
whereas
$\partial\widetilde{\field}_{p,\eps}[2]/\partial\thetavec[1]=0$.
Equivalently, the circulation around the positively oriented rectangle is
\[
  (v-u)
  \left[
  \widetilde{\field}_{p,\eps}[1](u,s)
  -\widetilde{\field}_{p,\eps}[1](u,t)
  \right]
  <0.
\]
The dual displacement is continuous, so
\Cref{thm:ftrl-circulation} produces bounded gradients with
$cT-O(1/\eta)$ FTRL regret.  This proves both the claimed algorithmic
separation and the strictness of
\Cref{prop:ftrl-exact-implies-md-exact}.
\end{proof}

\subsection{\texorpdfstring{Proofs for \Cref{sec:ftrl-lp-rates}}{Proofs for Section 9.4}}
\label{app:proofs-ftrl-lp-rates}

\begin{repeatcorollary}{cor:ftrl-lp-squared}
Let $1<p\le2$ and let $q\colloneq p/(p-1)$.  Then
\[
  \sum_{t=1}^T
  \inner{\gvec_t}{\xvec_t-\phi(\xvec_t)}
  \le
  \frac{B}{\eta}
  +\frac{L\eta}{2(p-1)}
   \sum_{t=1}^T\norm{\gvec_t}_q^2.
\]
If $\norm{\gvec_t}_q\le G$, the optimized bound is
$G\sqrt{2LBT/(p-1)}$.
\end{repeatcorollary}

\begin{proof}
The function $\frac12\norm{\cdot}_p^2$ is $(p-1)$-strongly convex with respect
to $\norm{\cdot}_p$ for $1<p\le2$.  If
$\thetavec=\nabla(\frac12\norm{\xvec}_p^2)+\nvec$ with
$\nvec\in\normal_K(\xvec)$, then
$\inner{\nvec}{\yvec-\xvec}\le0$ for $\yvec\in K$, so adding
$\indicator_K$ preserves the same strong-convexity inequality for every
selected subgradient.  Apply \Cref{thm:ftrl-regret} with uniform-convexity
exponent $r=2$, conjugate exponent $2$, and $\sigma=p-1$.
\end{proof}

\begin{repeatcorollary}{cor:ftrl-lp-powered}
Let $2\le p<\infty$ and let $q\colloneq p/(p-1)$.  Then
\[
  \sum_{t=1}^T
  \inner{\gvec_t}{\xvec_t-\phi(\xvec_t)}
  \le
  \frac{B}{\eta}
  +\frac{L\eta^{q-1}}
  {q(2^{2-p})^{q-1}}
  \sum_{t=1}^T\norm{\gvec_t}_q^q.
\]
If $\norm{\gvec_t}_q\le G$, the optimized bound is $
  p^{1/p}2^{(p-2)/p}G L^{1/q}B^{1/p}T^{1/q}$,
which is $O(T^{1-1/p})$.
\end{repeatcorollary}

\begin{proof}
For $A(s)=\abs{s}^{p-2}s$, the scalar inequality
\[
  (A(a)-A(b))(a-b)\ge2^{2-p}\abs{a-b}^p
\]
holds for all real $a,b$.  Integrating coordinatewise along the segment from
$\xvec$ to $\yvec$ gives
\[
  D_{p^{-1}\norm{\cdot}_p^p}(\yvec,\xvec)
  \ge\frac{2^{2-p}}p\norm{\yvec-\xvec}_p^p.
\]
The same normal-cone argument as above shows that adding $\indicator_K$
preserves the estimate for every selected subgradient.  Thus the regularizer is
$(2^{2-p},p)$-uniformly convex.  Apply \Cref{thm:ftrl-regret} and
Corollary~\ref{cor:ftrl-rate}.
\end{proof}
\par

\subsection{Proofs for \texorpdfstring{\Cref{sec:games}}{Section 10}}
\label{app:proofs-games}

\begin{repeatthm}{thm:equilibrium}
Let $\bar\mu_T\colloneq T^{-1}\sum_{t=1}^T\delta_{\xvec_t}$ and define
$\gvec_{i,t}\colloneq\nabla_{\xvec_i}\ell_i(\xvec_{i,t},\xvec_{-i,t})$.  If every player $i$
guarantees
\[
 \sup_{\phi_i\in\Phi_i}
 \sum_{t=1}^T\inner{\gvec_{i,t}}
 {\xvec_{i,t}-\phi_i(\xvec_{i,t})}\le R_{i,T},
\]
then $\bar\mu_T$ is an $\eps_T$-ConCE relative to $(\Phi_i)_i$, where
$\eps_T=\max_i R_{i,T}/T$.  In particular, OGD on
$\Phi_{i,\mathrm{exact}}(B_i,L_i)$ with $\norm{\gvec_{i,t}}\le G_i$ gives $
 \eps_T\le \max_i 2G_i\sqrt{L_iB_i/T}$.
\end{repeatthm}

\begin{proof}
Fix player $i$ and $\phi_i\in\Phi_i$.  By convexity of $\ell_i$ in its own
action,
\[
  &\ell_i(\xvec_{i,t},\xvec_{-i,t})-
  \ell_i(\phi_i(\xvec_{i,t}),\xvec_{-i,t})\\
  &\qquad\le
  \inner{\nabla_{\xvec_i}\ell_i(\xvec_{i,t},\xvec_{-i,t})}
  {\xvec_{i,t}-\phi_i(\xvec_{i,t})}.
\]
Summing, taking the supremum over $\Phi_i$, and dividing by $T$ gives the
definition of $\eps_T$-ConCE under $\bar\mu_T$.  For the final statement apply
\Cref{cor:uniform-regret} player by player.
\end{proof}

\begin{repeatprop}{prop:radial-equilibrium}
Let $\mathcal A_i$ be any family of feasible Borel-measurable
deviations for player $i$.  A
distribution satisfies
$\Gamma_i(\mu,\phi_i)\le0$ for every $\phi_i\in\mathcal A_i$ and every player
if and only if it satisfies the same inequalities for every
$\phi_i\in\operatorname{Rad}_I(\mathcal A_i)$.
\end{repeatprop}

\begin{proof}
Since $\mathcal A_i\subseteq\operatorname{Rad}_I(\mathcal A_i)$ by taking
$\alpha=1$, one implication is immediate.  For the converse, suppose that
$\mu$ satisfies every constraint in $\mathcal A_i$, and let
$\phi_i\in\operatorname{Rad}_I(\mathcal A_i)$.  Choose
$\alpha\in(0,1]$ such that
$\phi_{i,\alpha}=(1-\alpha)I+\alpha\phi_i$ belongs to $\mathcal A_i$.
Convexity of $\ell_i$ in its own action gives pointwise
\[
\begin{aligned}
  &\ell_i(\xvec_i,\xvec_{-i})
  -\ell_i(\phi_{i,\alpha}(\xvec_i),\xvec_{-i})\\
  &\qquad\ge
  \alpha\left(
  \ell_i(\xvec_i,\xvec_{-i})
  -\ell_i(\phi_i(\xvec_i),\xvec_{-i})\right).
\end{aligned}
\]
The joint continuity and compactness assumptions ensure that
both loss differences are bounded and measurable.  Thus both gains are
integrable.
Taking expectations yields
$\Gamma_i(\mu,\phi_{i,\alpha})
\ge\alpha\Gamma_i(\mu,\phi_i)$.  The left-hand side is nonpositive, so
$\Gamma_i(\mu,\phi_i)\le0$.
\end{proof}

\begin{repeatthm}{thm:conce-pce-equality}
Under the unrestricted conventions above and at zero approximation error $\eps = 0$, every convex game with compact
convex action sets and jointly continuous losses satisfies
\[
  \ConCE=\PCE.
\]
\end{repeatthm}

\begin{proof}
By Corollary~\ref{cor:euclidean-radial-closure}, the smooth exact-form class is the
identity-radial closure of the unrestricted proximal class.  Apply
Proposition~\ref{prop:radial-equilibrium} player by player.
\end{proof}

\begin{repeatcorollary}{cor:mirror-equilibrium-equality}
For any set of regularizers $R_i$ for $i \in [N]$, we have
\[
  \PCE^{\mathrm{Breg}}_{(R_i)_i}
  =
  \mathrm{Eq}^{\mathrm{rad}}_{(R_i)_i}.
\]
\end{repeatcorollary}

\begin{proof}
For every player,
$\PhiMirrorRad{R_i}(K_i)
=\operatorname{Rad}_I\bigl(\PhiMirrorProx{R_i}(K_i)\bigr)$ by
Definition~\ref{def:bregman-radial-class}.  Apply
Proposition~\ref{prop:radial-equilibrium} player by player.
\end{proof}

\begin{repeatcorollary}{cor:finite-support}
For every finitely supported distribution in a convex game,
\[
  \mu\in\CE\quad\Longleftrightarrow\quad
  \mu\in\ConCE\quad\Longleftrightarrow\quad
  \mu\in\PCE.
\]
In particular the three notions coincide in every finite game.
\end{repeatcorollary}

\begin{proof}
Only $\PCE\subseteq\CE$ needs proof.  If a general deviation $\tau_i$ has
positive gain, list the finitely many recommendations $\xvec^1,\ldots,\xvec^m$
to player $i$ and set $\yvec^j\colloneq\tau_i(\xvec^j)$.  By
\Cref{thm:finite-prox}, for sufficiently small common $\alpha>0$ there is a
proximal map satisfying
\[
 \phi_{i,\alpha}(\xvec^j)
 =(1-\alpha)\xvec^j+\alpha\tau_i(\xvec^j)
\]
at every supported recommendation.  Convexity yields
$\Gamma_i(\mu,\phi_{i,\alpha})\ge
\alpha\Gamma_i(\mu,\tau_i)>0$, so $\mu\notin\PCE$.
\end{proof}

\begin{repeatthm}{thm:one-dimensional-ce}
Suppose that every player has a compact interval action set
$K_i\colloneq[a_i,b_i]\subseteq\R$ and, in addition to the standing assumptions, every
loss $\ell_i:K\to\R$ is continuous on
$K\colloneq\prod_{j=1}^N K_j$.  Then
\[
  \CE=\ConCE=\PCE
\]
under the unrestricted definitions used in this section when the approximation
error is zero.
\end{repeatthm}

\begin{proof}
The inclusion $\CE\subseteq\ConCE$ follows directly from the deviation classes, and $\ConCE=\PCE$ follows from
\Cref{thm:conce-pce-equality}.  It remains to prove
$\ConCE\subseteq\CE$.

Suppose $\mu\notin\CE$. Then for some player $i$, there is a Borel-measurable deviation $\tau_i:K_i\to K_i$ such that $\Gamma_i(\mu,\tau_i)>0$. Let $\nu_i$ be the $i$-th marginal of $\mu$. Assume that $K_i\colloneq[a_i,b_i]$ with $a_i<b_i$ (the case in which $K_i$ is a singleton is trivial). 

For every $n$, by Lusin's theorem, we have a compact set $A_n\subseteq K_i$ with $\nu_i(K_i\setminus A_n)<1/n$ such that $\tau_i|_{A_n}$ is continuous. From classical extension theorem (bounded form of the Tietze extension theorem), we can get a continuous map $h_n:K_i\to K_i$ agreeing with $\tau_i$ on $A_n$. After rescaling $u\colloneq(x-a_i)/(b_i-a_i)$, let $B_mh_n$ denote the Bernstein polynomial 
\[
 (B_mh_n)(x)
 =\sum_{k=0}^m
 h_n\!\left(a_i+\frac{k}{m}(b_i-a_i)\right)
 \binom{m}{k}u^k(1-u)^{m-k}.
\]
Bernstein approximation is uniform for continuous functions, so choose $m_n$ such that $\|B_{m_n}h_n-h_n\|_\infty<1/n$ and set
$\phi_{i,n}\colloneq B_{m_n}h_n$. Bernstein coefficients are nonnegative
and sum to one, so we have $\phi_{i,n}(K_i)\subseteq K_i$ and $\nu_i\!\left(
   \bigl|\phi_{i,n}-\tau_i\bigr|>\frac1n
 \right)
 \le \nu_i(K_i\setminus A_n)<\frac1n.$ 
 So $\phi_{i,n}(X_i)\to\tau_i(X_i)$ in probability under $\mu$.

 Continuity on compact $K$ makes $\ell_i$ bounded and uniformly continuous. So we get $ \ell_i(\phi_{i,n}(X_i),X_{-i})
 \longrightarrow
 \ell_i(\tau_i(X_i),X_{-i})$ in probability. The variables are uniformly bounded so the convergence is in $L^1(\mu)$. This shows that $\Gamma_i(\mu,\phi_{i,n})
 \longrightarrow \Gamma_i(\mu,\tau_i)>0$. In particular, $\phi_{i,n}$ has positive gain for all sufficiently large $n$.

But every $\phi_{i,n}$ is a smooth exact-form deviation. Indeed, the
polynomial potential
$$
 \Psi_{i,n}(x)=\int_{a_i}^x\bigl(s-\phi_{i,n}(s)\bigr)\,\dd s
$$
is defined on a neighborhood of $K_i$ and satisfies
$\Psi_{i,n}'(x)=x-\phi_{i,n}(x)$. 

So a measurable profitable deviation
produces a smooth feasible exact-form profitable deviation, so
$\mu\notin\ConCE$. This proves $\ConCE\subseteq\CE$ and completes the equality.
\end{proof}

\begin{repeatthm}{thm:ce-separation}
There is a two-player convex game and a distribution $\mu$ with uncountable
support---indeed, with marginals absolutely continuous with respect to
two-dimensional Lebesgue measure---for which
\[
   \mu\in\ConCE=\PCE
   \qquad\text{but}\qquad
   \mu\notin\CE.
\]
Consequently $\CE\subsetneq\ConCE=\PCE$ under the definitions above.
\end{repeatthm}

\begin{proof}
Let $K\subseteq\R^2$ be the unit disk, and let
\[
  J\colloneq\begin{pmatrix}0&-1\\1&0\end{pmatrix}
\]
be the counterclockwise \(90^\circ\) rotation matrix. Player 1 has loss
$\ell_1(\avec,\svec)=\inner{J\svec}{\avec}$, player 2 has zero loss, and
$\mu$ is the law of $(X,X)$ for $X$ uniform on the disk. Thus both marginals
are continuous and the joint law has uncountable support, although it is
singular and supported on the diagonal of $K\times K$. For every exact map
$\phi=I-\nabla\Psi$, the gain is
\[
 \Gamma_1(\mu,\phi)
 =\frac1\pi\int_K\inner{J\xvec}{\nabla\Psi(\xvec)}\,\dd\xvec=0,
\]
because $\operatorname{div}(J\xvec)=0$ and $J\xvec$ has zero normal flux.
Thus $\mu\in\ConCE=\PCE$.  The feasible measurable map
$\tau(\xvec)\colloneq-J\xvec/\norm{\xvec}$ for $\xvec\ne0$, with $\tau(0)\colloneq0$, has
gain $\E\norm X=2/3$, so $\mu\notin\CE$.
\end{proof}

\end{document}